\pdfoutput=1
\def\year{2021}\relax
\documentclass[letterpaper]{article} % DO NOT CHANGE THIS
\usepackage{aaai21}  % DO NOT CHANGE THIS
\usepackage{times}  % DO NOT CHANGE THIS
\usepackage{helvet} % DO NOT CHANGE THIS
\usepackage{courier}  % DO NOT CHANGE THIS
\usepackage[hyphens]{url}  % DO NOT CHANGE THIS
\usepackage{graphicx} % DO NOT CHANGE THIS
\usepackage{natbib}  % DO NOT CHANGE THIS AND DO NOT ADD ANY OPTIONS TO IT
\usepackage{caption} % DO NOT CHANGE THIS AND DO NOT ADD ANY OPTIONS TO IT
\newcommand{\mytitle}
{TempLe: Learning Template of Transitions for Sample Efficient Multi-task RL}

\def\fighome{figures}

\usepackage{url}
\usepackage{dsfont}
\usepackage{xspace}
\usepackage[outdir=./]{epstopdf}
\usepackage{graphicx,float,pgfplots,wrapfig,sidecap,lipsum}
\usepackage{tabularx}
\usepackage{booktabs}
\usepackage{paralist}

\usepackage{algorithm}
\usepackage[noend]{algpseudocode}
\usepackage{amsfonts}
\usepackage{amsthm}
\usepackage{lipsum}
\usepackage{amsmath}
\usepackage{amssymb}
\usepackage{enumitem}
\usepackage{xcolor}
\usepackage[font=normal]{caption}
\usepackage{mathtools} 
\usepackage{bbm}
\usepackage{multirow}

\usepackage{tablefootnote}
\usepackage{subcaption}
\usepackage{subfloat}
\usepackage{tikz}
\usetikzlibrary{fit}
\usetikzlibrary{calc,shapes}
\usetikzlibrary{decorations.pathmorphing} % noisy shapes
\usetikzlibrary{fit}					% fitting shapes to coordinates
\usetikzlibrary{backgrounds}	% drawing the background after the foreground

\usepackage{bm}
\usepackage{mathrsfs} 
\usepackage{stmaryrd}
\usepackage[utf8]{inputenc}
\usepackage[english]{babel}
\usepackage{diagbox}

\usepackage[utf8]{inputenc}
\usepackage{pgfplots}
\pgfplotsset{compat=newest}
\usepgfplotslibrary{groupplots}
\usepgfplotslibrary{dateplot}

\newcommand*{\myvector}[1]{\bm{#1}}

\newcommand{\modname}{{TempLe}\xspace}

\newcommand{\ourmod}{{O-TempLe}\xspace}
\newcommand{\ourmodfull}{{Online Template Learning}\xspace}
\newcommand{\ourmodtwo}{{FM-TempLe}\xspace}
\newcommand{\ourmodtwofull}{{Finite-Model Template Learning}\xspace}
\newcommand{\patnamefull}{{Transition Template}\xspace}
\newcommand{\patnamefulls}{{Transition Templates}\xspace}
\newcommand{\patnamelowercase}{{transition template}\xspace}
\newcommand{\patname}{\textsf{TT}\xspace}
\newcommand{\patnames}{\textsf{TT}s\xspace}

\newcommand{\probgapname}{{ranking gap}\xspace}

\newcommand{\states}{\mathcal{S}}
\newcommand{\actions}{\mathcal{A}}
\newcommand{\pattern}{\mathbf{g}}
\newcommand{\patp}{\mathbf{g}^{(p)}}
\newcommand{\patr}{g^{(r)}}
\newcommand{\patsp}{g^{(p)}}

\newcommand{\patset}{\mathcal{G}}
\newcommand{\patvisset}{\mathcal{O}}
\newcommand{\patvis}{\mathbf{o}}
\newcommand{\patvisp}{\mathbf{o}^{(N)}}
\newcommand{\patvisr}{o^{(R)}}
\newcommand{\mdpset}{\mathcal{C}}
\newcommand{\permut}{\mathbf{\sigma}}

\newcommand{\thressm}{m_s}
\newcommand{\threslg}{m}
\newcommand{\mingap}{\omega}
\newcommand{\kset}{\mathcal{K}}

\def\tha{{\mbox{\tiny th}}}
\def\beq{\begin{equation}}

\def\eeq{\end{equation}\noindent}
\newcommand{\bp}{\begin{psfrags}}
\newcommand{\ep}{\end{psfrags}}
\newcommand{\bc}{\begin{center}}
\newcommand{\ec}{\end{center}}

\newtheorem{theorem}{Theorem}
\newtheorem{lemma}[theorem]{Lemma}
\newtheorem{definition}[theorem]{Definition}

\floatname{algorithm}{Algorithm}
\title{\mytitle}
\author{
    Yanchao Sun, \textsuperscript{\rm 1} 
    Xiangyu Yin,  \textsuperscript{\rm 2} 
    Furong Huang \textsuperscript{\rm 1}
    \\
}
\affiliations{
    \textsuperscript{\rm 1} University of Maryland, College Park, MD, 20742\\
    \textsuperscript{\rm 2} Beijing University of Posts and Telecommunications, China\\
    ycs@umd.edu, yinxiangyu@bupt.edu.cn, furongh@umd.edu

}
\begin{document}

\maketitle

\begin{abstract}
Transferring knowledge among various environments is important for efficiently learning multiple tasks online. Most existing methods directly use the previously learned models or previously learned optimal policies to learn new tasks. However, these methods may be inefficient when the underlying models or optimal policies are substantially different across tasks. In this paper, we propose Template Learning (TempLe), a PAC-MDP method for multi-task reinforcement learning that could be applied to tasks with varying state/action space without prior knowledge of inter-task mappings. TempLe gains sample efficiency by extracting similarities of the transition dynamics across tasks even when their underlying models or optimal policies have limited commonalities. We present two algorithms for an ``online'' and a ``finite-model'' setting respectively. We prove that our proposed TempLe algorithms achieve much lower sample complexity than single-task learners or state-of-the-art multi-task methods. We show via systematically designed experiments that our TempLe method universally outperforms the state-of-the-art multi-task methods (PAC-MDP or not) in various settings and regimes.
\end{abstract}

%!TEX root = 0_aaai_all.tex

\section{Introduction}\label{sec:intro}
% \textit{Reinforcement learning (RL)}~\cite{sutton2018reinforcement} studies the problem of making decisions sequentially by interacting with an environment modeled by a Markov Decision Process (MDP). 
% There are mainly two types of methods: model-based methods and model-free methods, which learn the optimal policy with or without building a model of the environment. 
% Though having a variety of promising applications, RL is challenging due to the lack of prior knowledge, the sparsity of rewards, and the stochasticity of the environment. A possible solution is \textit{Transfer Learning (TL)}~\cite{taylor2009transfer}, which aims to transfer the knowledge learned from a source task to the learning of a target task, so that the amount of experience needed for the target is reduced.

\textit{Multi-task reinforcement learning (MTRL)}~\cite{wilson2007multi,brunskill2013sample,modi2018markov} 
%combines single-task \textit{Reinforcement learning (RL)}~\cite{sutton2018reinforcement} and \textit{Transfer Learning (TL)}~\cite{taylor2009transfer}, and it 
requires the agent to efficiently tackle a series of tasks.
% \fh{Is it true that the tasks are always drawn from some underlying unknown distribution?}. 
A key goal of MTRL is to improve per-task learning efficiency compared against single-task learners, by using the knowledge obtained from previous tasks to learn new tasks. 
Despite the recent rapid progress in MTRL% and TL~\cite{taylor2009transfer,torrey2010transfer}
, some issues remain unsettled. 
 \emph{(1) Guaranteed sample efficiency.} % Sample comsuming.
Only a few existing methods have guarantees on sample efficiency, the most common bottleneck of RL algorithms. 
\emph{(2) Correctness v.s. efficiency.} %Negative transfer v.s. inefficient transfer.} 
An overly aggressive application of previous knowledge may transfer incorrect knowledge and deteriorate the performance on new tasks,  resulting in a ``negative transfer''~\cite{taylor2009transfer}. 
However, if an agent is overly conservative in applying previously learned knowledge, much of the similarities between tasks will be ignored, resulting in an ``inefficient transfer''.
It is nontrivial to balance between the correctness and efficiency or achieve both. 
\emph{(3) Varying state/action space across tasks.} %Transfer for changing~\fh{formal phrase} state/action space.} 
In practice, transferring knowledge learned from smaller environments to learning in larger environments is extremely useful.  
However, most existing works on MTRL assume % a shared environment or 
the state/action space is shared across tasks.% with varying dynamics or rewards only.
% If the state/action space varies, a mapping relation is usually needed. ~\fh{Do we need this, if so, give citations.}

%There is no existing method that achieves efficient and always positive transfer for unfixed state-action space with guaranteed sample complexity?~\fh{?}

%Learning models from previous tasks and reusing them appropriately is a good way of improving sample efficiency and avoiding negative transfer.

In an effort to provide guaranteed sample efficiency for MTRL, \citet{brunskill2013sample} 
% propose the first PAC-MDP to
% (Probably Approximately Correct in Markov Decision Processes)
 % \footnote{An algorithm is PAC-MDP (Probably Approximately Correct in Markov Decision Processes) if its sample complexity is polynomial in the environment size and approximation parameters with high probability.} 
propose an algorithm that clusters the underlying Markov Decision Processes (MDPs) of tasks into groups and identifies new tasks as learned groups.
%uses previously learned knowledge from the group. 
However, transferring knowledge from the clustered MDP models could be an ``inefficient transfer'' if the underlying models are too different to be clustered into a small number of groups. 
Similarly, most existing model-based approaches~\cite{liu2016pac,modi2018markov} only exploit model-level similarities, which also makes it difficult to transfer knowledge among different-sized tasks.
% to characterize the similarities across tasks.
%But it is very common that two tasks only share some similar components. In this case, no knowledge could be transferred from one to the other.
%This problem~\fh{which problem, specify. Rephrase} also exists for many other model-based approaches~\cite{liu2016pac,modi2018markov}. 
%They treat a model as a whole instead of a combination of separate parts~\fh{rephrase}. 

We remedy the aforementioned three issues by % using the prior knowledge learned from the \emph{state-action transition dynamics}, % (the transition probability combined with its reward for a state-action pair), 
extraction of more commonalities in tasks without suffering from ``negative transfer''. 
A motivating example is the navigation problem in mazes with slippery floors which result in stochastic transitions. For instance, the agent taking an action of going \emph{up} on ice could slip to the \emph{left}, \emph{right} or \emph{down} (instead of \emph{up}) 
with a certain probability determined by the slipperiness of ice.  
%, where an agent aims to reach an objective state/location in a maze with 4 actions: go up, down, left and right. 
The slipperiness of the floor %, which makes the transition dynamics stochastic, 
depends on the landform of the location, such as sand, marble and ice. 
% As a result, the number of underlying MDP models across different mazes could be large, the number of distinct landforms is small. 
We show some examples of different combinations/distributions of the landforms in the maze in Figure~\ref{fig:landforms_eg}; the MDP models are drastically different across different mazes, therefore transferring knowledge using similarity of models is inefficient. 

\begin{figure}[!hbp]
  \centering
    \begin{subfigure}[t]{0.24\columnwidth}
    \centering
    \includegraphics[width=0.8\textwidth, bb=0 0 170 170]{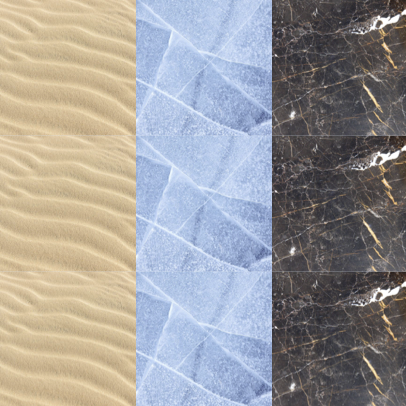}
    % \caption{3 $\times$ 3}
    \end{subfigure}
    \hfill
    \begin{subfigure}[t]{0.24\columnwidth}
    \centering
    \includegraphics[width=0.8\textwidth, bb=0 0 170 170]{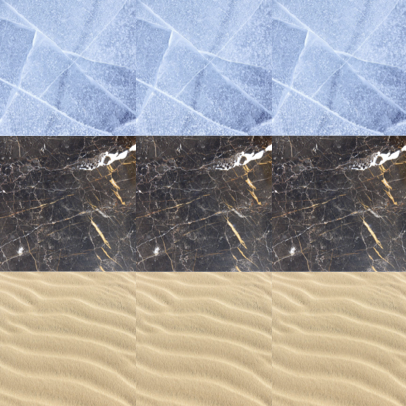}
    % \caption{3 $\times$ 3}
    \end{subfigure}
    \hfill
    \begin{subfigure}[t]{0.24\columnwidth}
    \centering
    \includegraphics[width=0.8\textwidth, bb=0 0 170 170]{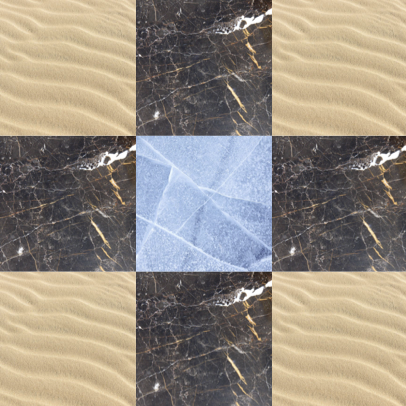}
    % \caption{3 $\times$ 3}
    \end{subfigure}
    \hfill
    \begin{subfigure}[t]{0.24\columnwidth}
    \centering
    \includegraphics[width=0.8\textwidth, bb=0 0 340 340]{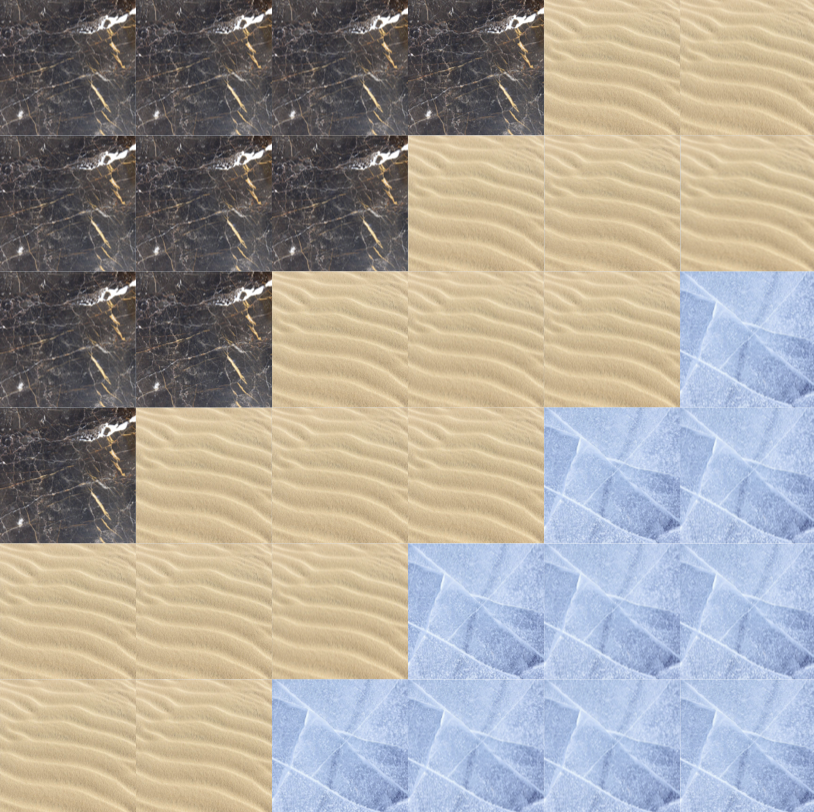}
    % \caption{6 $\times$ 6}
    % \label{sfig:larger}
    \end{subfigure}
   \vspace*{-0.5em}
    \caption{Examples of landform combinations in Maze, where \protect\includegraphics[height=0.7em, bb=0 0 500 500]{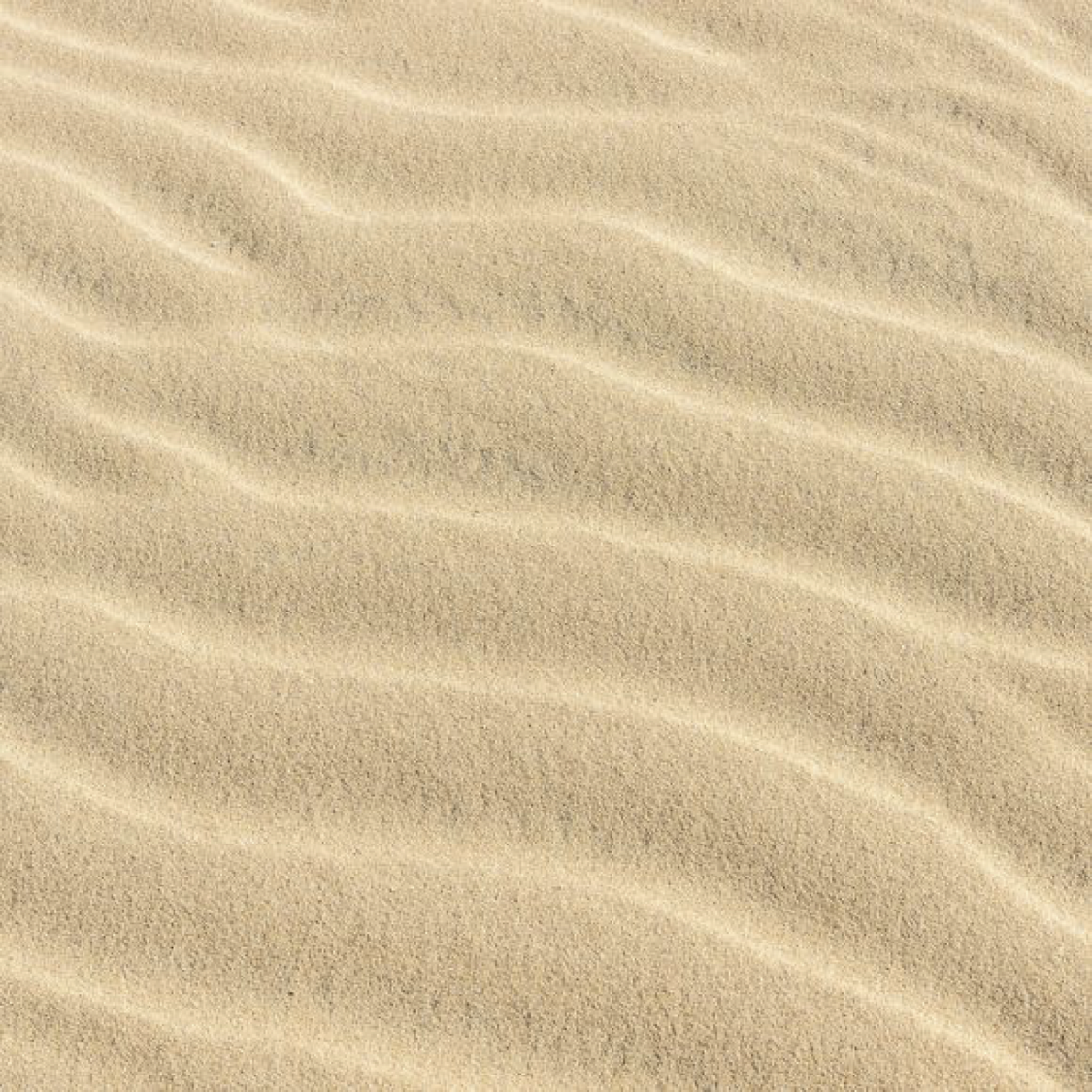} stands for sand, \protect\includegraphics[height=0.7em, bb=0 0 500 500]{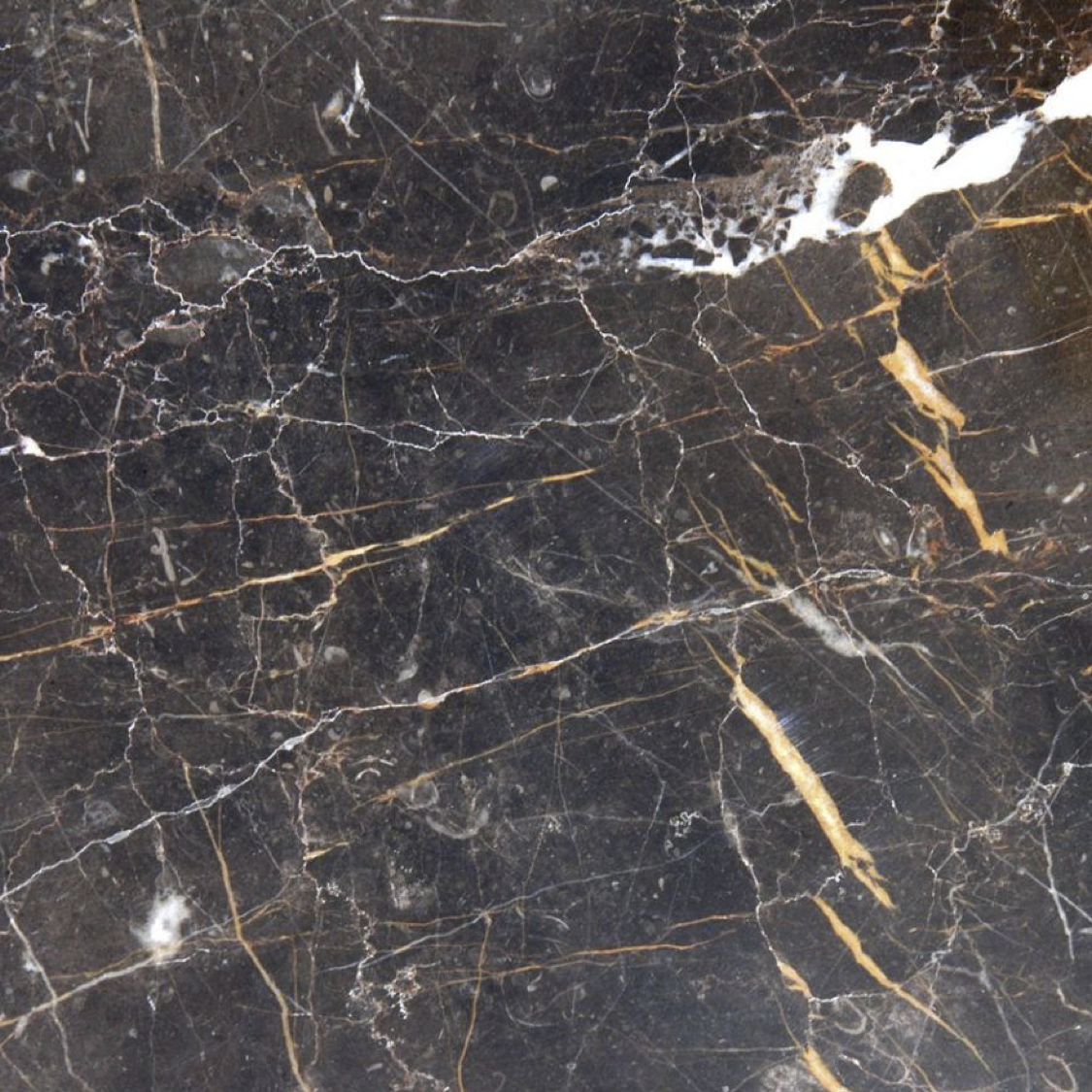} stands for marble and \protect\includegraphics[height=0.7em, bb=0 0 500 500]{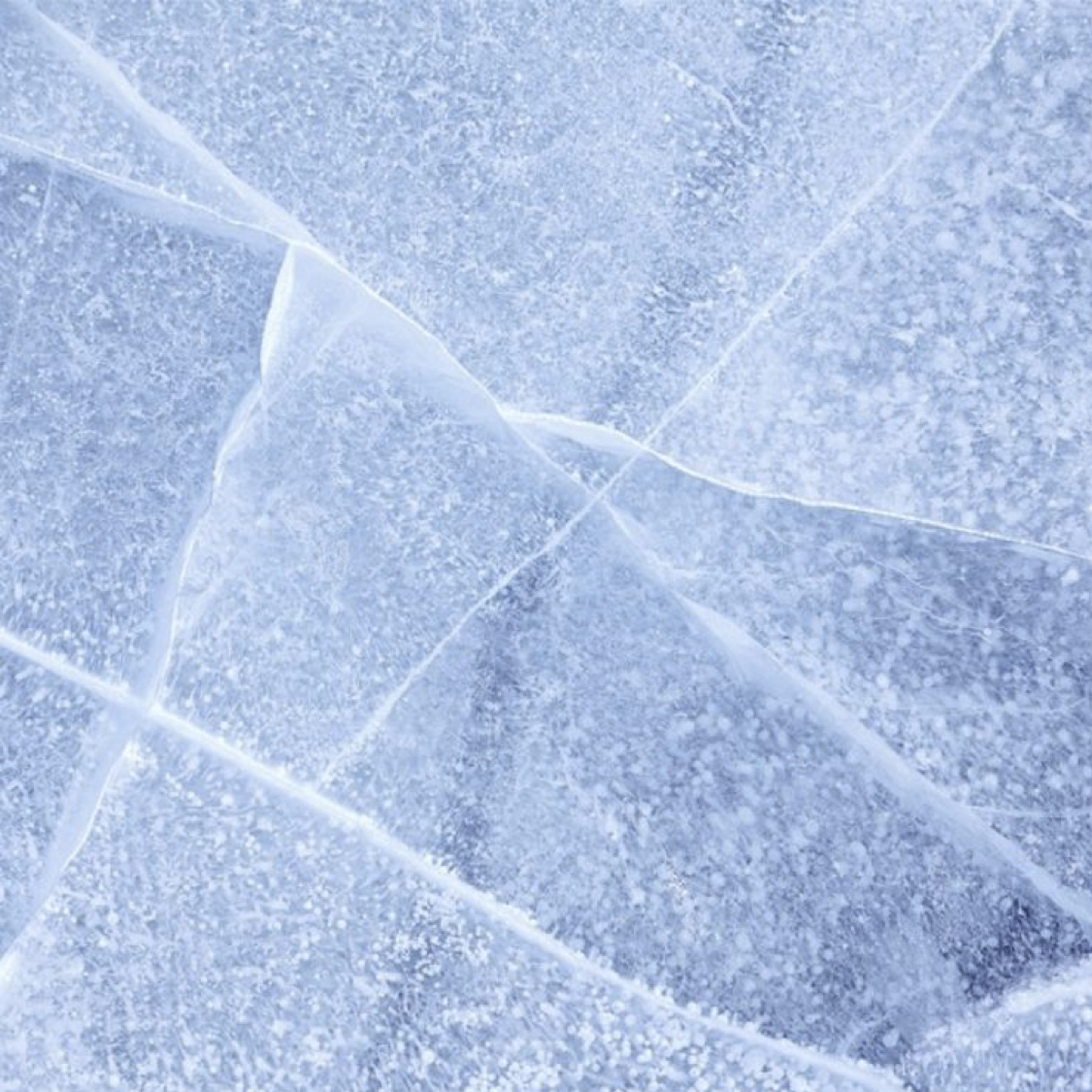} stands for ice. Different landforms have different slippery probability, thus different transition dynamics. Consider a $\sqrt{S} \times \sqrt{S}$ maze with $G$ types of landforms. 
    There could be up to $G^S$ different MDP models, making it prohibitive to extract similarities from the models. 
    However, the types of underlying transition dynamics associated with each state/location are governed by the number of distinct landforms $G$. }
    \label{fig:landforms_eg}
\end{figure}

However, our key observation is that the same landforms share the transition dynamics, and knowledge could be transferred from sand to sand, marble to marble, and ice to ice.
More importantly, we can extend the knowledge learned from a maze to any-sized mazes consisting of these same types of landforms (e.g., the 4th example in Figure~\ref{fig:landforms_eg}). 
% We remedy the aforementioned ``inefficient transfer'' problem  using the prior knowledge learned from the \emph{state-action transition dynamics}, % (the transition probability combined with its reward for a state-action pair), 
% allowing extraction of more commonalities in tasks without suffering from ``negative transfer''. 
With this idea, we achieve more effective and efficient {knowledge transfer by exploiting similarities at the level of} \emph{state-action transition dynamics} \textbf{instead of} \emph{MDP model dynamics}, allowing knowledge transfer between tasks with \textbf{varying state/action space without prior knowledge of inter-task mappings.}
The challenge of learning is now reduced to extracting such ``landforms'' without prior knowledge of the tasks.

%Motivated by the ideas above, w
We propose a novel method called \emph{Template Learning ({TempLe})} for MTRL, which provably guarantees sample efficiency and achieves efficient transfer learning for multi-task reinforcement learning with varying state/action space.
% with formal PAC guarantees. 
% \yc{these two sentences are repeated} we provide a method that provably guarantees sample efficiency and achieves efficient transfer learning without compromising accuracies for multi-task reinforcement learning with varying state/action space.
We extract templates for similar state-action transition dynamics (landforms in the example above), called \emph{\patnamefulls}, and confidently improve the efficiency of transition dynamics estimation in new tasks. 
By sharing experience among state-action pairs associated with similar templates, the learning process is expedited. 
We introduce two versions of TempLe: one is for online MTRL without prior knowledge about models, named \emph{\ourmodfull (\ourmod)}, the other further improves the learning efficiency based on a finite-model assumption, named \emph{\ourmodtwofull (\ourmodtwo)}.

% \noindent \textbf{Comparison with C-UCRL~\cite{asadi2019model}.} C-UCRL learns a single task by leveraging a state-action equivalence structure that is similar with our proposed templates. They provide an improved regret bound in the case of a known equivalence structure. However, in the more challenging case of an unknown equivalence structure, as is the setting of our paper, no regret bound is provided. 
% In contrast, our work provides a sample complexity guarantee under the unknown equivalence structure scenario. 
% % In addition, our mechanism of template extraction works between tasks with varying state/action space, a property not achieved by C-UCRL. 
% In addition, C-UCRL does not extend trivially to multi-task setting since it finds a coarse partition of all state-action pairs at every step, while in MTRL, new state-action pairs come with new tasks. Negative transfer problem may exist when the equivalence structure is unknown. 

\noindent \textbf{Summary of Contributions:}
\textbf{(1)} TempLe achieves a significant \emph{reduction of sample complexity} compared with state-of-the-art PAC-MDP (Probably Approximately Correct in Markov Decision Processes) algorithms. 
% \fh{with a sample complexity linearly dependent on the number of states and the number of templates.}
%The sample complexity is only linearly dependent on the number of states and the number of templates, when high precision is desired. \fh{why do we need to say high precision?}
\textbf{(2)} TempLe covers two realistic settings, solving MTRL problems in different regimes -- \emph{with or without prior knowledge of models}. %, that thoroughly utilize the similarities over various environments.
\textbf{(3)} To the best of our knowledge, TempLe is the \emph{first PAC-MDP algorithm} that is able to learn tasks with \emph{varying state/action spaces} without any prior knowledge of inter-task mappings.

\section{Related Work}
% \vspace*{-0.3em}
%1.Directed Exploration in RL with Transferred Knowledge~\cite{mann2012directed}
%2.Efficient RL with Relocatable Action Models~\cite{leffler2007efficient}
%3.Efficient Exploration With Latent Structure~\cite{leffler2005efficient}
%4.CORL: A Continuous-state Offset-dynamics Reinforcement Learner~\cite{brunskill2012corl}
%5.An Algebraic Approach to Abstraction in RL~\cite{ravindran2004algebraic}
%6.SMDP homomorphisms: an algebraic approach to abstraction in SMDPs~\cite{ravindran2003smdp}
%7.Using homomorphisms to transfer options across continuous RL domains~\cite{soni2006using}
%8.Transfer via soft homomorphisms~\cite{sorg2009transfer}
%9.Policy and value transfer in lifelong RL~\cite{abel2018policy}

\noindent \textbf{PAC-MDP MTRL Algorithms.}
\citet{brunskill2013sample} present the first formal analysis of the sample complexity for MTRL. 
They propose a two-phase algorithm and prove that per-task sample complexity is reduced compared with single-task learners. 
% The agent first performs single-task learning during the whole first phase, then group the tasks based on a known model gap and a known upper bound of model numbers. In the second phase, its goal is to classify every new task into a model group, then directly use the knowledge in that group. 
However, they require all tasks coming from a small number of models, and when the number of distinct models is large, their algorithm becomes similar to single-task learning. 
In this paper, we show our proposed methods outperform the method provided by~\citet{brunskill2013sample} both in theory and in experiments.
There are other PAC-MDP algorithms for multi-task RL, considering the problem from different perspectives. 
For example, \citet{brunskill2014pac} discuss lifelong learning in semi-Markov decision processes (SMDPs), where options are involved. 
\citet{liu2016pac} extend the finite-model method~\cite{brunskill2013sample} to continuous state space. 
\citet{feng2019does} and \citet{tirinzoni2020sequential} significantly reduce the sample complexity, but are under the assumption of generative models.
\citet{modi2018markov} improve the learning efficiency through the assistance of side informations. 
\citet{abel2018policy} propose MaxQInit, which transfers the maximum Q values across tasks. We empirically compare with MaxQInit in this paper.

\noindent \textbf{Reducing MDPs to Compact Ones.} There is a line of research that reduces the original MDPs to compact ones to achieve sample efficiency, including Relocatable Action Model (RAM)~\cite{leffler2007efficient}, homomorphism~\cite{ravindran2003smdp}, and $\epsilon$-equivalent MDP~\cite{even2003approximate}. However, since learning such compact structures is usually difficult (e.g., learning homomorphism is NP-hard as noted by~\citet{soni2006using}), most of the previous works require some prior knowledge.
To give a detailed comparison, \textbf{our algorithm}
\textbf{(1) requires no prior knowledge about the MDP structure.}
RAM~\cite{leffler2007efficient} requires knowledge of the ``type'' of all states (walls, pits, etc) and the next-state function of all states and type-action outcomes. Its continuous extension~\cite{brunskill2012corl} also needs knowledge of the types. Homomorphism works~\cite{ravindran2004algebraic,ravindran2003smdp,soni2006using} require knowledge of (candidate) homomorphisms to compress an SMDP or transfer knowledge between MDPs. 
\textbf{(2) works for general RL problems with PAC guarantee.}
Although~\citet{leffler2005efficient} (learns latent structure by clustering) and~\citet{sorg2009transfer} (learns soft homomorphisms) provide methods that do not require knowledge of the structure, \citet{leffler2005efficient} study a simplified non-MDP problem where actions do not influence state transitions, and~\citet{sorg2009transfer} do not provide theoretical guarantees when the target model is not known in advance.

Overall, our method is different from the above works, as we do not pre-define the compact structure. Instead, we observe that the transition dynamics, if permuted into descending order, could be naturally grouped to some template. Notably, we \emph{learn} the similarities rather than assuming knowledge of them. Our method could be more practical than the above works~\cite{leffler2007efficient,leffler2005efficient,brunskill2012corl,ravindran2004algebraic,ravindran2003smdp,soni2006using,sorg2009transfer} in multi-task RL, since a new task is often drawn randomly and knowing its structure in advance could be unrealistic.

\noindent \textbf{Comparison with C-UCRL~\cite{asadi2019model}.} C-UCRL learns a single task by leveraging a state-action equivalence structure that is similar with our proposed templates. They provide an improved regret bound in the case of a known equivalence structure. However, in the more challenging case of an unknown equivalence structure, as is the setting of our paper, no regret bound is provided. 
In contrast, our work provides a sample complexity guarantee under the unknown equivalence structure scenario. 
% In addition, our mechanism of template extraction works between tasks with varying state/action space, a property not achieved by C-UCRL. 
In addition, C-UCRL does not extend trivially to multi-task setting since it find a coarse partition of all state-action pairs at every step, while in MTRL, new state-action pairs come with new tasks, and negative transfer problem may exist when the equivalence structure is unknown.

% Additional related work discussions are in Appendix B
% ~\ref{app:related}
% \footnote{Appendix can be found on \url{https://arxiv.org/abs/2002.06659} }.

% Our proposed method re-constructs the environment
% In Appendix~\ref{sec:prelim}, standard RL notations such as states, actions, model dynamics, policies, values are reviewed following the convention. The concept of sample complexity in RL is also recalled.
%!TEX root = 0_aaai_all.tex

\section{Preliminaries and Notations}
\label{sec:prelim}
% \subsection{Single-task Learning}
\noindent \textbf{Standard RL Notations.} %In single task learning, an agent interacts with the environment described by a Markov Decision Process (MDP). 
% In this paper, we focus on tabular (discrete and finite) MDPs. 
An MDP is defined as a tuple $ \langle \mathcal{S},\mathcal{A},p(\cdot|\cdot,\cdot),r(\cdot, \cdot),\mu,\gamma \rangle$, where $\mathcal{S}$ is the state space (with cardinality $S$); $\mathcal{A}$ is the action space (with cardinality $A$); $p(\cdot|\cdot,\cdot)$ is the transition probability function with $p(s^\prime | s, a)$ representing the probability of transiting to state $s^\prime$ from state $s$ by taking action $a$; $r(\cdot, \cdot)$ is the reward function with $r(s, a)$ recording the reward achieved by taking action $a$ in state $s$; $\mu$ is the initial state distribution; $\gamma$ is the discount factor. 
Denote the maximum value of $r$ as $R_{max}$. 
Without loss of generality, suppose $0\leq r(s,a) \leq 1$ for all $(s,a)$, so $R_{\max}=1$.
Here $p(\cdot|\cdot,\cdot)$ and $r(\cdot, \cdot)$ together are the \textbf{model dynamics} of the MDP.

At every step, the agent selects an action based on the current \textit{policy} $\pi$.  
%A deterministic policy $\pi: \states \to \actions$ maps each state to an action ~\fh{No need for this sentence. Are we only focusing on deterministic policy?} ~\yc{Yes. In our algorithm description, the policy is deterministic}.
% We use the value function to evaluate how good a policy is.
The value function of a policy $V^\pi(s) $, which evaluates the performance of a policy $\pi$ , is the expected future reward gained by following $\pi$ starting from $s$.
% $$V^\pi(s) = \mathbb{E}[\sum_{h=0}^{\infty} \gamma^h r(s_h, \pi(s_h)) | s_0 = s].$$
Similarly, the action value $Q^\pi(s,a)$ is the expected future reward starting from pair $(s,a)$. 
In an RL task, an agent searches for the optimal policy by interacting with the MDP.  
% does not know the dynamics (transition probability and reward function). What it can do is to interact with the MDP, i.e., to choose one action in every step. 
%The optimal policy $\pi^*$ is the policy that achieves the largest possible value $V^*_M$. 
We use $V_{\max}$ to denote the upper bound of $V$. 
In the discounted setting $V_{\max}=\frac{R_{\max}}{1-\gamma}=\frac{1}{1-\gamma}$.

\noindent \textbf{Sample Complexity.} The general goal of RL algorithms is to learn an optimal policy for an MDP with as few interactions as possible. 
% A widely-used framework to evaluate the performance of RL algorithms is \textit{sample complexity of exploration}~\cite{kakade2003sample}, or \textit{sample complexity} for short. 
% \begin{definition}[Sample complexity of exploration]
% \label{def:sc}
For any $\epsilon>0$ and any step $h > 0$, if the policy $\pi_h$ generated by an RL algorithm $L$ satisfies $V^*-V^{\pi_h} \leq \epsilon$, we say $L$ is near-optimal at step $h$. If for any $0 < \delta < 1$, the total number of steps that $L$ is not near-optimal is upper bounded by a function $\zeta(\epsilon, \delta)$ with probability at least $1-\delta$, then $\zeta$ is called the \textit{sample complexity}~\cite{kakade2003sample} of $L$.
% \end{definition}
% Sample complexity measures the number of steps in which the agent does not act near-optimally. 

%!TEX root = 0_aaai_all.tex

\section{Learning with Templates}

% The main idea of this work is to boost the learning process by clustering state-action (s-a) pairs with similar transition probabilities and rewards. By re-ordering the transition probability vector, we get a template of transition dynamics as defined in Definition~\ref{def:pat}. We will show that the template is an effective and powerful abstraction of the environment. If we correctly group the s-a pairs with the same template, then it is possible to estimate them more accurately and more quickly via experience sharing. 

% Section~\ref{sec:setting} describes two settings of multi-task learning that our proposed two algorithms work on. Then, we formally define Dynamic Template (\patname), as well as some related concepts in Section~\ref{sec:model}. 
% By using \patnames, Section~\ref{sec:alg1} proposes an online algorithm named \ourmod, and Section~\ref{sec:alg2} introduces \ourmodtwo, a more efficient algorithm based on a finite model assumption.

% \fh{Do not mix s-a pair with transition dynamics. Sometimes you use s-a pair to refer transition dynamics. Be careful.}

%\subsection{Main Idea}\label{sec:idea}
%\yc{move to the beginning}
As motivated in the example described in Section~\ref{sec:intro}, the main idea of this work is to boost the learning process by aggregating similar state-action transition dynamics (see Definition~\ref{def:dynamics}). %~\fh{shall we call them transition dynamics?}. 
We permute the elements of transition dynamics/probability vectors to be in descending order, and aggregate these permuted transition probabilities to obtain ``templates of transition'' defined in Definition~\ref{def:pat}. 
We show that the templates are effective abstractions of the environment. 

\subsection{\patnamefull: An Abstraction of Dynamics}
\label{sec:model}
% \yc{An Abstraction of Dynamics}
% \yc{pattern to dynamic template}
% \yc{more messages}

In this section, we introduce a more compact way to represent the model dynamics of an MDP.
We first formally define the transition dynamics of a state-action (s-a) pair.

% A key goal of model-based reinforcement learning is to learn a good estimation of the dynamics. For discrete environments, we care about the dynamics for all state-action (s-a) pairs.

\begin{definition} [State-Action (s-a) Transition Dynamics]
\label{def:dynamics}
    For any state-action pair $(s,a)$, its \textbf{transition dynamics} is defined as a length-$(S+1)$ vector $\theta(s,a) = [p(s_1|s,a), p(s_2|s,a), \cdots, p(s_S|s,a), r(s,a)]$, where $S$ is the number of states.
\end{definition}

Note that s-a transition dynamics are different from the model dynamics, which characterize the transitions for all s-a pairs. 
In s-a transition dynamics, the first $S$ elements form the transition probability vector $p(\cdot|s,a)$. 
As defined in most RL literatures~\cite{kakade2003sample,brunskill2013sample}, the order of elements in $p(\cdot|s,a)$ is the natural order of the states. 
In contrast, we re-order the elements of $p(\cdot|s,a)$ by their values, and obtain a more compact representation of the transition dynamics called \emph{\patnamefull}. 

\begin{definition} [\patnamefull]
\label{def:pat}
    A \textbf{\patnamefull(\patname)} $\pattern$ is defined as a tuple $(\patp, \patr)$, where $\patp \in \mathbb{R}^{S}$ is a transition probability vector with non-increasingly ordered elements,
    i.e., $\sum_{i=1}^{S}\patsp_i = 1$ and $\patsp_i \geq \patsp_j \geq 0, \forall 1 \leq i \leq j \leq S$; $0 \leq \patr \leq 1$ is a scalar representing the reward. 
\end{definition}

Any s-a transition dynamics can be permuted to an unique \patname by re-arranging the transition probability vector $p(\cdot|s,a)$ in a decreasing order and maintaining the reward $r(s,a)$ to $\patr$, i.e., $\pattern_{(s,a)}=(\mathrm{desc}(p(\cdot|s,a)),r(s,a))$, where $\mathrm{desc}$ orders the elements of $p(\cdot|s,a)$ from the largest value to the smallest value.
For example, if $\theta(s_1,a_1)=[0.3,0.7,0,1]$, and $\theta(s_2,a_2)=[0,0.3,0.7,1]$, then $(s_1,a_1)$ and $(s_2,a_2)$ have the same \patname $([0.7,0.3,0],1)$, although their s-a transition dynamics are different.
% Thus, s-a pairs with different transition dynamics could share the same underlying \patname.

A \patname is a representation of multiple s-a transition dynamics with some similarities.
It ignores how the s-a pair transits to a specific next state, but only considers the  
%relative differences
patterns of transition probabilities, allowing more efficient exploitation of similarities. 
% which may be governed by a small number of general regularities. 
 % There are only 2 \patnames in this environment. So learning \patnames is much more efficient than learning the individual dynamics separately.
An intuitive example is given in Figure 4 
% ~\ref{fig:pat_eg} 
in Appendix A\footnote{Appendix can be found on \url{https://arxiv.org/abs/2002.06659} }.,
% ~\ref{app:examples}, 
 where there are 100 distinct s-a transition dynamics, but only 2 distinct \patnames. Appendix F.5 
% ~\ref{app:universal} 
 further discusses the universal existence of such similarities.

\subsection{Empirical Estimation of \patnamefulls}
\label{sec:estimation}
Section~\ref{sec:model} defines \patname based on the underlying s-a transition dynamics. However, in reality, we do not have access to the underlying dynamics. In model-based RL, a key step is to estimate the dynamics and to build a model of the environment. 
We now illustrate the estimation of \patnames, as well as how \patnames augments the learning process.

\noindent\textbf{The conventional estimation of s-a transition dynamics. }
A direct estimate of $\theta(s,a)$ is obtained through experience,
$
\hat{\theta}(s,a) = [\frac{n(s,a,s_1)}{n(s,a)}, \frac{n(s,a,s_2)}{n(s,a)},\cdots, \frac{n(s,a,s_S)}{n(s,a)},\frac{R(s,a)}{n(s,a)}],
$
where $n(s,a,s^\prime)$ is the number of observations of transitioning from $s$ to $s^\prime$ by taking action $a$, $n(s,a)$ is the total number of observations of $(s,a)$, and $R(s,a)$ is the cumulative rewards obtained by $(s,a)$.
An accurate estimate of the transition dynamics $\theta(s,a)$ requires a large enough number of observations $n(s,a)$ according to the theory of concentration bounds. 
%$\hat{\theta}(s,a)$ can confidently approximate $\theta(s,a)$ when $n(s,a)$ is very large. 
Therefore, it is sample-consuming to accurately estimate the transition dynamics of each s-a pair in this way.

\noindent\textbf{Augmented estimation of s-a transition dynamics.}
As discussed in Section~\ref{sec:model}, different s-a pairs may share the same \patnames. 
Our goal is then to aggregate the estimations of s-a transition dynamics associated with the same \patnames. 
%Assume we are able to correctly convert the noisy estimation of every s-a transition dynamics to its true \patname. Then we can execute 
We introduce the following process to obtain estimates of all s-a transition dynamics: \\
\emph{(1) rough estimation}: obtain $\widehat{\theta}(s,a)=[\frac{\mathbf{n}(s,a,\cdot);R(s,a)}{n(s,a)}]$ for each $(s,a)$ with a small $n$; \\
\emph{(2) permutation}: permute each $\widehat{\theta}(s,a)$ to its corresponding permuted estimates $\tilde{\myvector{g}}_{(s,a)}$; \\
% \fh{not g, you need a new notation to denote the permuted estimations, i used g tilde}; \\
% $\pattern_{(s,a)}$ 
\emph{(3) template identification}: identify the group of the permuted estimate $\tilde{\myvector{g}}_{(s,a)}$ such that permuted estimates are similar within the group, and obtain a more confident estimate of \patname $\widehat{\pattern}$ aggregating within-group statistics.\\% using the total number of visits $N(s,a,\cdot)$ within each cluster. \yc{this is not accurate. $\pattern$ does not have specific $s$ and $a$. It is re-ordered. $N$ is also re-ordered.}\\
%and get the \patname set $\hat{\patset}=\{\hat{\pattern}_1,\cdots,\hat{\pattern}_G\}$ with estimation $\hat{\pattern}_i = N_i(\cdot)/N_i$, where $N_i$ is the total number of visits of $(s,a)$'s corresponding to $\pattern_i$; \\
% \emph{(3) clustering}: cluster $\pattern_{(s,a)}$\fh{hat g} of all $(s,a)$'s, and get the \patname set $\hat{\patset}=\{\hat{\pattern}_1,\cdots,\hat{\pattern}_G\}$ with estimation $\hat{\pattern}_i = N_i(\cdot)/N_i$, where $N_i$ is the total number of visits of $(s,a)$'s corresponding to $\pattern_i$; \\
\emph{(4) augmentation}: for every $(s,a)$, obtain a more confident estimate of the transition dynamics by permuting back its corresponding \patname with accumulated knowledge.

The noisy estimate of transition dynamics will not render error other than the smaller amount of noise in estimated transition templates if it is identified into the right group. To guarantee accurate identification, the ordering of the elements in the noisy estimate should be consistent with the ground truth. 
Therefore, the consistency of our estimation depends on \patname gap as defined in Definition 5 
% ~\ref{def:pat_gap} 
and ``\probgapname'' as defined in Definition 8 
% ~\ref{def:prob_gap} 
(see Appendix D 
% ~\ref{app:def} 
for details). An example in Appendix A.1 
% ~\ref{app:example} 
shows how augmented estimation helps save a large number of samples compared against the conventional estimation.

% The above process is guaranteed to succeed if we can correctly group and permute the noisy estimation, which requires the knowledge of the \patname gap defined in Definition~\ref{def:pat_gap}, and a similar concept ``\probgapname'' defined in Definition~\ref{def:prob_gap} (See Appendix for details). 

Now we are ready to formally introduce our algorithms in two settings, Online MTRL and Finite-Model MTRL.

\begin{algorithm}[!tbp]
   \caption{\ourmodfull (\ourmod)}
   \label{alg:main}
   \hspace*{\algorithmicindent} \textbf{Input:} 
   user-specified \patname gap $\hat{\tau}$; error tolerance $\epsilon$; discount factor $\gamma$; regular known threshold $\threslg$; small known threshold $\thressm$ \\
   \hspace*{\algorithmicindent} \textbf{Output} 
   Near-optimal policies $\{\pi_t\}_{t=1,2,\cdots}$
\begin{algorithmic}[1]
   \State Initialize an empty \patname group set $\patset$ and \patname visit set $\patvisset$    
   \For{$t \gets 1,2,\cdots$}
    \State Receive a task $M_t$ 
    \State Initialize visits $\bm{n}(s,a,\cdot)\gets\bm{0}$, accumulative rewards $R(s,a)\gets 0$, $\forall (s,a) \in (\states, \actions)$, an empty known state-action set $\kset$, and an initial policy $\pi$
    \For{$h \gets 1,2,\cdots,H$}
     % \STATE Action $a_h \gets \pi(s_h)$ 
      % \State \Call{act-and-update}{$s_h$} \Comment{{\scriptsize{Procedure \Call{act-and-update}{} in Appendix}}}
      \State Take action $a_h \gets \pi(s_h)$, get $s_{h+1}$ and $r_h$
      \State Update visits $n(s_h,a_h,s_{h+1})$ and $R(s_h,a_h)$
      % \setlength\abovedisplayskip{2pt}
      % \setlength\belowdisplayskip{2pt} 
      % \begin{align*}
      % &n(s_h,a_h) \leftarrow n(s_h,a_h)+1 \\
      % &n(s_h,a_h,s_{h+1}) \leftarrow n(s_h,a_h,s_{h+1})+1 \\
      % &R(s_h,a_h) \leftarrow R(s_h,a_h)+r_h
      % \end{align*}
      % \State {\small{$n(s_h,a_h) \leftarrow n(s_h,a_h)+1$, $n(s_h,a_h,s_{h+1}) \leftarrow n(s_h,a_h,s_{h+1})+1$, $R(s_h,a_h)\leftarrow R(s_h,a_h)+r_h$}}
      \If{ $(s_h,a_h)\! \notin\! \kset $ \textbf{and}  $ \lVert \bm{n}(s_h,a_h,\cdot) \rVert_{\ell_1}\! =\! \thressm $}  \Comment{{\small{\patname identification with the small threshold}}}
        \State {\small{$\tilde{\pattern}, \patvis_{\tilde{\pattern}}, \sigma \gets$ \Call{gen-TT}{$\bm{n}(s_h,a_h,\cdot),R(s_h,a_h)$} }} % \Comment{{\scriptsize{Procedure \Call{gen-TT}{} in Appendix}}}
        \If{no $\pattern \in \patset$ is $\hat{\tau}$-close to $\tilde{\pattern}$}
          % \COMMENT{Initialize a new \patname}
          % \State \tcb{\# Initialize a new \patname}
          \State Add $\tilde{\pattern}$ to $\patset$,
          %\State $\patvisp_{\tilde{\pattern}}=desc(\bm{n}(s_h,a_h,\cdot))$
          %\State $\patvisr_{\tilde{\pattern}}=R(s_h,a_h)$
      % $\patvis_{\tilde{\pattern}}=(\patvisp_{\tilde{\pattern}}, \patvisr_{\tilde{\pattern}})$ 
              $\patvis_{\tilde{\pattern}}$ to $\patvisset$
        \Else
          \State Find the closest \patname $\pattern^*$ to $\tilde{\pattern}$ 
          \State {\small{\Call{\patname-update}{$\pattern^*, \patvis_{\pattern^*}, \bm{n}(s_h,a_h,\cdot), R(s_h,a_h)$}}} %\Comment{{\scriptsize{Procedure \Call{\patname-update}{} in Appendix}}}
          \State {\Call{augment}{$\patvis_\pattern^*$,$\bm{n}(s_h,a_h,\cdot)$,$R(s_h,a_h)$,$\sigma$}} %\Comment{{\scriptsize{Procedure \Call{augment}{} in Appendix}}}
          % \State \tcb{\# Add current experience to \patname}
          % \State $\patvis_{\pattern^*} \leftarrow \patvis_{\pattern^*} + (desc(\bm{n}(s,a,\cdot)),R(s,a))$
          % \State $\pattern^* \leftarrow (\patvisp_{\pattern^*}/\sum \patvisp_{\pattern^*}, \patvisr_{\pattern^*}/\sum \patvisp_{\pattern^*})$
          % \State \tcb{\# Add \patname experience to current task}
          % \State Get the permutation $\permut$ that orders $\bm{n}(s,a,\cdot)$ descendingly
          % \State $\bm{n}(s,a,\cdot) \leftarrow \bm{n}(s,a,\cdot) + \permut^{-1}(\patvisp_{\pattern^*})$,\\ $R(s,a) \leftarrow R(s,a) + \patvisr_{\pattern^*}$, \\ $n(s,a) \leftarrow n(s,a) + \sum_{s^\prime} \bm{n}(s,a,s^\prime)$
           % \fh{Update $\pattern^*\leftarrow [\frac{n_{\pattern^*}(s,a,s_1)}{n_{\pattern^*}(s,a)},\ldots,\frac{n_{\pattern^*}(s,a,s_S)}{n_{\pattern^*}(s,a)}]$}
        \EndIf
        % \State Update threshold $A(s_h,a_h) \gets \threslg$ 
      \EndIf
      \If{ $ (s_h,a_h)\! \notin\! \kset $ \textbf{and}  $ \lVert \bm{n}(s_h,a_h,\cdot) \rVert_{\ell_1}\! \geq\! \threslg $} \Comment{{\small{policy update with the regular threshold}}}
        \State Update $\pi$ using visits $\bm{n}$ and $R$ by RMax
        \State add $(s_h,a_h)$ to $\kset$
        % \State 
      \EndIf
    \EndFor
    \For{all $(s,a) \in (\states, \actions)$ with identified \patname $\pattern_{(s,a)}$}
      \State \Call{\patname-update}{$\pattern_{(s,a)}, \patvis_{\pattern_{(s,a)}}, \bm{n}(s,a,\cdot), R(s,a)$} %  \Comment{{\scriptsize{Procedure \Call{\patname-update}{} in Appendix}}}
      % \State\tcb{\# Add the rest current experience to \patname}
      % \State$\patvis_{\pattern_{(s,a)}} \leftarrow \patvis_{\pattern_{(s,a)}} + (desc(\bm{n}(s,a,\cdot)),R(s,a))$
      % \State$\pattern_{(s,a)} \leftarrow (\patvisp_{\pattern_{(s,a)}}/\sum \patvisp_{\pattern_{(s,a)}}, \patvisr_{\pattern_{(s,a)}}/\sum \patvisp_{\pattern_{(s,a)}})$
    \EndFor 
    \EndFor
    % \Procedure{act-and-update}{$s_h$} 
    %   \State Take action $a_h \gets \pi(s_h)$, get $s_{h+1}$ and $r_h$
    %   % \State $n(s_h,a_h) \leftarrow n(s_h,a_h)+1$
    %   \State $n(s_h,a_h,s_{h+1}) \leftarrow n(s_h,a_h,s_{h+1})+1$
    %   \State $R(s_h,a_h)\leftarrow R(s_h,a_h)+r_h$
    % \EndProcedure
\end{algorithmic}
% \begin{algorithmic}[1]
%    \STATE{\bfseries Function} Update($s,a,s^\prime,r$)
%    \STATE Update $n(s,a),n(s,a,s^\prime), r(s,a)$ 
%     \IF{$n(s,a)\geq A(s,a) \wedge A(s,a) = \thressm$}
%       \STATE $\tilde{\pattern} \gets Patternize(\hat{p}(s,a),\hat{r}(s,a))$ 
%       \IF{no $\pattern \in \patset$ is $\tau$-close to $\tilde{\pattern}$}
%         \STATE Add $\tilde{\pattern}$ to $\patset$ with visits 
%       \ELSE
%         \STATE Find the closest pattern $\pattern^*$ to $\tilde{\pattern}$ 
%         \STATE Add visits of $(s,a)$ to $\pattern^*$ 
%         \STATE Incorporate all visits of $\pattern^*$ to $(s,a)$
%       \ENDIF
%       \STATE $A(s,a) \gets \threslg$ 
%     \ENDIF
%     \IF{$n(s,a)\geq A(s,a) \wedge A(s,a) = \threslg$}
%       \STATE $\pi \gets SingleRMax(M_t, A)$
%     \ENDIF
% \end{algorithmic}
\end{algorithm}
%!TEX root = ../0_aaai_all.tex

\begin{algorithm}[!ht]
   \caption{\patname Functions}
   \label{alg:tt_func}
\begin{algorithmic}[1]
    \Function{gen-TT}{$\bm{n}, R$} \Comment{{\small{generate \patname}}}
    	\State find permutation $\sigma$ s.t. $\sigma(\bm{n})$ is in descending order
        \State ordered visits {\small{$\patvisp_{\pattern}$$\gets$$\sigma(\bm{n})$,$\patvisr_{\pattern}$$\gets$$R$,$\patvis_\pattern$$\gets$($\patvisp_{\pattern}$,$\patvisr_{\pattern}$)}} 
        \State \patnamelowercase $\pattern \gets (\frac{\patvisp_{\pattern}}{\lVert \bm{n} \rVert_{\ell_1}}, \frac{\patvisr_{\pattern}}{\lVert \bm{n} \rVert_{\ell_1}})$
      \State \textbf{return} $\pattern, \patvis_\pattern, \sigma$
    \EndFunction
    \Function{\patname-update}{$\pattern, \patvis_\pattern, \bm{n}, R$} \Comment{{\small{add visits to \patname}}}
      \State $\patvis_{\pattern} \gets \patvis_{\pattern} + (descending(\bm{n}),R)$
      \State $\pattern \gets (\frac{\patvisp_{\pattern}}{\lVert \patvisp_{\pattern} \rVert_{\ell_1}}, \frac{\patvisr_{\pattern}}{\lVert \patvisp_{\pattern} \rVert_{\ell_1}})$
    \EndFunction
    \Function{augment}{$\patvis_\pattern, \bm{n}, R,\sigma$} \Comment{{\small{augment visits by \patname}}}
      %\State Get the permutation $\permut$ that orders $\bm{n}$ descendingly
      \State $\bm{n} \gets \bm{n} + \permut^{-1}(\patvisp_{\pattern})$
      % , $n \gets n + \lVert \bm{n} \rVert_{\ell_1}$%~\fh{shouldn't $n$ be $\lVert \bm{n} \rVert_{\ell_1}$?}
      \State $R \gets R + \patvisr_{\pattern}$ 
    \EndFunction
\end{algorithmic}
\end{algorithm}
\subsection{\ourmod: \ourmodfull}
\label{sec:alg1}
%\textbf{Online MTRL. }
In the \textit{online MTRL setting}, an agent interacts with multiple tasks streaming-in, each of which corresponding to a specific MDP. The tasks are i.i.d. drawn from a set $\mathcal{M}$ of MDPs (models). MDPs in $\mathcal{M}$ may have different %dynamics or 
state/action spaces. The number of MDPs $|\mathcal{M}|$ can be arbitrarily large.
% The learning process is: for $t=1,2,\cdots$, the agent (1) receives a new task $M_t \in \mathcal{M}$, (2) chooses a policy $\pi_t$, and (3) experiences $M_t$ till the end of the task.

% \fh{Try shortening the part if no space.}
% Assume the minimal pattern difference gap $\tau$ is known.
% Procedure~\ref{alg:main} works for general cases where pattern difference gap is known.
% \yc{Explanations. RMax as single task learning algorithm}

% Procedure~\ref{alg:main} illustrates how to utilize patterns in online multi-task learning. 
We introduce \ourmodfull (\ourmod) for the online MTRL setting. 
\ourmod is a meta-learning algorithm with model-based ``\emph{base learners}'' which compute policies for the current task. 
We use RMax~\cite{brafman2003rmax} as the base learner, and it can be replaced by other model-based methods such as $E^3$~\cite{kearns2002near} and MBIE~\cite{strehl2005a}. 
The principle of RMax algorithm on an MDP $M$ is to build an induced MDP based on a known threshold $m$. 
A state-action pair is said to be $m$-known if the number of visits/observations $n(s,a) \geq m$. A state is $m$-known if $n(s,a) \geq m, \forall a\in \actions$. The set of all $m$-known states induces an MDP $M_k$, where for any $m$-known state $s$, $p(s^\prime|s,a) = \frac{n(s,a,s^\prime)}{n(s,a)}$, $r(s,a) = \frac{R(s,a)}{n(s,a)}$ 
% \begin{equation}
% p(s^\prime|s,a) = \frac{n(s,a,s^\prime)}{n(s,a)}, \quad r(s,a) = \frac{R(s,a)}{n(s,a)},
% \end{equation}
and for any non-$m$-known state $s$, $p(s^\prime|s,a) = \mathbb{I}\{s^\prime=s\}$, $r(s,a) = R_{\max}$.
Then, RMax computes an optimal policy based on the optimistic model by dynamic programming.

In contrast, \ourmod uses augmented estimation introduced in Section~\ref{sec:estimation}. to reduce the required number of visits to every single s-a pair. 
% by clustering the experience of s-a transition dynamics with the same \patnames.
Instead of aggregating the estimates of all s-a transition dynamics at once, \ourmod asynchronously identifies the \patnames of s-a pairs and updates the template groups in an online manner, through measuring the distances among \patnames.

% \begin{equation}
% p(s^\prime|s,a) = \mathbb{I}\{s^\prime=s\}, \quad r(s,a) = R_{\max}.
% \end{equation}
% The known threshold $m$ is a key element in RMax, since it illustrates the number of visits needed to an s-a pair before trusting its estimation. 
% To get an accurate model estimation, $m$ needs to be a relatively large number. 
% However, if some s-a transition dynamics are grouped, then they share all visits and the total number of visits will decrease~\fh{Why?}.

% In \ourmod, for each s-a pair, we first set its known threshold to be a relatively small number, $\thressm$, which is the smallest number to ensure correct grouping of s-a transition dynamics.
% If an s-a pair has been visited over $\thressm$ times, it is ``\emph{roughly known}''.
% % We call this status \textit{roughly know}.
% After $n(s,a)$ reaches $\thressm$, it will be permuted and grouped to identify its corresponding \patname.
% \fh{I think this is unclear: 
% then we re-set its known threshold to be a larger and regular one $\threslg$, to make sure we collect enough experience with $(s,a)$. Once the agent has visited $(s,a)$ for more than $\threslg$ times, we say $(s,a)$ is \textit{fully known}.}
% \fh{I change it to:}
% The estimation of \patname uses the accumulated visits to the state-action pair from all previous and current tasks $N(s,a)$. 
% When the accumulated visits $N(s,a)$ reaches a larger threshold $\threslg$ (which is easy to achieve), the estimation of the transition dynamics associated with this \patname is confidently more accurate. 

Algorithm~\ref{alg:main} illustrates how \ourmod works.
In addition to the regular known threshold $\threslg$ used in RMax, we design a smaller known threshold $\thressm$, which is the smallest number of visits to ensure identifying the \patnames of all s-a pairs.
% Instead of a single known threshold controlling all s-a transition dynamics, we maintain a known threshold for every individual s-a transition dynamics, stored in the table $A$, so that we update the s-a transition dynamics asynchronously. 
If for any $(s,a)$, the total number of visits ($\lVert \bm{n}(s,a,\cdot) \rVert_{\ell_1}$) reaches $\thressm$, then the estimated \patname $\tilde{\pattern}$ of $(s,a)$ will be generated by function \textsc{GEN-TT}. If $\tilde{\pattern}$ has at least $\hat{\tau}$-distance with all existing \patnames, we regard it as a new \patname and append it to set $\patset$ (Line 10-11); otherwise (Line 12-15), we find the closest \patname to $\tilde{\pattern}$, then synchronize the experience of $(s,a)$ in the current task and the accumulated experience that its \patname holds by calling functions \textsc{TT-update} and \textsc{augment}, which respectively send the current visits of $(s,a)$ to the corresponding \patname, and feed the accumulative visits of the \patname to the current $(s,a)$.
\textsc{GEN-TT}, \textsc{TT-update} and \textsc{augment} involve the permutation operations, and are given by Algorithm~\ref{alg:tt_func}. 
Accumulated experience of each \patname is stored in a tuple $\patvis_\pattern=(\patvisp_\pattern, \patvisr_\pattern)$, where $\patvisp_\pattern$ is the total visits accumulated by permuted $\textbf{n}(s,a,s^\prime)$ of all $(s,a)$'s with \patname $\pattern$.
% ~\fh{why? I changed it. We maintain the accumulated visits $N_{\pattern^*}(s,a)$. Please introduce this accordingly in the main text}. \yc{We don't do this. We cannot maintain $N_{\pattern^*}(s,a)$ because $\pattern$ ignores states and actions.}
% Once the \patname of $(s,a)$ is known, we augment $(s,a)$ with the historical visits related with the \patname.
% ~\fh{change this sentence. First don't use $\permut(s,a)$ as it is not defined. Second the phrase incorporate is not consistent with the previous description. Use augmentation instead.}. 
% Once the \patname of $(s,a)$ is known, we update the known threshold of $(s,a)$ to be $\threslg$ (Line 21), and continue learning with augmented visits of $(s,a)$. 
When $(s,a)$ is $m$-known, the policy is updated (Line 16-18).
% As historical visits accumulates, this process of having $m$ visits for an s-a pair will be significantly accelerated.
Overall, our \ourmod allows grouped s-a transition dynamics to share their visit counts, making it much easier for them to reach $\threslg$ visits than in regular RMax. 

Note that Algorithm~\ref{alg:main} also works for tasks with varying state/action space, since the comparison of \patnames considers the non-zero elements of the transition vectors only. 
One can compute the difference between two different-sized \patnames by simply padding zeros to the end of the shorter \patname.

%%%%%%%%%%%%%%%%%%%%%%%%%%%%%%%%%%%%
\subsection{\ourmodtwo: \ourmodtwofull}
\label{sec:alg2}
%\textbf{Finite-Model MTRL. } 
Online MTRL setting requires no prior knowledge of the types of underlying MDPs and improves the sample efficiency by accumulating knowledge with \patname groups. 
% As a result, the  \ourmod agent still needs to visit all s-a pairs before the model dynamics become known, although sample efficiency is improved by accumulating knowledge within \patname groups. 
However, under a more restrictive assumption that the number of possible MDPs $C=|\mathcal{M}|$ is known and small, it is possible to get rid of the dependence on the size of state-action space and achieve \emph{more efficient learning}.

% We consider the Finite-Model MTRL setting, where the agent still interacts with streaming-in tasks drawn from a set $\mathcal{M}$ of MDPs, but the number of MDPs in the set $|\mathcal{M}|$ is small and known. 

%Finite-model MTRL assumes small types of MDPs, which could potentially allow more efficient learning. 
%{A key difference is the latter requires knowledge of number of models.}
 %This is natural for model-based algorithms. 

  %This idea, proposed by~\cite{brunskill2013sample}, is called Finite-Model RL (FMRL), which makes this ``finite model'' assumption and achieves better sample complexity than single-task learners.  \fh{This paragraph and the paragraph below need to be rewritten. It doesn't clearly state our improvement and difference from FMRL. }

We propose \ourmodtwofull(\ourmodtwo), an extension of our \ourmod, under the \textit{finite-model MTRL setting}, where the agent still interacts with streaming-in tasks drawn from a set $\mathcal{M}$ of MDPs, but the number of MDPs in the set $\mathcal{M}$ is small and known. 
% \ourmodtwo keeps all techniques that \ourmod has. 

In contrast with \ourmod, \ourmodtwo is able to correctly identify the \patnames of some s-a pairs before they are visited for $\thressm$ times. 
% This is because for a fixed MDP (model), how \patnames are distributed over all s-a pairs is also fixed. 
This is because the number of underlying models is small, and thus identifying the model is easy and inexpensive.  
% If the number of underlying models is small, 
% then it is possible to quickly identify which model a task belongs to, and then we immediately know the \patnames for all s-a pairs.
It is possible to obtain the \patnames for all s-a pairs immediately after identifying the model, since the way how \patnames are distributed over all s-a pairs is fixed for each MDP model.

% As a result, once all types of MDP models and their \patnames of all s-a pairs are learned, the \patnames of all s-a pairs in a new task can be known without further learning once it is grouped to a known model. 

The main steps of \ourmodtwo are stated below, and the details are illustrated in Algorithm 3
% ~\ref{alg:group}
 in Appendix C.
% ~\ref{app:algo}.
% has two phases; the first phase (tasks $1 - T_1$) collects models and the second phase (tasks $T_1+1,T_1+2\cdots$) identifies models.
% In Phase 1 (Line 2-3), the agent acts in the same way as \ourmod.
% At the end of Phase 1 (Line 4-5), the first $T_1$ tasks are clustered into finite groups of models. 
% Then, in Phase 2 (Line 6-17), the agent still follows \ourmod and identifies \patname, but also tries to find the true model for the current task from all candidate models, by ruling out the models that are very different with the current one.
\textit{(1) Collecting Models:} for the first $T_1$ tasks, the agent acts in the same way as \ourmod, but also stores the \patname structure of each model.
\textit{(2) Grouping Models:} the first $T_1$ tasks are clustered into finite groups of models based on their \patname structures.
\textit{(3) Identifying Models:} for any new task, the agent still follows \ourmod, but also seeks the true model for the current task from all the model groups, by ruling out the groups of models that have different \patname structures.

% $u(c)$ is a model score of group $c \in \mdpset$ which measures how possible $c$ is the true model for the current task \fh{comment in the algorithm itself}.
% $\pattern_{(s,a,c)}$ and $\permut_{(s,a,c)}$ in Line 19 are the \patname and permutation of $(s,a)$ in group $c$. One mismatch~\fh{what mismatch?} leads to a reduction of $u(c)$~\fh{not a good english. Just say model score $u(c)$ descreases.}, and a negative $u(c)$ means $c$ is not likely to be the true model for the current task. $\eta$ is the model error tolerance, and it makes sure that even though the agent classifies an $(s,a)$ into a wrong \patname, the right $c$~\fh{c is not defined here} will not be immediately ruled out.

% ~\fh{we need to connect these two settings better here. Right now, the writing seems like to suggest that these two are completely different.} 

\citet{brunskill2013sample} make the same finite-model assumption and propose an algorithm FMRL which extracts model similarities. However, FMRL can not transfer knowledge between two models which are the same except for one state-action pair.
In contrast, our \ourmodtwo extracts state-action dynamics similarities and thus transferring happens among any state-action pairs that have similar dynamics.
Compared with FMRL, \ourmodtwo not only has lower sample complexity as proved in Section~\ref{sec:sc_group}, but also saves computations due to the direct comparison of \patnames.
%We will prove and verify that \ourmodtwo outperforms FMRL.
% when the number of possible MDPs, denoted by $C$, is known and smaller than $\sqrt{SA}$.

% \fh{It is not clear how clustering is performed in FM-TEmpLe. What is the function u(c) exactly? More details of this step are important to allow reproducibility.}

%!TEX root = 0_aaai_all.tex
\section{Theoretical Analysis}
\label{sec:analysis}
This section provides sample complexity analysis of the proposed two algorithms \ourmod and \ourmodtwo. % (Procedure~\ref{alg:main} and Procedure~\ref{alg:group}). 
Although \ourmod and \ourmodtwo can be applied to tasks with varying state/action spaces, we assume all tasks have the same $\states$ and $\actions$ for simplicity of notations, and the analysis extends to varying state/action spaces trivially.

% The distance between two different \patnames is defined as the $\ell_2$ distance between $\patp$'s plus the absolute difference between $\patr$'s. See Definition~\ref{def:pat_gap} Appendix~\ref{app:def} for more details. %\fh{consider moving to appendix}

We first assume there is a diameter $D$ such that any state $s^\prime$ is reachable from any states $s$ in at most $D$ steps on average. This assumption is commonly used in RL~\cite{jaksch2010near}, and it ensures the reachability of all state from any state on average.

We further define the underlying minimal $\ell_2$-distance among \patnames as $\tau$, namely \patname gap.
We also define $\nu$ as the ranking gap; a large ranking gap implies that for any s-a pair, the probabilities of transitioning to any two states are substantially different.
%either very close, or substantially different. 
For any $\pattern \in \patset$, if $\patp_i > \patp_j$ are two adjacent elements in $\patp$, then either $\patp_i - \patp_j \geq \nu$, or $\patp_i - \patp_j \leq \mathcal{\tilde{O}}(\frac{\epsilon (1-\gamma)}{\sqrt{S}V_{\max}})$(logarithmic terms are hided in $\mathcal{\tilde{O}}(\cdot)$). The ranking gap implies that for any s-a pair, the probabilities of transitioning to any two states are either very close, or substantially different. 
Note that the algorithms take a user-specified $\tau$, but do not require input of $\nu$.
% is used in analysis only, not in the algorithms. 
See Appendix D
% ~\ref{app:def}
 for formal definitions of \patname gap and ranking gap.
For notation simplicity, let $\mingap$ denote $\max \{ \min (\tau,\nu), \mathcal{O}(\frac{\epsilon (1-\gamma)}{\sqrt{S}V_{\max}})\}$.

\begin{theorem}[\textbf{Sample Complexity of \ourmod}]
\label{thm:main}
For any given $\epsilon > 0$, $1 > \delta > 0$, running Algorithm~\ref{alg:main} on $T$ tasks, each for at least $\mathcal{O}(\frac{DSA}{\mingap^2} \ln\frac{1}{\delta})$ steps, generates 
at most
% \setlength\abovedisplayskip{1pt}
% \setlength\belowdisplayskip{1pt}
% \begin{equation}
% \label{eq:main}
$\mathcal{\tilde{O}} \Big(\frac{S G V^3_{\max}}{\epsilon^3 (1-\gamma)^3} + \frac{TSA V_{\max}}{\mingap^2 \epsilon (1-\gamma)} \Big)$
% \end{equation}
non-$\epsilon$-optimal steps, with probability at least $1-\delta$, where $G$ is the total number of \patnames.
\end{theorem}
% \vspace{-0.8em}

\noindent\textbf{Remark.} 
(1) Our provided bound achieves \emph{state-of-the-art} dependence on the environment size $T,S,A$ for general MTRL, given that $G$ is independent of $T,S,A$. 
% \fh{is it true? Did we compare with the most recent paper pointed by the reviewers?}
%(1) No existing methods solve the online MTRL setting with sample efficiency guarantees. We are the first to provide the sample complexity of online MTRL.
(2) When $\epsilon$ is small, the sample complexity only has a linear dependence on the number of states $S$ and the number of templates $G$, because the first term dominates. 
By definition, $G$ is always no larger than $TSA$, the number of all s-a pairs. And in most environments, we have $G \ll TSA$, as discussed in Appendix F.5.
% ~\ref{app:universal}.
(3) When $\epsilon$ is not small or $T$ is very large, the sample complexity has linear dependences on $T$, $S$ and $A$ since the second term dominates.

% By definition, $G$ is always no larger than $TSA$, and tends to be small in many stochastic environments, as discussed in Appendix~\ref{app:universal}.
\ourmod \emph{does not necessarily require the number of templates $G$ to be small}.
A large $G$ suggests the environment is highly stochastic, e.g., the slipping probabilities of every grid in maze is sampled from a Gaussian distribution. In this case, we can still cluster s-a pairs with adequately close templates, as verified in experiments (see Section~\ref{sec:gaussian}).
% because $G=TSA$ only happens when every $p(s,a)$ is uniquely and independently determined.

\noindent\textbf{Proof Sketch. } We first show that for any s-a pair, $\thressm=\mathcal{\tilde{O}}(\frac{1}{\mingap^2})$ samples would guarantee correct template identification and aggregation, and $\threslg = \mathcal{\tilde{O}}(\frac{S V^2_{\max}}{\epsilon^2 (1-\gamma)^2})$ samples are sufficient for estimating the s-a transition dynamics. Then we prove that all s-a pairs reach $\thressm$ within finite steps. Finally, by computing the number of visits to unknown s-a pairs and applying the PAC-MDP theorem proposed by~\citet{strehl2012incremental}, we get the sample complexity result. 
Proof details are in Appendix E.
% ~\ref{app:proof}.
% The values of $\threslg$ and $\thressm$ are related to the environment and are given in Equation~\ref{eq:thres}. We will introduce theoretical details of the known thresholds in Section~\ref{sec:analysis}.
% \begin{equation}
% \label{eq:thres}
% \begin{aligned}
% \threslg = \mathcal{\tilde{O}}(\frac{S V^2_{\max}}{\epsilon^2 (1-\gamma)^2}), \quad
% \thressm = \mathcal{\tilde{O}}(\frac{1}{\min\{\tau^2, \nu^2\}})
% \end{aligned}
% \end{equation}

\noindent\textbf{Comparison with a single-task learner.} If RMax is sequentially run for every task, the total sample complexity for $T$ tasks is $\mathcal{\tilde{O}}\left(\frac{T S^2 A V_{\max }^{3}}{\epsilon^{3}(1-\gamma)^{3}}\right)$. \\
% Our \ourmod will always be no worse than RMax, since $G \leq TSA$ and $\mingap \geq \mathcal{O}(\frac{\epsilon (1-\gamma)}{SV_{\max}})$ always hold. Thus, \ourmod avoids the negative transfer problem.
(1) When precision is high, i.e., $\epsilon$ is small, a significant improvement is achieved, if $\mathcal{O}(SG) \ll \mathcal{O}(TS^2A)$. \\
(2) When $T$ is large, as long as $\mathcal{\tilde{O}}(\frac{SV_{\max}^2}{\epsilon^2 (1-\gamma)^2}) 
\gg \mathcal{\tilde{O}} (\frac{1}{\mingap^2})$, our \ourmod gains improved sample efficiency.\\ 
(3) \ourmod will not cause negative transfer among tasks. In the worst case, $G=TSA$ (there is no similarity among all s-a transition dynamics) or $\mingap^2=\mathcal{\tilde{O}}(\frac{SV_{\max}^2}{\epsilon^2 (1-\gamma)^2})$, \ourmod has the same-order sample complexity with RMax.

%\textbf{Sample Complexity of \ourmodtwo}
\label{sec:sc_group}

% \ourmodtwo works for the finite-model assumption, which is the same as FMRL~\cite{brunskill2013sample} and stated below.
%\begin{assumption}
%	There are at most $C$ MDPs for all tasks, i.e., $|\mathcal{M}| \leq C$. Each task has at least $p_{\min}>0$ task-prior probability to be drawn from $\mathcal{M}$.
%\end{assumption}

%With this assumption, we derive the sample complexity of \ourmodtwo.

\begin{theorem}[\textbf{Sample Complexity of \ourmodtwo}]
\label{thm:main_2}
Under the finite-model assumption of there are at most $C$ MDPs for all tasks, 
for any given $\epsilon>0, 1 > \delta > 0$, Algorithm 3
% ~\ref{alg:group}
 on $T$ tasks follows $\epsilon$-optimal policies for all but
\setlength\abovedisplayskip{1pt}
\setlength\belowdisplayskip{1pt}
\begin{equation}
\label{eq:main_2}
\mathcal{\tilde{O}}\Big( \frac{SGV^3_{\max}}{\epsilon^3(1-\gamma)^3} + \frac{T_1 SA V_{\max}}{\mingap^2 \epsilon (1-\gamma)}  
+ \frac{(T-T_1)DC^2 V_{\max}}{\mingap^2 \epsilon (1-\gamma)} \Big)
\end{equation}
steps with probability at least $1-\delta$, where $G$ is the total number of \patnames, 
% \setlength\abovedisplayskip{0pt}
% \setlength\belowdisplayskip{2pt}
% \begin{equation}
% \label{eq:t1}
$T_1 = \Omega (\frac{1}{p_{\min}} \ln \frac{C}{\delta})$
% \end{equation}
is the number of tasks in the first phase, where $p_{\min}$ is the minimal probability for a task to be drawn from $\mathcal{M}$.
\end{theorem}

% The first term is the number of steps needed to identify all s-a pairs in Phase 1; the second term is to fully know all patterns and is not related with $T$; the last term is the number of steps needed to identify the corresponding model.

\noindent\textbf{Remark. } 
(1) When $C$ is very large, or $p_{\min}$ is very small, $T_1 \to T$ and \ourmodtwo degenerates to \ourmod. (2) If $DC^2 < SA$ and $T \gg T_1$, \ourmodtwo requires fewer samples than \ourmod.

% \paragraph{Proof Sketch. } \yc{Todo}

% \paragraph{Comparison with \ourmod. } Based on a direct comparison between Equation~\ref{eq:main_2} and Equation~\ref{eq:main}, one can see \ourmodtwo requires fewer samples than \ourmod when $DC^2 < SA$. Further, the larger $T$ is, the more \ourmodtwo outperforms \ourmod. 

% \paragraph{Emma's Algorithm~\cite{brunskill2013sample}.} The sample complexity is $\mathcal{\tilde{O}}(\frac{\zeta_e V_{\max}}{\epsilon (1-\gamma)})$
% where
% $$
% \zeta_e = \mathcal{\tilde{O}}\left(T_{e1} \zeta_s + C \zeta_s + (T-T_{e1}) (\frac{DC^2}{\Gamma^2} + \frac{N V_{\max }^{2} C}{\epsilon^{2}(1-\gamma)^{2}})\right),
% $$
% and
% $$
% T_{e1} = \mathcal{O}\left( \frac{1}{q_{\min}} \ln \frac{C}{\delta}\right).
% $$
\noindent\textbf{Comparison with FMRL~\cite{brunskill2013sample}}
\ourmodtwo has a large improvement over FMRL in most cases. 
The sample complexity of FMRL for $T$ tasks in our notation is
\setlength\abovedisplayskip{2pt}
\setlength\belowdisplayskip{2pt} 
\begin{equation}
\begin{aligned}
\label{eq:fmrl}
\mathcal{\tilde{O}}\Big( & \frac{C S^2AV^3_{\max}}{\epsilon^3 (1-\gamma)^3} + \frac{T_1 S^2AV^3_{\max}}{\epsilon^3 (1-\gamma)^3} \\
& + (T-T_1) \big( \frac{DC^2V_{\max}}{\Gamma^2 \epsilon (1-\gamma)} + \frac{SCV^3_{\max}}{\epsilon^3 (1-\gamma)^3} \big) \Big).
\end{aligned}
\end{equation}
where $T_1 = \Omega (\frac{1}{p_{\min}} \ln \frac{C}{\delta})$, and $\Gamma$ is the model difference gap defined by~\citet{brunskill2013sample}. We organize Equation~\ref{eq:main_2} and Equation~\ref{eq:fmrl} both as three-term forms. The first term is for learning of all \patnames or all models, where \ourmodtwo reduces the dependence on $S$ and gets rid of the dependence on $A$. The second term is for the first phase, where FMRL performs the same with a single-task RMax learner, while \ourmodtwo requires much fewer samples to get optimal policies. Finally, the last term is for the second phase. FMRL needs an additional model elimination step for each task, while \ourmodtwo does not. 
\ourmodtwo is worse than FMRL only in extreme cases where there are few MDP models with large model gaps, and a large number of \patnames with small \patname gaps or ranking gaps.

\section{Experiments}
In this section, we demonstrate empirical results to show \ourmod and \ourmodtwo outperform existing state-of-the-art algorithms both in the finite-model setting and in the more realistic online setting. 
TempLe is able to transfer knowledge between tasks with different sized environments.
More importantly, TempLe has a high tolerance to model perturbations; it implements efficient transfer even when the underlying number of \patnames is infinite. 
Our code is available at \url{https://github.com/umd-huang-lab/template-reinforcement-learning}.

%!TEX root = 0_aaai_all.tex
% \vspace{-1.5em}
\begin{figure*}[!htbp]
\centering
	\begin{subfigure}[t]{0.24\textwidth}
		\centering
		% This file was created by tikzplotlib v0.9.0.
\begin{tikzpicture}[scale=0.5]

\definecolor{color0}{rgb}{0.462745098039216,0.654901960784314,0.490196078431373}
\definecolor{color1}{rgb}{0.4,0.470588235294118,0.67843137254902}
\definecolor{color2}{rgb}{0.776470588235294,0.443137254901961,0.443137254901961}
\definecolor{color3}{rgb}{0.662745098039216,0.756862745098039,0.835294117647059}
\definecolor{color4}{rgb}{0.901960784313726,0.662745098039216,0.517647058823529}

\begin{axis}[
legend cell align={left},
legend columns=2,
legend style={fill opacity=0.3, draw opacity=1, text opacity=1, at={(0.3,0.4)}, anchor=north west, draw=white!80!black, font=\large},
tick align=outside,
tick pos=left,
x grid style={white!69.0196078431373!black},
xlabel={\LARGE{Tasks}},
xmajorgrids,
xmin=-1.45, xmax=52.45,
xtick style={color=black},
y grid style={white!69.0196078431373!black},
ylabel={\LARGE{Per-task Reward}},
ymajorgrids,
ymin=6984.06581161696, ymax=22003.4846227104,
ytick style={color=black}
]
\addplot [thick, line width=1.5, color2, mark=diamond*, mark size=2, mark repeat=3, mark options={solid}]
table {%
1 17180.6333333333
2 17531.3133333333
3 17878.9153333333
4 17956.7304666667
5 18486.2007533333
6 18524.3173446667
7 18509.9889435333
8 18751.8533825133
9 18796.7947109287
10 19078.3485731691
11 19345.5003825189
12 19429.890344267
13 19670.7213098403
14 19695.2725121896
15 19826.4952609706
16 20090.3024015402
17 20099.2188280529
18 20184.8436119143
19 20250.5292507228
20 20375.5896589839
21 20607.0940264188
22 20649.4012904436
23 20850.0578280659
24 20821.8920452593
25 20871.6828407334
26 21107.2745566601
27 21320.7837676607
28 21132.008724228
29 21044.0011851385
30 21040.5410666247
31 21246.7069599622
32 20768.7829306326
33 20781.931304236
34 20553.7448404791
35 20705.5003564312
36 20689.1403207881
37 20549.3462887093
38 20607.1049931717
39 20495.0644938545
40 20468.6747111357
41 20184.4505733555
42 20228.6055160199
43 20253.3082977513
44 20197.1041346428
45 20062.6103878452
46 19985.1560157273
47 20176.3737474879
48 20204.0263727392
49 20326.4404021319
50 20375.1496952521
};
\addlegendentry{\ourmodtwo}
\addplot [thick, line width=1.5, color0, mark=*, mark size=2, mark repeat=3, mark options={solid}]
table {%
1 15561.3333333333
2 15806.7533333333
3 16099.1646666667
4 16087.5748666667
5 16526.2640466667
6 16506.4443086667
7 16414.1732111333
8 16535.5692233533
9 16498.5656343513
10 16668.4690709162
11 16810.8788304912
12 16789.3709474421
13 16906.2305193646
14 16860.9408007615
15 16944.1867206853
16 17108.4547152834
17 17074.1559104218
18 17084.9536527129
19 17071.4782874416
20 17154.8971253641
21 17308.2740794944
22 17296.5700048783
23 17440.0063377238
24 17339.7490372847
25 17347.2441335563
26 17534.463053534
27 17723.3400815139
28 17615.1094066959
29 17532.4417993596
30 17536.2542860903
31 17693.5421908146
32 17423.2113050665
33 17381.0435078932
34 17253.4924904372
35 17335.0899080601
36 17270.1775839208
37 17164.133158862
38 17196.6898429758
39 17184.0175253449
40 17249.6391061438
41 17088.941862196
42 17135.7976759764
43 17209.6245750455
44 17215.9787842076
45 17222.5342391202
46 17212.1808152081
47 17342.4527336873
48 17404.5741269853
49 17457.1467142867
50 17591.5820428581
};
\addlegendentry{Q-learning}
\addplot [thick, line width=1.5, color1, mark=square*, mark size=2, mark repeat=3, mark options={solid}]
table {%
1 7666.76666666667
2 7815.18
3 7974.232
4 7962.47546666667
5 8237.71125333333
6 8197.06346133333
7 8149.12044853333
8 8257.18840368
9 8232.83622997867
10 8365.87594031413
11 8461.29834628272
12 8458.72517832111
13 8553.48599382234
14 8501.7073944401
15 8538.39665499609
16 8654.66032282982
17 8608.01762388017
18 8612.20919482549
19 8607.56827534294
20 8652.33811447531
21 8772.66096969445
22 8768.564872725
23 8860.06505211917
24 8803.82854690725
25 8803.60569221653
26 8917.57178966154
27 9102.28127736205
28 9045.18648295918
29 8977.97783466326
30 8959.78338453027
31 9072.48171274391
32 8875.10354146952
33 8875.25318732257
34 8740.85453525698
35 8824.00908173128
36 8777.39817355815
37 8720.38168953567
38 8744.04685391543
39 8712.29550185722
40 8757.07261833817
41 8637.69535650435
42 8658.72915418725
43 8716.76290543519
44 8713.71328155834
45 8682.66528673584
46 8651.99875806225
47 8730.74554892269
48 8791.80099403042
49 8839.10756129405
50 8926.15680516464
};
\addlegendentry{RMax}
\addplot [thick, line width=1.5, white!36.8627450980392!black, mark=triangle*, mark size=2, mark repeat=3, mark options={solid}]
table {%
1 7720.03333333333
2 7867.98
3 8018.45866666667
4 7996.44613333333
5 8268.04485333333
6 8220.08370133333
7 8167.16199786667
8 8270.99246474667
9 8252.68988493867
10 8386.75756311147
11 8481.25514013365
12 8962.16295945362
13 9091.32999684159
14 9771.2836638241
15 9596.48529744169
16 10081.5301010309
17 9896.05375759443
18 10625.0017151683
19 10538.3582103182
20 11493.269055953
21 11385.1421503577
22 11762.1812686553
23 12271.7598084564
24 12507.5738276108
25 12846.3164448497
26 12846.6748003647
27 13712.5039869949
28 13539.1869216288
29 13465.2248961326
30 13892.519073186
31 14616.473832534
32 14706.3964492806
33 14740.7234710192
34 14528.097790584
35 14744.4080115256
36 14659.7705437064
37 14939.4534893357
38 14772.7081404021
39 14611.9606596953
40 14716.4545937257
41 14719.1758010198
42 14789.8548875845
43 14909.4293988261
44 14631.5764589435
45 14787.9854797158
46 15045.7102650775
47 15148.4125719031
48 15289.2579813795
49 15246.0055165749
50 14942.5849649174
};
\addlegendentry{FMRL}
\addplot [thick, line width=1.5, color3, mark=asterisk, mark size=2, mark repeat=3, mark options={solid}]
table {%
1 14760.2666666667
2 15035.35
3 15386.235
4 15343.8315
5 15602.7316833333
6 15706.1051816667
7 15635.2379968333
8 15896.3975304833
9 15897.947777435
10 16317.5963330248
11 16452.723366389
12 16433.3776964168
13 16678.2532601084
14 16713.1979340976
15 16827.6914740212
16 17063.492326619
17 17079.0464272905
18 17087.7117845614
19 17011.1606061053
20 17137.4245454948
21 17409.9154242786
22 17365.6272151841
23 17515.1144936657
24 17411.2630442991
25 17462.6834065359
26 17626.2150658823
27 18021.5635592941
28 17876.470536698
29 17810.4134830282
30 17748.978801392
31 17926.9575879195
32 17702.2451624609
33 17829.0639795481
34 17599.0275815933
35 17713.574823434
36 17513.6873410906
37 17384.4952736482
38 17529.8157462834
39 17618.244171655
40 17601.1864211562
41 17419.8377790406
42 17531.3973344698
43 17511.3142676895
44 17395.8795075872
45 17263.1582234952
46 17330.0690678123
47 17249.9788276978
48 17337.5676115947
49 17465.4441837685
50 17415.9464320583
};
\addlegendentry{Abs-RL}
\addplot [thick, line width=1.5, color4, mark=x, mark size=2, mark repeat=3, mark options={solid}]
table {%
1 12892.4666666667
2 13126.8766666667
3 13279.8123333333
4 13213.1244333333
5 13575.9386566667
6 13614.7381243333
7 13462.4643119
8 13585.75788071
9 13523.7354259723
10 13779.3985500418
11 13844.0053617043
12 14104.0748255338
13 14613.5606763138
14 14867.7879420157
15 15152.5491478142
16 15514.8975663661
17 15705.7678097295
18 15904.4910287565
19 16034.3852592142
20 16257.5133999594
21 16621.9220599635
22 16795.4365206338
23 17130.6995352371
24 17150.2862483801
25 17376.7309568754
26 17543.7178611878
27 17781.9394084024
28 17791.2821342288
29 17818.0172541393
30 17818.952195392
31 18102.2236425195
32 17895.5479449342
33 18021.2798171074
34 17775.2685020634
35 17978.5149851904
36 17743.3568200047
37 17670.0778046709
38 17682.7233575371
39 17783.4410217834
40 17720.6635862717
41 17570.1105609779
42 17583.1861715468
43 17738.9675543921
44 17796.7874656195
45 17844.3220523909
46 17949.4631804852
47 17981.47019577
48 18111.7298428597
49 18261.3835252404
50 18232.7651727163
};
\addlegendentry{MaxQInit}
\end{axis}

\end{tikzpicture}
		\vspace{-1.4em}
		\caption{{Finite-Model MTRL}}
		\label{sfig:group}
	\end{subfigure}
	\hfill
	\begin{subfigure}[t]{0.24\textwidth}
		\centering
		% This file was created by matplotlib2tikz v0.7.4.
\begin{tikzpicture}[scale=0.5]

\definecolor{color0}{rgb}{0.462745098039216,0.654901960784314,0.490196078431373}
\definecolor{color1}{rgb}{0.4,0.470588235294118,0.67843137254902}
\definecolor{color2}{rgb}{0.776470588235294,0.443137254901961,0.443137254901961}
\definecolor{color3}{rgb}{0.662745098039216,0.756862745098039,0.835294117647059}

\begin{axis}[
legend cell align={left},
legend style={at={(0.91,0.5)}, anchor=east, draw=white!80.0!black, font=\Large},
tick align=outside,
tick pos=left,
x grid style={white!69.01960784313725!black},
xlabel={\LARGE{Tasks}},
xmajorgrids,
xmin=-3.95, xmax=104.95,
xtick style={color=black},
y grid style={white!69.01960784313725!black},
ylabel={\LARGE{Per-task Reward}},
ymajorgrids,
ymin=25062.4232533439, ymax=53726.2332664562,
ytick style={color=black}
]
\addplot [thick, line width=1.5, color2, mark=diamond*, mark size=2, mark repeat=5, mark options={solid}]
table {%
1 40015.15
2 40907.92
3 41595.684
4 42454.6766
5 43343.27694
6 44265.852246
7 44988.2820214
8 45413.77181926
9 45947.639637334
10 46396.7566736006
11 46699.5740062405
12 47226.9596056165
13 47681.9986450548
14 48144.2947805493
15 48573.9063024944
16 48796.134672245
17 48542.1672050205
18 48993.1104845184
19 48969.4094360665
20 49324.7674924599
21 49553.6147432139
22 49816.8662688925
23 49649.5206420032
24 50066.5855778029
25 49707.6880200226
26 50065.2482180203
27 50057.5383962183
28 50226.5235565964
29 50372.1442009368
30 50538.2067808431
31 50592.2221027588
32 50884.4238924829
33 51023.5815032346
34 51064.5843529112
35 51261.98891762
36 51279.316025858
37 51015.6604232722
38 51212.165380945
39 51017.7598428505
40 51167.6258585654
41 51504.2932727089
42 51342.070945438
43 51468.1288508942
44 51655.9349658048
45 51591.4264692243
46 51421.7758223019
47 51195.4642400717
48 51160.4868160645
49 51598.787134458
50 51590.8424210122
51 51779.115178911
52 51507.6886610199
53 51550.5247949179
54 51550.9843154261
55 51547.3688838835
56 51258.6819954951
57 51282.0117959456
58 51393.9506163511
59 51645.264554716
60 51717.5500992443
61 51780.4080893199
62 52092.0862803879
63 51962.3246523491
64 51977.9501871142
65 51650.0411684028
66 51637.0340515625
67 51592.0146464062
68 51396.2621817656
69 51175.841963589
70 51247.6107672301
71 51411.0216905071
72 51472.7785214564
73 51280.0366693107
74 51146.1430023797
75 51446.5567021417
76 51309.9410319275
77 51129.7019287348
78 51300.7397358613
79 51308.4937622752
80 51398.4753860476
81 51314.0198474429
82 51450.1338626986
83 51499.9614764287
84 51731.4333287858
85 51927.8199959072
86 52132.5659963165
87 51945.0693966849
88 52294.8964570164
89 52423.3328113147
90 52382.4695301832
91 52230.7225771649
92 51821.8863194484
93 51608.8256875036
94 51859.5391187532
95 51763.4982068779
96 52009.2483861901
97 52359.1085475711
98 52417.088692814
99 52218.9768235326
100 52174.8781411793
};
\addlegendentry{\ourmod}
\addplot [thick, line width=1.5, color0, mark=*, mark size=2, mark repeat=5, mark options={solid}]
table {%
1 48526.0699999999
2 48502.1989999999
3 48172.8270999999
4 48192.90439
5 48185.169951
6 48390.6339559
7 48382.50156031
8 48099.145404279
9 47912.5348638511
10 47882.314377466
11 47730.5519397194
12 47800.2867457474
13 47868.1860711727
14 47893.6374640554
15 47932.1677176499
16 47887.9529458849
17 47357.8036512964
18 47605.0562861667
19 47268.8016575501
20 47446.0094917951
21 47537.3355426156
22 47643.828988354
23 47295.7490895186
24 47693.9581805667
25 47189.6083625101
26 47452.4685262591
27 47451.2096736332
28 47457.0587062698
29 47436.7128356428
30 47492.6105520786
31 47482.1614968707
32 47701.0623471836
33 47790.7741124653
34 47756.8687012187
35 47931.0138310969
36 47907.6754479872
37 47591.4989031885
38 47696.7780128696
39 47465.7972115827
40 47531.3524904244
41 47762.596241382
42 47611.7846172438
43 47729.5091555194
44 47887.5032399674
45 47852.5859159707
46 47720.2403243736
47 47499.9642919363
48 47450.9208627426
49 47761.4387764684
50 47695.1418988215
51 47920.8987089394
52 47559.8148380454
53 47608.2583542409
54 47671.2245188168
55 47707.8370669351
56 47409.1443602416
57 47318.8209242174
58 47422.5988317957
59 47665.8269486161
60 47682.3812537545
61 47666.2911283791
62 47977.1110155412
63 47825.656913987
64 47977.5672225883
65 47623.6545003295
66 47679.7030502965
67 47591.7947452669
68 47390.6702707402
69 47222.2942436662
70 47247.6648192996
71 47356.8123373696
72 47449.4901036326
73 47290.6450932694
74 47150.3635839424
75 47453.7322255482
76 47265.8080029934
77 47098.694202694
78 47292.1987824246
79 47354.3439041822
80 47486.8775137639
81 47398.1147623875
82 47483.4132861488
83 47485.2479575339
84 47737.4571617805
85 47909.7734456025
86 48082.0781010422
87 47805.378290938
88 48196.1784618442
89 48391.9006156598
90 48335.7295540938
91 48147.1855986844
92 47635.752038816
93 47492.7198349344
94 47751.1958514409
95 47622.9892662969
96 47818.8853396672
97 48145.0248057005
98 48158.1783251304
99 47983.4714926174
100 47979.0813433556
};
\addlegendentry{Q-learning}
\addplot [thick, line width=1.5, color1, mark=square*, mark size=2, mark repeat=5, mark options={solid}]
table {%
1 27743.83
2 27650.031
3 27447.2189
4 27411.74501
5 27385.613509
6 27432.1581581
7 27359.78034229
8 27167.262308061
9 27067.3170772549
10 27036.6013695294
11 26889.1352325765
12 26916.7827093188
13 26963.6274383869
14 26970.1546945482
15 26992.1672250934
16 26920.6485025841
17 26604.3376523257
18 26746.4058870931
19 26525.6422983838
20 26650.6490685454
21 26694.3341616909
22 26738.6257455218
23 26489.1571709696
24 26698.2874538726
25 26365.3237084854
26 26535.0973376368
27 26499.5316038732
28 26496.7004434858
29 26512.7223991373
30 26525.8211592235
31 26524.6070433012
32 26694.2573389711
33 26757.335605074
34 26766.3030445666
35 26878.6707401099
36 26837.9516660989
37 26673.063499489
38 26785.3261495401
39 26639.2125345861
40 26670.4142811275
41 26821.8808530147
42 26694.7697677133
43 26758.6577909419
44 26871.1330118477
45 26816.639710663
46 26713.2087395967
47 26574.317865637
48 26524.7870790733
49 26804.781371166
50 26740.9712340494
51 26911.6711106444
52 26723.15199958
53 26736.718799622
54 26755.2709196598
55 26801.1168276938
56 26587.2071449244
57 26580.007430432
58 26641.1616873888
59 26788.7545186499
60 26803.2500667849
61 26804.0580601064
62 27001.0232540958
63 26908.2709286862
64 26974.1728358176
65 26735.1205522358
66 26733.2354970122
67 26662.025947311
68 26519.7963525799
69 26378.7517173219
70 26424.7305455897
71 26570.7794910308
72 26620.4685419277
73 26482.7636877349
74 26409.0383189614
75 26587.7994870653
76 26475.8945383587
77 26415.6880845229
78 26522.5242760706
79 26559.0338484635
80 26656.6414636172
81 26540.9053172555
82 26594.4697855299
83 26593.3108069769
84 26715.8377262792
85 26866.9279536513
86 26976.0591582862
87 26872.9382424576
88 27118.3114182118
89 27191.9452763906
90 27200.2647487516
91 27126.1922738764
92 26820.0180464888
93 26715.5322418399
94 26895.8280176559
95 26842.2182158903
96 27001.3953943013
97 27185.1118548711
98 27204.775669384
99 27042.5761024456
100 27003.9954922011
};
\addlegendentry{RMax}
\addplot [thick, line width=1.5, color3, mark=asterisk, mark size=2, mark repeat=5, mark options={solid}]
table {%
1 46748.58
2 47007.336
3 46828.6134
4 47261.47006
5 47331.425054
6 47435.8065486
7 47372.10989374
8 47194.270904366
9 47479.1548139294
10 47596.1493325364
11 47517.5333992828
12 47658.5920593545
13 47613.4678534191
14 47851.8760680772
15 48005.3484612694
16 47698.9936151425
17 47412.3162536282
18 47470.4716282654
19 47300.0454654389
20 47279.081918895
21 47338.7547270055
22 47500.4132543049
23 47624.4559288744
24 47830.296335987
25 47446.1707023883
26 47346.4066321495
27 47457.0239689345
28 47487.7765720411
29 47427.9639148369
30 47650.7185233533
31 47443.7436710179
32 47481.9803039161
33 47428.9172735245
34 47246.7405461721
35 47493.3284915549
36 47213.6276423994
37 47122.8108781594
38 47489.3157903435
39 47416.2272113091
40 47595.4414901782
41 47717.1983411604
42 47315.4705070444
43 47422.6824563399
44 47602.7232107059
45 47635.9668896353
46 47614.0302006718
47 47295.4661806046
48 47320.4175625441
49 47516.4548062897
50 47390.3763256607
51 47328.2176930947
52 47176.1209237852
53 47221.8508314067
54 47253.872748266
55 47182.2064734394
56 47027.3448260955
57 47147.4083434859
58 46833.0245091373
59 46981.4200582236
60 47296.8420524012
61 47231.8178471611
62 47525.657062445
63 47368.6633562005
64 47480.8750205805
65 47823.7255185224
66 47782.5069666702
67 47676.0482700031
68 47491.5594430028
69 47318.4404987025
70 47208.6864488323
71 47352.7458039491
72 47496.7582235541
73 47291.0304011987
74 47043.6893610789
75 47308.456424971
76 47245.7677824739
77 47162.5470042265
78 47208.8713038038
79 47361.7241734234
80 47491.2027560811
81 47290.163480473
82 47530.3771324257
83 47726.7944191831
84 47403.3969772648
85 47455.8242795383
86 47512.9828515845
87 47509.2365664261
88 47454.5919097834
89 47593.5287188051
90 47645.9338469246
91 47542.6034622321
92 47138.6611160089
93 47097.257004408
94 47201.4483039672
95 47190.1914735705
96 47279.7443262134
97 47519.0438935921
98 47521.1685042329
99 47352.0106538096
100 47428.7185884286
};
\addlegendentry{Abs-RL}
\end{axis}

\end{tikzpicture}
	\vspace{-1.4em}
	\caption{{Online MTRL}}
	\label{sfig:maze}
	\end{subfigure}
	\hfill
	\begin{subfigure}[t]{0.24\textwidth}
		\centering
		% This file was created by matplotlib2tikz v0.7.4.
\begin{tikzpicture}[scale=0.49]

\definecolor{color0}{rgb}{0.462745098039216,0.654901960784314,0.490196078431373}
\definecolor{color1}{rgb}{0.4,0.470588235294118,0.67843137254902}
\definecolor{color2}{rgb}{0.776470588235294,0.443137254901961,0.443137254901961}

\begin{axis}[
legend cell align={left},
legend style={at={(0.03,0.97)}, anchor=north west, draw=white!80.0!black, font=\Large},
tick align=outside,
tick pos=left,
x grid style={white!69.01960784313725!black},
xlabel={\LARGE{Tasks}},
xmajorgrids,
xmin=-2.95, xmax=83.95,
xtick style={color=black},
y grid style={white!69.01960784313725!black},
ylabel style={align=center, font=\LARGE}, ylabel=Advantage Per-task Reward\\Compared with RMax,
ymajorgrids,
ymin=-1910.62190790987, ymax=40123.0600661073,
ytick style={color=black}
]
\addplot [thick, line width=1.5, color2, mark=diamond*, mark size=2, mark repeat=5, mark options={solid}]
table {%
1 6134.70000000001
2 6601.52000000001
3 6896.39800000001
4 7284.08820000001
5 7713.93938000001
6 8036.50544200001
7 8183.45489780001
8 8396.56940802
9 8569.722467218
10 8766.92022049621
11 8955.60819844659
12 9152.08737860193
13 9329.13864074173
14 9349.16477666756
15 9488.0482990008
16 9465.89346910072
17 9443.34412219065
18 9622.83970997159
19 9726.83573897443
20 9771.04216507699
21 10564.7079485693
22 11551.5871537124
23 12394.6284383411
24 13318.465594507
25 13875.4090350563
26 14316.9081315507
27 14485.0573183956
28 15134.291586556
29 15384.5224279004
30 15656.2801851104
31 15802.2621665994
32 16121.0759499394
33 16308.4083549455
34 16410.9375194509
35 16277.4737675058
36 16657.2063907553
37 16875.6457516797
38 17178.6411765118
39 17225.5670588606
40 17063.4103529745
41 17919.4293176771
42 19418.4363859094
43 19633.3127473184
44 20464.4114725866
45 21237.3503253279
46 22075.2852927951
47 22434.8567635156
48 23086.1010871641
49 23389.2809784477
50 23765.4328806029
51 24026.8595925426
52 24311.0136332883
53 24997.6922699595
54 25287.4930429636
55 25619.8937386672
56 26392.6343648005
57 26670.8709283204
58 26650.0738354884
59 26570.1964519395
60 26874.1168067456
61 28379.895126071
62 29447.5256134639
63 30036.4030521175
64 30657.8327469058
65 31709.5494722152
66 32635.2845249937
67 33754.2260724943
68 35055.4534652449
69 34941.6481187204
70 35668.0733068484
71 35735.6459761635
72 35723.2113785472
73 35850.2502406924
74 36371.3752166232
75 36144.6676949609
76 36573.8109254648
77 36446.8198329183
78 36590.8878496265
79 37685.3090646638
80 38212.4381581975
};
\addlegendentry{\ourmod}
\addplot [thick, line width=1.5, color1, mark=square*, mark size=2, mark repeat=5, mark options={solid}]
table {%
1 0
2 0
3 0
4 0
5 0
6 0
7 0
8 0
9 0
10 0
11 0
12 0
13 0
14 0
15 0
16 0
17 0
18 0
19 0
20 0
21 0
22 0
23 0
24 0
25 0
26 0
27 0
28 0
29 0
30 0
31 0
32 0
33 0
34 0
35 0
36 0
37 0
38 0
39 0
40 0
41 0
42 0
43 0
44 0
45 0
46 0
47 0
48 0
49 0
50 0
51 0
52 0
53 0
54 0
55 0
56 0
57 0
58 0
59 0
60 0
61 0
62 0
63 0
64 0
65 0
66 0
67 0
68 0
69 0
70 0
71 0
72 0
73 0
74 0
75 0
76 0
77 0
78 0
79 0
80 0
};
\addlegendentry{RMax}
\end{axis}

\end{tikzpicture}
		\vspace{-1.4em}
		\caption{{Varying-sized MTRL}}
		\label{sfig:var}
	\end{subfigure}
	\hfill
	\begin{subfigure}[t]{0.24\textwidth}
		\centering
		% This file was created by matplotlib2tikz v0.7.4.
\begin{tikzpicture}[scale=0.5]

\definecolor{color0}{rgb}{0.462745098039216,0.654901960784314,0.490196078431373}
\definecolor{color1}{rgb}{0.4,0.470588235294118,0.67843137254902}
\definecolor{color2}{rgb}{0.776470588235294,0.443137254901961,0.443137254901961}

\begin{axis}[
legend cell align={left},
legend style={at={(0.91,0.4)}, anchor=east, draw=white!80.0!black, font=\Large},
tick align=outside,
tick pos=left,
x grid style={white!69.01960784313725!black},
xlabel={\LARGE{Tasks}},
xmajorgrids,
xmin=-3.95, xmax=104.95,
xtick style={color=black},
y grid style={white!69.01960784313725!black},
ylabel={\LARGE{Per-task Reward}},
ymajorgrids,
ymin=14095.7043125133, ymax=37245.2086755668,
ytick style={color=black}
]
\addplot [thick, line width=1.5, color2, mark=diamond*, mark size=2, mark repeat=5, mark options={solid}]
table {%
1 28091.9000000001
2 28108.0860000001
3 28735.1014000001
4 28988.5672600001
5 29785.1865340001
6 30367.9078806001
7 30827.0730925401
8 31207.0957832861
9 31340.9662049575
10 31878.4395844617
11 32359.1816260156
12 32605.553463414
13 32830.7121170726
14 33269.0929053654
15 33440.9436148288
16 33522.899253346
17 34136.4513280114
18 34045.4821952103
19 34144.2259756892
20 34403.6793781203
21 34305.4854403083
22 34255.0868962775
23 34398.5442066497
24 34235.3097859848
25 34012.1888073863
26 34361.1159266477
27 34590.0403339829
28 34774.3623005846
29 34756.5480705262
30 34626.7372634736
31 34847.0575371262
32 35500.5117834136
33 35559.4666050723
34 35270.881944565
35 35514.6517501085
36 35789.2025750977
37 36034.9223175879
38 36024.7560858291
39 36192.9584772462
40 35876.7946295216
41 35983.2551665695
42 35606.4696499125
43 35503.9446849213
44 35393.9082164292
45 35930.9513947863
46 35515.1742553076
47 35336.5908297769
48 35286.1057467992
49 35603.9291721193
50 35138.5162549074
51 35170.2026294167
52 34845.446366475
53 34509.3377298275
54 34356.1459568448
55 34228.3193611603
56 33843.6374250443
57 33979.2916825399
58 33898.0285142859
59 33833.0496628573
60 34149.2706965716
61 33738.8776269145
62 34026.173864223
63 34489.3884778007
64 35078.2956300206
65 34806.8320670186
66 34961.1188603167
67 35123.5449742851
68 35285.8564768566
69 34924.1228291709
70 34846.5325462538
71 34209.6732916285
72 33544.5479624656
73 33524.8611662191
74 34012.0010495972
75 34554.2869446375
76 34948.5382501737
77 34763.4944251564
78 34973.8029826408
79 34223.8106843767
80 34305.053615939
81 34025.1882543451
82 34840.1014289106
83 35004.1952860196
84 34554.2157574176
85 34553.4101816759
86 34928.8151635083
87 34757.6616471575
88 35190.9354824417
89 35280.2899341976
90 35011.5149407778
91 34874.3854467
92 35096.0509020301
93 35379.4878118271
94 34739.7110306444
95 34725.3559275799
96 34900.0423348219
97 35351.4021013398
98 35365.7918912058
99 35787.1307020852
100 35743.0076318767
};
\addlegendentry{O-TempLe}
\addplot [thick, line width=1.5, color0, mark=*, mark size=2, mark repeat=5, mark options={solid}]
table {%
1 34589.4400000001
2 33771.5720000001
3 33515.3968000001
4 32890.8371200001
5 32795.9314080001
6 32667.9402672001
7 32389.3142404801
8 32153.0748164321
9 31635.5793347889
10 31549.12140131
11 31655.143261179
12 31582.7449350611
13 31520.524441555
14 31577.7179973996
15 31486.7421976596
16 31290.3979778937
17 31653.1381801043
18 31376.5223620939
19 31326.2781258845
20 31305.3263132961
21 30899.0616819665
22 30671.0895137698
23 30616.5525623929
24 30365.6653061536
25 30157.3747755382
26 30271.0352979844
27 30460.615768186
28 30573.0001913674
29 30603.9261722307
30 30419.6895550076
31 30542.1405995069
32 31182.5345395562
33 31001.2270856006
34 30735.9003770405
35 30800.8263393365
36 31028.0837054029
37 31133.4493348626
38 31087.3904013763
39 31465.9913612387
40 31169.5962251148
41 31183.8586026034
42 30780.328742343
43 30629.7158681087
44 30584.4442812979
45 31138.9158531681
46 30629.9342678513
47 30205.2408410662
48 29990.7447569596
49 30450.5742812637
50 29887.8088531373
51 29955.1399678236
52 29780.3319710412
53 29463.5967739371
54 29296.3710965434
55 29232.3479868891
56 28964.1191882002
57 29005.3332693802
58 28853.4579424422
59 28672.508148198
60 29000.4533333782
61 28536.7340000404
62 28609.8726000364
63 28944.6313400327
64 29582.3382060295
65 29251.0003854265
66 29349.8763468839
67 29509.5567121955
68 29657.319040976
69 29228.4851368784
70 29333.0786231906
71 28699.9747608715
72 28000.0792847844
73 28023.5013563059
74 28495.9692206754
75 28930.1462986078
76 29235.0516687471
77 29153.7405018724
78 29219.5304516852
79 28459.4154065167
80 28558.343865865
81 28293.1434792785
82 28930.9431313507
83 29095.9788182156
84 28750.3849363941
85 28815.6164427547
86 29324.4207984792
87 29202.7227186313
88 29618.0924467682
89 29723.3592020914
90 29484.0592818822
91 29383.495353694
92 29584.4758183246
93 29862.2622364922
94 29324.642012843
95 29282.4218115587
96 29311.2916304028
97 29713.8664673626
98 29618.7498206263
99 30000.1588385637
100 29947.3449547074
};
\addlegendentry{Q-learning}
\addplot [thick, line width=1.5, color1, mark=square*, mark size=2, mark repeat=5, mark options={solid}]
table {%
1 18951.3
2 18497.322
3 18334.3338
4 17942.80042
5 17880.932378
6 17787.4531402
7 17619.18382618
8 17462.839443562
9 17163.5534992058
10 17159.7681492852
11 17168.7133343567
12 17081.466000921
13 17061.7394008289
14 17104.301460746
15 17064.4893146714
16 16939.2803832043
17 17206.1983448838
18 17116.2565103955
19 17087.2768593559
20 17067.3791734203
21 16821.0792560783
22 16693.0993304704
23 16643.1033974234
24 16493.1430576811
25 16284.1087519129
26 16372.9438767217
27 16519.8794890495
28 16535.5775401445
29 16531.6457861301
30 16404.6872075171
31 16532.6484867654
32 16933.1116380888
33 16836.0704742799
34 16632.1954268519
35 16667.5278841667
36 16813.3230957501
37 16910.7947861751
38 16938.0153075576
39 17088.6277768018
40 16933.2709991216
41 16971.9538992094
42 16743.4105092885
43 16634.7434583596
44 16621.4111125237
45 16900.8820012713
46 16583.0818011442
47 16365.7556210298
48 16282.6420589268
49 16479.7278530341
50 16114.1250677307
51 16131.3965609576
52 15991.7209048618
53 15756.1928143757
54 15615.0995329381
55 15511.6135796443
56 15286.2902216798
57 15370.9011995119
58 15359.6010795607
59 15231.1769716046
60 15440.1372744441
61 15218.6395469997
62 15356.9755922998
63 15568.4560330698
64 15934.9644297628
65 15742.8279867865
66 15821.9451881079
67 15830.6426692971
68 15944.6564023674
69 15772.6887621306
70 15881.4658859176
71 15523.2672973258
72 15160.4005675932
73 15147.9545108339
74 15454.5270597505
75 15701.7903537754
76 15946.0253183979
77 15859.2467865581
78 15983.7061079023
79 15515.5194971121
80 15589.4275474009
81 15330.0207926608
82 15767.8447133947
83 15897.1302420552
84 15595.1172178497
85 15670.7354960647
86 15898.6459464583
87 15787.6313518124
88 16016.7982166312
89 16069.3003949681
90 15826.4423554713
91 15744.9601199241
92 15841.9201079317
93 16011.0240971385
94 15696.1536874247
95 15611.7083186822
96 15628.859486814
97 15853.2135381326
98 15843.0681843193
99 16109.1533658874
100 16126.7660292986
};
\addlegendentry{RMax}
\end{axis}

\end{tikzpicture}
		\vspace{-1.4em}
% 		\caption{{Gaussian landforms}}
% 		\label{sfig:noise}
% 	\end{subfigure}
% \vspace{-0.5em}
% \caption{{Performance of \ourmod and \ourmodtwo compared against baselines in (a) Online MTRL, (b) Finite-Model MTRL, (c) varying sized MTRL and (d) Online MTRL with Gaussian distributed landforms. }
		\caption{Infinite number of \patnames}
		\label{sfig:noise}
	\end{subfigure}
\vspace{-0.5em}
\caption{Performance of \ourmod and \ourmodtwo compared against state-of-the-art baselines in \textbf{(a)} Online MTRL (to show TempLe's ability to efficiently transfer knowledge), \textbf{(b)} Finite-Model MTRL (to show TempLe outperforms baselines even under environments that the baselines are designed for),\textbf{(c)} varying sized MTRL (to show TempLe extends to varying sized state space) and \textbf{(d)} Online MTRL with Mixture-of-Gaussians distributed landforms (to show TempLe's robustness against noise and model-perturbation). All results are averaged over 20 different random sequences of tasks. Confidence intervals are omitted to reduce overlapping. 
}
\label{fig:all}
\end{figure*}
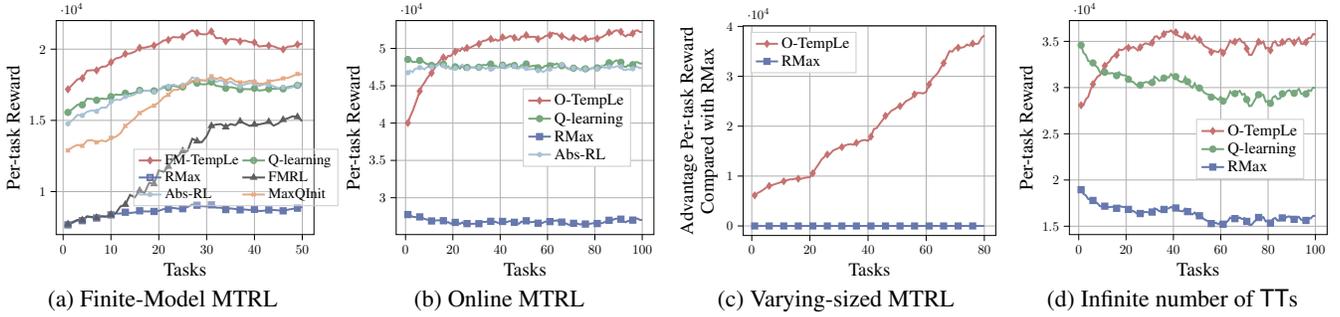
% \vspace{-0.5em}

\noindent\textbf{Baselines. }
We choose the state-of-the-art MTRL algorithms, \emph{Abstraction RL (Abs-RL)}~\cite{abel2018state}, MaxQInit~\cite{abel2018policy} and \emph{FMRL}~\cite{brunskill2013sample} as baselines. For Abs-RL and MaxQInit, we use the code provided by authors. Note that Abs-RL and MaxQInit have multiple versions due to the selection of different base learners, we show the ones with their best performance in this section, and other versions in Appendix F.3.
% ~\ref{app:exp_baseline}.
% Abs-RL proposes multiple state abstraction types that can be combined with model-free methods such as Q-learning and Delayed-Q. We compare against  ``$\phi_{Q_\epsilon^*}$'', ``$\phi_{Q_d^*}$'', and ``$\phi_{Q^*}$'' combined with Q-learning, namely, ``Abs-$\phi_{Q_\epsilon^*}$'', ``Abs-$\phi_{Q_d^*}$'', and ``Abs-$\phi_{Q^*}$''. 
% Abs-RL is not PAC-MDP, so we could only compare with it empirically, not theoretically.
Abs-RL works for both the online and finite-model setting, 
% ~\fh{Emphasize  that abstraction is not PAC-MDP, therefore we could only compare empirically, not theoretically.  } 
whereas MaxQInit and FMRL work for the finite-model setting only, since they both require the number of tasks to be small and known. 
Meanwhile, to show the effectiveness of our proposed algorithms and other MTRL algorithms, we also run RMax and Q-learning~\cite{watkins1992q} for every single task without knowledge transfer.

\subsection{Finite-Model MTRL}
\noindent\textbf{Environment.}
All the baselines including FMRL, Abs-RL, MaxQInit are designed for the finite-model setting (note that Abs-RL also works in the online setting), where the number of models $C$ is small.
We use a similar maze environment as in FMRL, where MDPs only differ at the goal state. 

\noindent\textbf{Performance.}
We generate two $4 \times 4$ maze tasks with different goal states as the underlying models, and then randomly sample 50 tasks from the two underlying models. 
Figure~\ref{sfig:group} shows the comparison of per-task rewards. 
FMRL has the same performance with RMax in the model-collecting phase, and then achieves increasing rewards in the following tasks after it successfully identifies the underlying two types of MDPs. 
After 30 tasks, all state-actions pairs in the models become known, so the per-task reward converges. Similarly, MaxQInit gains more rewards when it collects adequate knowledge of the Q values.
In contrast, \ourmodtwo has a better start as it learns \patnames from the beginning. 
And model identification further helps with efficient learning. Over all tasks, \ourmodtwo substantially outperforms other agents, despite that baselines are designed for the finite-model case.

\subsection{Online MTRL}
\noindent\textbf{Environment.}
For the more realistic Online MTRL which allows the \emph{number of MDP models to be extremely large}, we generalize the traditional maze environment to have arbitrary combinations of landforms, as shown in Figure~\ref{fig:landforms_eg}. We use 3 types of landforms, sand, marble and ice, respectively with slipping probabilities 0, 0.2, and 0.4. In this scenario, under a certain number of states $S$, the number of possible tasks is exponential in $S$. 

\noindent\textbf{Performance.}
In the online setting, we consider $4\times 4$ mazes with different arrangements of landforms streaming in. 
The per-task rewards of each agent are displayed in Figure~\ref{sfig:maze}.
Among all agents, our \ourmod obtains the highest average reward. 
We see during the first 40 tasks, the performance of \ourmod continuously and rapidly grows by transferring previous knowledge. 
In contrast, the performance of Abs-RL does not increase as more tasks come in and keeps the same with single-task Q-learning, because the maze environment is not efficiently abstracted by Abs-RL.

\noindent\textbf{Performance on Varying State Space.}
To show the feasibility of TempLe for {varying-sized environment tasks}, and its ability to generalize knowledge learned in small tasks to speed up learning in larger tasks, we vary the size of the mazes across tasks. More specifically, the first 20 tasks are $3\times 3$ mazes, followed by 20 $4 \times 4$ mazes, 20 $5 \times 5$ mazes and 20 $6 \times 6$ mazes. 
We show \ourmod's per-task advantage rewards over single task RMax in Figure~\ref{sfig:var}, since other MTRL baselines are not feasible in this setting. 
The performance advantage over RMax increases over more observed tasks, verifying that \ourmod transfers knowledge among different-sized mazes. Experiments on varying action spaces are shown in Appendix F.4.
% ~\ref{app:action}.

\subsection{MTRL with Infinite \patnames}
\label{sec:gaussian}
\noindent\textbf{Environment.}
We also conduct experiments to show TempLe's robustness to noise and model perturbations.
which is crucial for its application to real-world settings where ``landforms'' could vary continuously. 
We draw the landforms (slipping probabilities) of each grid from a mixture of Gaussian distributions, which are centered at 0.2, 0.4, and 0.6 with standard derivation 0.05. In this case, the number of \patnames could be infinitely large. 

\noindent\textbf{Performance.}
We show \ourmod's per-task advantage rewards over single task RMax and Q-learning in Figure~\ref{sfig:noise}, in which \ourmod still achieves successful multi-task learning. This result demonstrates \ourmod's ability of tolerating noise and generalizing to real-life applications.

% \subsection{Universal Applicability of TempLe}

\subsection{Robustness to Hyper-parameters}

\modname requires a user-specified \patname gap $\hat{\tau}$ as input. Also, both FMRL and \ourmodtwo require a user-specified model gap $\Gamma$. 
We test various hyper-parameters to understand how significantly the performance of the algorithms could be affected by inaccurate guesses of $\hat{\tau}$ and $\Gamma$, shown in Figure~\ref{fig:paras}.
%However, in real-life tasks where the underlying dynamics are unknown, a user usually makes an educated guess about the values of $\hat{\tau}$ and $\Gamma$. How robust \modname and FMRL are against inaccurate input guesses? 
% Figure~\ref{fig:paras} displays the impact that various inputs have on the gained rewards.

\begin{figure}[!htbp]
\centering
\vspace{-1em}
	\begin{subfigure}[t]{0.4\columnwidth}
		\centering
		% This file was created by matplotlib2tikz v0.7.4.
\begin{tikzpicture}[scale=0.5]

\definecolor{color0}{rgb}{0.4,0.470588235294118,0.67843137254902}
\definecolor{color1}{rgb}{0.776470588235294,0.443137254901961,0.443137254901961}
\definecolor{color2}{rgb}{0.901960784313726,0.662745098039216,0.517647058823529}

\begin{axis}[
legend cell align={left},
legend style={draw=white!80.0!black, at={(0.2,0.3)}, anchor=north west, font=\Large},
tick align=outside,
tick pos=left,
x grid style={white!69.01960784313725!black},
xlabel={\LARGE{User-specified \patname Gap $\hat{\tau}$}},
xmajorgrids,
xmin=-0.0845, xmax=1.9945,
xtick style={color=black},
y grid style={white!69.01960784313725!black},
ylabel={\LARGE{Average Reward}},
ymajorgrids,
ymin=-6603.82084870129, ymax=53186.7259555929,
ytick style={color=black}
]
\addplot [thick, line width=1.5, color0, mark=square*, mark size=2, mark options={solid}]
table {%
0.01 26176.508
0.05 26176.508
0.1 26176.508
0.2 26176.508
0.3 26176.508
0.4 26176.508
0.5 26176.508
0.6 26176.508
0.7 26176.508
0.8 26176.508
0.9 26176.508
1 26176.508
1.1 26176.508
1.2 26176.508
1.3 26176.508
1.4 26176.508
1.5 26176.508
1.6 26176.508
1.7 26176.508
1.8 26176.508
1.9 26176.508
};
\addlegendentry{RMax}
\addplot [thick, line width=1.5, color1, mark=diamond*, mark size=2, mark options={solid}]
table {%
0.01 34352.972
0.05 38186.806
0.1 42962.855
0.2 46543.3265
0.3 48419.93325
0.4 49466.022625
0.5 49965.2893125
0.6 50320.45365625
0.7 50468.973828125
0.8 50324.9569140625
0.9 50369.2994570312
1 49708.9857285156
1.1 49785.4683642578
1.2 46871.6376821289
1.3 25935.7168410645
1.4 13075.6769205322
1.5 3793.66446026612
1.6 -873.422769866942
1.7 -2102.31588493347
1.8 -3026.50144246674
1.9 -3886.06872123337
};
\addlegendentry{O-TempLe}
\addplot [thick, line width=2, color2, dashed, forget plot]
table {%
0.15 -6603.82084870129
0.15 53186.7259555929
};
\end{axis}

\end{tikzpicture}
		\vspace{-1em}
		\caption{\patname gap $\hat{\tau}$}
		\label{sfig:paras_tau}
	\end{subfigure}
	\hfill
	\begin{subfigure}[t]{0.48\columnwidth}
		\centering
		% This file was created by tikzplotlib v0.9.0.
\begin{tikzpicture}[scale=0.5]

\definecolor{color0}{rgb}{0.4,0.470588235294118,0.67843137254902}
\definecolor{color1}{rgb}{0.776470588235294,0.443137254901961,0.443137254901961}
\definecolor{color2}{rgb}{0.901960784313726,0.662745098039216,0.517647058823529}

\begin{axis}[
legend cell align={left},
legend style={fill opacity=0.8, draw opacity=1, text opacity=1, draw=white!80!black, font=\Large},
tick align=outside,
tick pos=left,
x grid style={white!69.0196078431373!black},
xlabel={\LARGE{User-specified Model Gap $\Gamma$}},
xmajorgrids,
xmin=0.12, xmax=1.88,
xtick style={color=black},
xtick={0,0.2,0.4,0.6,0.8,1,1.2,1.4,1.6,1.8,2},
xticklabels={0.0,0.2,0.4,0.6,0.8,1.0,1.2,1.4,1.6,1.8,2.0},
y grid style={white!69.0196078431373!black},
ylabel={\LARGE{Average Rewards}},
ymajorgrids,
ymin=-1717.10101229568, ymax=24053.2495720141,
ytick style={color=black}
]
\addplot [thick, line width=1.5, color0, mark=square*, mark size=2, mark options={solid}]
table {%
0.2 10279.81
0.4 10279.81
0.6 10279.81
0.8 10279.81
1 10279.81
1.2 10279.81
1.4 10279.81
1.6 10279.81
1.8 10279.81
};
\addlegendentry{RMax}
\addplot [thick, line width=1.5, color1, mark=diamond*, mark size=2, mark options={solid}]
table {%
0.2 22881.87
0.4 22556.414
0.6 22307.4364
0.8 18276.82584
1 10345.399504
1.2 5203.4637024
1.4 1962.45822144
1.6 255.750932864
1.8 -545.7214402816
};
\addlegendentry{FM-TempLe}
\addplot [thick, line width=1.5, white!36.8627450980392!black, mark=triangle*, mark size=2, mark options={solid}]
table {%
0.2 12729.35
0.4 12554.27
0.6 11028.55
0.8 10694.402
1 8185.2772
1.2 6691.18632
1.4 5791.635792
1.6 5249.9574752
1.8 4925.95048512
};
\addlegendentry{FMRL}
\addplot [thick, line width=2, color2, dashed, forget plot]
table {%
0.6 -1717.10101229568
0.6 24053.2495720141
};
\end{axis}

\end{tikzpicture}	
		\vspace{-1em}
		\caption{Model gap $\Gamma$} 
		\label{sfig:paras_gamma}
	\end{subfigure}
\vspace{-0.5em}
\caption{Hyper-parameter test of \patname gap $\hat{\tau}$ and model gap $\Gamma$(the vertical dashed line shows the underlying true value). 
%Results are averaged over 10 runs.
}
\vspace{-0.5em}
\label{fig:paras}
\end{figure}
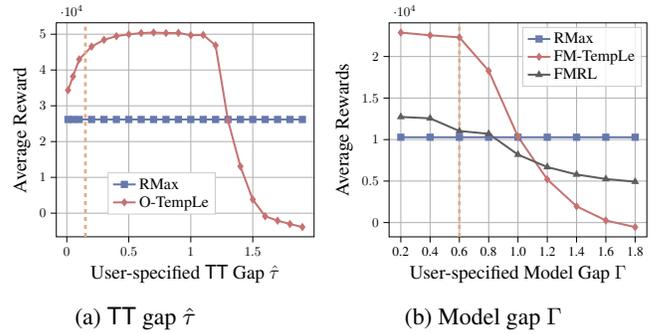

% Figure~\ref{fig:paras} displays the impact that various inputs have on the gained rewards.
According to Figure~\ref{sfig:paras_tau}, the performance of \ourmod drops when $\hat{\tau}$ is too large. 
However, the rewards remain high for relatively small $\hat{\tau}$. 
% So the accepted range of $\hat{\tau}$ is large, and can be much larger than the true \patname gap. 
Figure~\ref{sfig:paras_gamma} shows that \ourmodtwo gets higher rewards than RMax when setting $\Gamma \leq 1$, although $\Gamma$ has a larger influence on \ourmodtwo compared to FMRL, potentially because the failure of model clustering will cause more inaccurate \patname identification. 
% \fh{intuition, why?}
%while FMRL is only substantially better than RMax when $\Gamma \leq 0.4$.
Note that by definition, both $\hat{\tau}$ and $\Gamma$ would not exceed 2 (see Lemma 6). 
% ~\ref{lem:bound}
So we still have a large chance to get higher rewards than RMax by making an educated guess. 
The results in Figure~\ref{fig:paras} guide the users to specify hyper-parameters when using \modname.
The results with confidence intervals, comparison of cumulative rewards, and experiments on additional environments are shown in Appendix F.
% ~\ref{app:exp_addition}.  
We also provide 
% a justification of the applicability of TempLe to more environments in Appendix F.5.
% ~\ref{app:universal}.
an extension of our work to deep RL is discussed in Appendix F.6.
% ~\ref{app:deep}.

%!TEX root = 0_aaai_all.tex
% \vspace*{-1em}
\section{Conclusion and Discussion}
\label{sec:conclusion}
% \vspace*{-0.5em}
In this work, we propose \modname, the first PAC-MDP MTRL algorithm that works for tasks with varying state/action space without any inter-task mappings or prior knowledge of the MDP structures. 
This work can be extended in many directions. For example, 
% we define \patname as the concatenation of transition probability and reward for simplicity, but 
one may benefit from investigating transition probability and reward separately. 
% Moreover, it is possible to extend the template idea to continuous MDP and deep model-based RL, by 
The idea of extracting modular similarities can also be extended to continuous MDP and deep model-based RL.
% Our future work includes extension to continuous MDPs by discretizing the state/action space. It is also possible to apply our idea to deep model-based RL, where the the state prediction is usually a Gaussian distribution, then the learned derivation can be regarded as ``templates'' and augmented by grouping.

% \textcolor{red}{Two algorithms, \ourmod and \ourmodtwo, are introduced for online setting and finite-model setting respectively. We show in theory and experiments that \ourmod and \ourmodtwo both achieve higher sample efficiency than state-of-the-art methods.}

% \fh{remove the red, but add the discussion of deep RL and potential extension to continuous space?}

\section*{Acknowledgements}
Huang is supported by startup fund from Department of Computer Science of University of Maryland, National Science Foundation IIS-1850220 CRII Award 030742- 00001, DOD-DARPA-Defense Advanced Research Projects Agency Guaranteeing AI Robustness against Deception (GARD), Laboratory for Physical Sciences at University of Maryland, and Adobe, Capital One and JP Morgan faculty fellowships.

\section*{Ethical Impact}
Our presented algorithms on multi-task reinforcement learning facilitate the learning of new tasks using knowledge accumulated from previously learned tasks. 
In scenarios where an RL agent needs to sequentially interact with a series of environments, e.g., navigation in various places, our proposed algorithm could be applied to improve the learning efficiency without loss of accuracy. 
More importantly, our algorithms are guaranteed to learn near-optimal policies and avoid negative transfer, which are crucial for high-stakes applications, such as autonomous driving, market making, and health-care systems.

% Due to the generality of our work, it is widely applicable to interactive decision making processes.
% For instance, autonomous driving systems, AI robots systems, gaming systems, market making systems and security monitoring systems.  
Nowadays, Deep Reinforcement Learning (DRL) has achieved great success in many applications. However, problems like high variance and instability restrict the use of DRL in real-life problems. Thus, it is important to study tabular RL with guarantees, which could potentially benefit DRL and applications involving DRL. Our proposed algorithms, although not in the scope of DRL, could be potentially extended to DRL in the following ways.
\textbf{(1)} Our idea of extracting ``relative'' transition probability similarity could be directly used in model-based DRL. For example, the next-state prediction model usually outputs a Gaussian distribution for every s-a pair, and one can augment the learned derivation by averaging over predicts with close derivations, assuming some similarity about the uncertainty among different states.
\textbf{(2)} It is possible to discretize state space and apply count-based methods, as suggested in by~\citet{tang2016exploration}.

Our work on multi-task reinforcement learning also has the potential to be applied to other transfer learning tasks within and outside of the Reinforcement Learning community. 
Any learning in systems that share modular similarities could potentially benefit from our algorithms to speed up the training process. 

% \bibliography{\bibhome/supp_bib}

\newpage
\onecolumn
\appendix
{\centering{\Large Appendix: \mytitle}}
%!TEX root = ../0_aaai2021_multitaskrl_main.tex

\section{Intuitive Examples}
\label{app:examples}

\begin{figure}[!htbp]
  \centering
    \begin{subfigure}[t]{0.24\columnwidth}
    \centering
    \includegraphics[width=0.8\textwidth, bb=0 0 230 230]{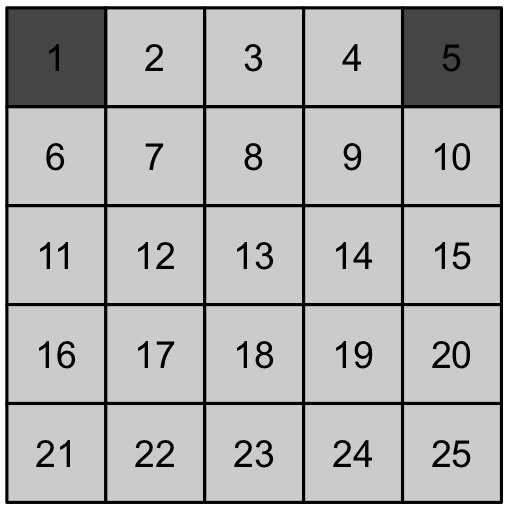}
    \caption{action=`$\uparrow$'}
    \label{sfig:pateg_up}
    \end{subfigure}
    \hfill
    \begin{subfigure}[t]{0.24\columnwidth}
    \centering
    \includegraphics[width=0.8\textwidth, bb=0 0 230 230]{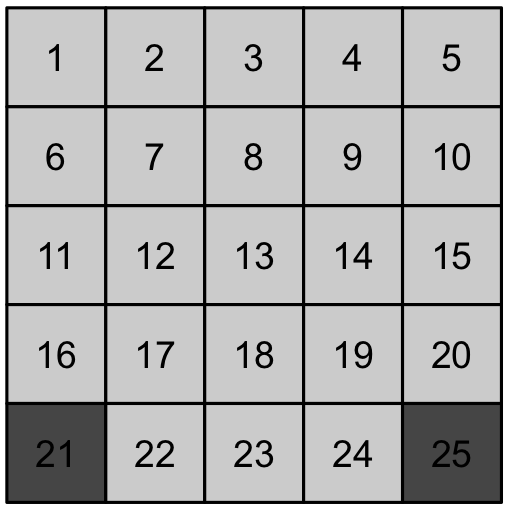}
    \caption{action=`$\downarrow$'}
    \label{sfig:pateg_down}
    \end{subfigure}
    \hfill
    \begin{subfigure}[t]{0.24\columnwidth}
    \centering
    \includegraphics[width=0.8\textwidth, bb=0 0 230 230]{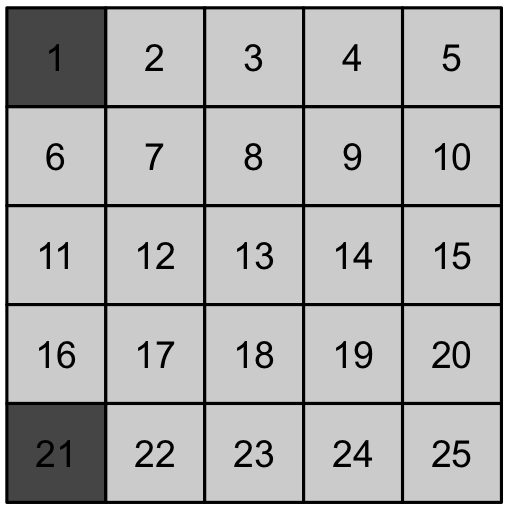}
    \caption{action=`$\leftarrow$'}
    \label{sfig:pateg_left}
    \end{subfigure}
    \hfill
    \begin{subfigure}[t]{0.24\columnwidth}
    \centering
    \includegraphics[width=0.8\textwidth, bb=0 0 230 230]{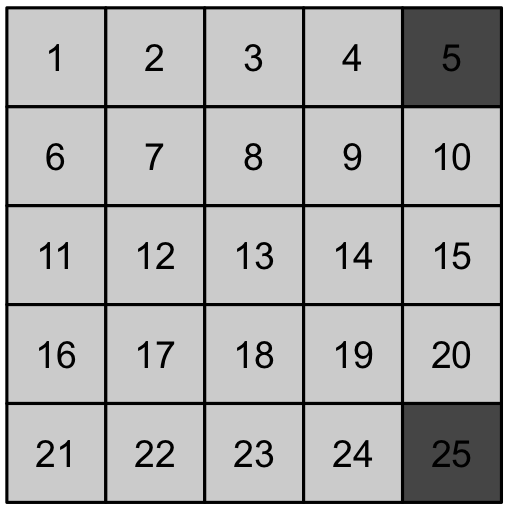}
    \caption{action=`$\rightarrow$'}
    \label{sfig:pateg_right}
    \end{subfigure}
    % \hfill
    % \begin{subfigure}[t]{0.15\textwidth}
    % \centering
    % \includegraphics[width=0.8\textwidth]{\fighome/legend.png}
    % \end{subfigure}
    \caption{An example of \patnames in a $5\times 5$ slippery gridworld with no reward and slipping probability$=$0.4. The template at all \protect\includegraphics[height=0.7em, bb=0 0 550 550]{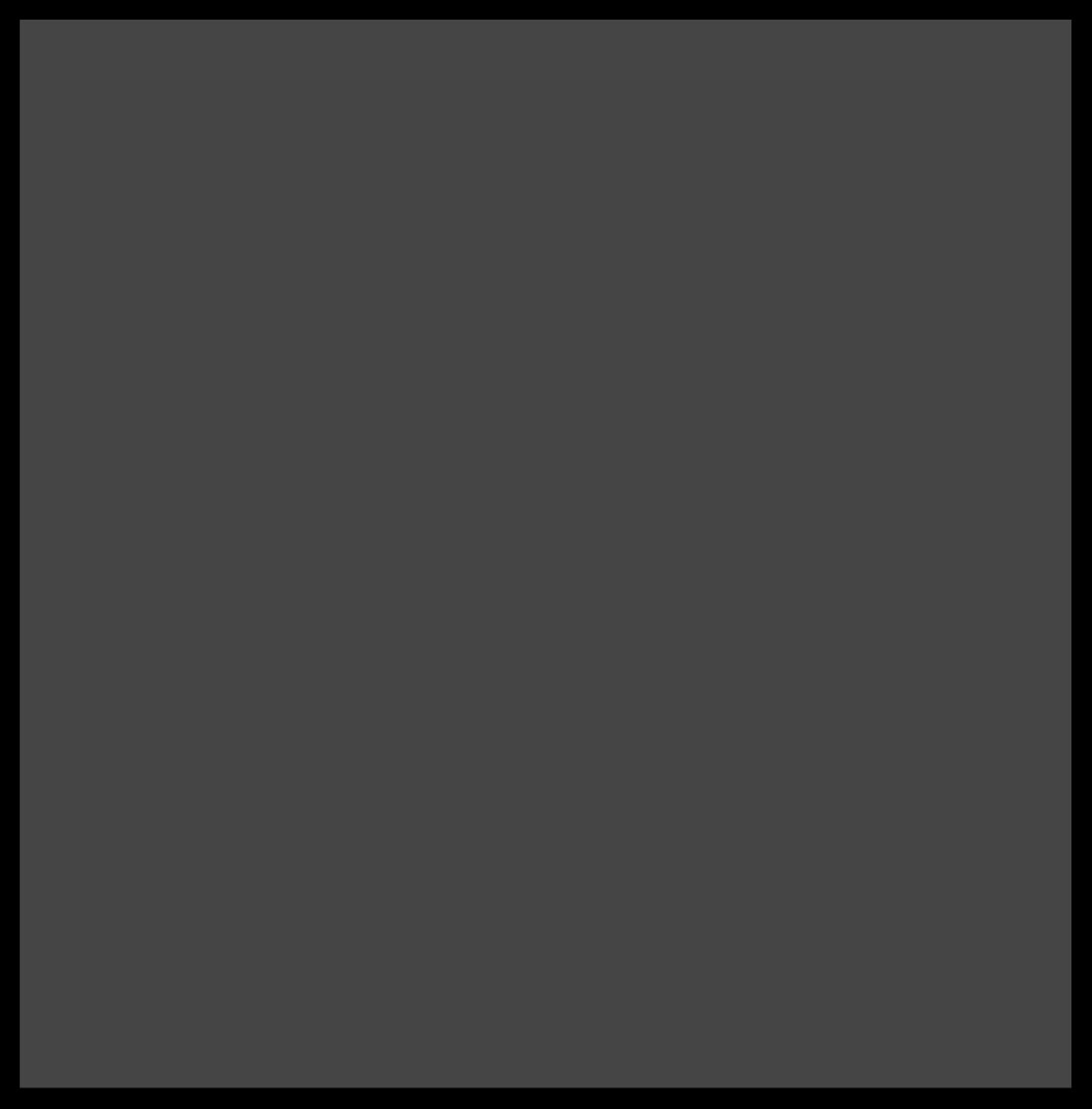}s is $\pattern_1=([0.8,0.2,0,\cdots,0],0)$, and the template at all \protect\includegraphics[height=0.7em, bb=0 0 550 550]{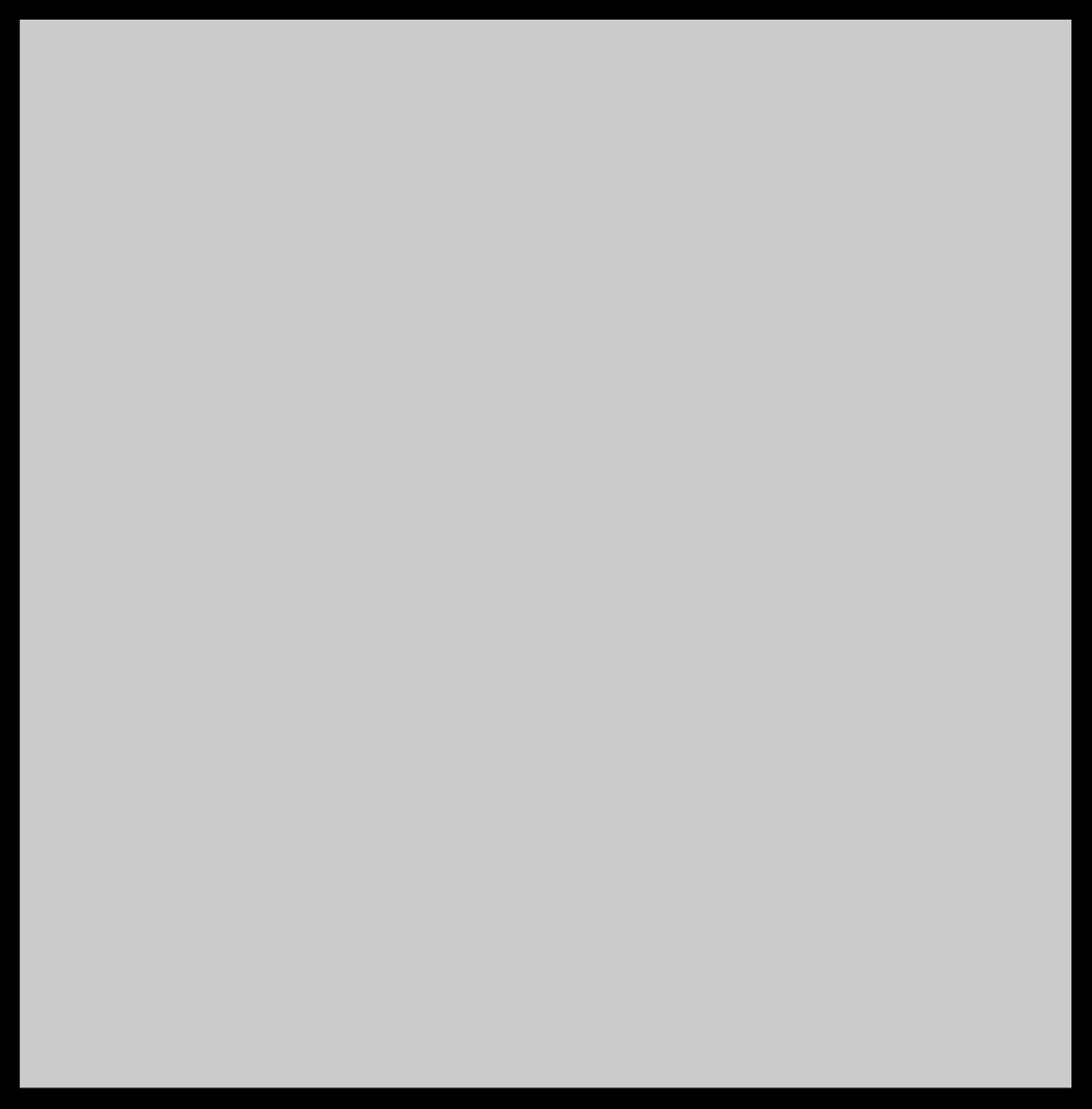}s is $\pattern_2=([0.6,0.2,0.2,0,\cdots,0], 0)$. }
    % \fh{In the example of Figure 2, it is not really clear why the two reported vectors are the only TTs. It could be expanded and clarified a little.}}
    \label{fig:pat_eg}
\end{figure}
\textbf{Explanations for Figure~\ref{fig:pat_eg}.}
Consider a $5\times 5$ gridworld where the agent has 4 actions: $\uparrow, \downarrow, \leftarrow$ and $\rightarrow$, as well as 25 states, as shown in Figure~\ref{sfig:pateg_up}, ~\ref{sfig:pateg_down}, ~\ref{sfig:pateg_left} and ~\ref{sfig:pateg_right}. Thus there are 100 distinct state-action pairs in total. 
Since the slipping probability is 0.4, which means action $\uparrow$ will become $\leftarrow$ or $\rightarrow$ with probability 0.2 respectively,
we know the transition probability $p(\cdot|s=1, a=\uparrow)$ is 
\begin{equation}
p(s^\prime|s=1, a=\uparrow)=
\begin{cases}
0.8 & \text{if } s^\prime=1 \\
0.2 & \text{if } s^\prime=2 \\
0 & \text{otherwise} 
\end{cases}
\end{equation} 
By re-ordering, its \patname is $([0.8,0.2,0,\cdots,0],0)$.

Similarly, for state 2 and action $\uparrow$,
\begin{equation}
p(s^\prime|s=2, a=\uparrow)=
\begin{cases}
0.6 & \text{if } s^\prime=2 \\
0.2 & \text{if } s^\prime=1 \text{ or } 3 \\
0 & \text{otherwise} 
\end{cases}
\end{equation} 
Its \patname is $([0.6,0.2,0.2,0,\cdots,0], 0)$.

For state 3 and action $\uparrow$,
\begin{equation}
p(s^\prime|s=3, a=\uparrow)=
\begin{cases}
0.6 & \text{if } s^\prime=3 \\
0.2 & \text{if } s^\prime=2 \text{ or } 4 \\
0 & \text{otherwise} 
\end{cases}
\end{equation} 
Its \patname is also $([0.6,0.2,0.2,0,\cdots,0], 0)$.

In this way, we can find that there are only 2 distinct \patnames, like shown in the figure, which is much less than the number of state-action pairs. The state-action pairs with the same \patnames are able to share the ``relative'' probability of transitioning.

%!TEX root = ../0_aaai2021_multitaskrl_main.tex

\subsection{An Example Illustrating Sample Efficiency of Our Algorithm}
\label{app:example}

% \textbf{An example: saving $99\%$ samples via augmented estimation.}
We use the gridworld in Figure~\ref{fig:pat_eg} as an example to illustrate how the augmented estimation saves samples by estimating \patnames instead of every s-a pairs. 
For simplicity, we assume the possibility of arbitrarily sampling any s-a pair and observing its transitions. 
The gridworld in Figure~\ref{fig:pat_eg} has 100 distinct s-a transition dynamics, but only 2 distinct \patnames. The distance between the 2 \patnames is $0.28$.

The objective is to estimate $p(s,a)$ such that $\| \hat{p}(s,a)-p(s,a) \| \leq 0.01$ for all $(s,a)$ with probability $95\%$. 
Using the conventional estimation, the total number of samples we need is\footnote{According to Hoeffding's inequality, $\mathcal{O}(\frac{1}{\alpha^2}\ln \frac{1}{\delta})$ samples are needed to achieve an $\alpha$-accurate estimation with probability $1-\delta$.} $\mathcal{O}(100\times \frac{1}{0.01^2} \ln \frac{100}{0.05}) \approx 7.6 \times 10^6$. 
However our proposed augmented estimation only requires  $\sim 1.25\%$ of samples needed by the conventional estimation, 
since it takes $\mathcal{O}(100\times \frac{1}{0.28^2} \ln \frac{100}{0.025}) \approx 1.1 \times 10^4$ samples to correctly identify the \patnames of all s-a pairs with probability $1-\delta/2$,
%~\fh{you are not clustering templates. you are clustering the permuted transition dynamics and extract a template for clustered transition dynamics}, 
plus $\mathcal{O}(2 \times \frac{1}{0.01^2} \ln \frac{2}{0.025}) \approx 8.4 \times 10^4$ samples to get $0.01$-accurate estimations of 2 \patnames with probability $1-\delta/2$.
%!TEX root = ../0_aaai2021_multitaskrl_main.tex

\section{Additional Related Works}\label{app:related}

% \fh{~\cite{abel2018policy} PAC-MDP algorithms for multi-task RL, need some explanation.}

\noindent \textbf{Non-PAC MTRL Algorithms. }
Besides the above methods with PAC guarantees, there are many interesting approaches aiming to effectively transfer knowledge across tasks. 
Some approaches augment the learning of a new task by reusing the policies learned from previous tasks~\cite{ramamoorthy2013clustering,li2018optimal}, though a library of reliable source policies is usually needed. 
The hierarchical multi-task learning algorithm proposed by~\citet{wilson2007multi} learns a Bayesian mixture model from previous tasks, and use the learned distribution as a prior of new tasks. 
Although good experimental results are shown, there are no theoretical guarantees. \cite{abel2018state} introduce two types of state-abstractions that can reduce the problem complexity and improve the learning of new tasks. 
However, as shown in the paper, the proposed abstractions, when combined with PAC-MDP algorithms such as RMax~\cite{brafman2003rmax}, the PAC guarantee does not hold anymore and the number of mistakes made by the agent can be arbitrarily large. In this paper, we show our proposed method outperforms the abstraction method empirically.

\noindent \textbf{Empirical Studies.}
MTRL/lifelong RL and \textit{Transfer Learning (TL)}~\cite{torrey2010transfer} are closely related, and have been studied for years. 
% TLRL only focuses on improving learning on the target task, whereas MTL and lifelong RL optimize performance over all tasks.
\citet{taylor2009transfer} survey a wide range of empirical results on transferring knowledge among tasks, and point out some problems of previous works, including negative transfer, partially due to the lack of theoretical analysis.

\noindent\textbf{Cross-Domain Learning.}
Knowledge transfer between tasks with various state/action spaces (task domains) is also an important topic. 
As summarized by \cite{taylor2009transfer}, most early works require hand-coded inter-task mappings~\cite{taylor2007representation}, or only learn from unchanged problem representations~\cite{konidaris2006autonomous,sharma2007transfer}.
Recently, \citet{ammar2015autonomous} propose an algorithm that can perform cross-domain transfer efficiently. 
The authors assume tasks are from a finite set of domains, and parametrize each task's policy by the product of a shared knowledge base and task-specific coefficients. 
However, although convergence guarantee is provided, there is no guarantee for sample efficiency.
\citet{mann2012directed} study cross-domain transfer learning with an inter-task mapping of s-a pairs (with similarly-bounded Q values). However, such inter-task mappings are often not available in multi-task RL.

%!TEX root = ../0_aaai2021_multitaskrl_main.tex

\section{Algorithm Pseudo-code}\label{app:algo}

% We present the \patname functions in Procedure~\ref{alg:tt_func}, including \textsc{GEN-TT}, which generates a new \patname by summarizing and re-ordering visits, \textsc{TT-update}, which updates an existing \patname by incorporating the visits of a state-action pair with the same \patname, and \textsc{augment}, which augments a state-action pair's visits by incorporating the accumulative visits of a \patname.

% \input{algos/tt_func}
% \input{algos/single}

Procedure~\ref{alg:group} shows \ourmodtwo introduced in Section~\ref{sec:alg2}, which learns \patnames as Procedure~\ref{alg:main} does, and also expedites learning by clustering models.
%!TEX root = ../0_aaai_all.tex

\begin{algorithm}[!ht]
   \caption{\mbox{\ourmodtwofull (\ourmodtwo)}}
   \label{alg:group}
   \hspace*{\algorithmicindent} \textbf{Input:} 
   $\hat{\tau}$; $\epsilon$; $\gamma$; $\thressm$; $\threslg$ (same as Procedure~\ref{alg:main}); \\
   number of tasks in the first phase $T_1$; number of models $C$; model error tolerance $\eta$\\
   \hspace*{\algorithmicindent} \textbf{Output} 
   Near-optimal policies $\{\pi_t\}_{t=1,2,\cdots}$
\begin{algorithmic}[1]
   % \State {\bfseries Input:} $\hat{\tau}$; $\epsilon$; $\gamma$; $\thressm$; $\threslg$ (same as Algorithm~\ref{alg:main});\\
   % number of tasks in the first phase $T_1$; number of models $C$; model error tolerance $\eta$
   % \STATE {\bfseries Output:} Near-optimal policies $\{\pi_t\}_{t=1,2,\cdots}$
   \State Initialize the \patname group set $\patset$, the \patname visit set $\patvisset$, and the MDP group set $\mdpset$ as empty
   % \State Set the value of $\thressm$ and $\threslg$ according to Equation~\ref{eq:thres}~\fh{delete this, as it is shown in algorithm 1 already?}
   \For{$t \gets 1,2,\cdots, T_1$} \Comment{{\scriptsize{Phase 1}}}
    \State Run Procedure~\ref{alg:main} Line 3-18 and get visits $\bm{n}(s,a,\cdot)$ and $R(s,a)$, $ \forall (s,a) \in (\states,\actions)$
   \EndFor
   \State Cluster the past $T_1$ tasks into $C$ groups (store in $\mdpset$). \Comment{{\scriptsize{model clustering}}}
   \State $\pattern_{(s,a,c)}, \permut_{(s,a,c)} \gets$ \Call{gen-TT}{$\bm{n}_c(s,a), R_c(s,a)$} $\forall (s,a,c)$%\in (\states,\actions,\mdpset)$ 
   % \NoNumber{ $\pattern_{(s,a,c)} \gets \left( \frac{descending(\bm{n}(s,a,c,\cdot))}{n(s,a,c)}, \frac{R(s,a,c)}{n(s,a,c)} \right)$} 
   % ~\fh{how to compute \patname and permutations should be specified. Quantities needed for computing \patname are known after running Line 4-30}
   \For{$t \gets T_1+1, T_1+2, \cdots$} \Comment{{\scriptsize{Phase 2}}}
    \State Receive a task $M_t$, do initializations as Line 3 in Procedure~\ref{alg:main}
    \State Initialize model score $u(c)\gets \eta, \forall c \in \mdpset$ \Comment{{\scriptsize{$u(c)$ measures how possible $c$ is the true model for the current task}}}
    \For{$h \gets 1,2,\cdots,H$}
      \State Run Procedure~\ref{alg:main} Line 5-16, and get updated $\bm{n}(s_h, a_h,\cdot)$, $R(s_h,a_h)$, $\pattern_{(s_h,a_h)}$, $\permut(s_h,a_h)$ %of current state action $(s_h,a_h)$
      \If{$\lVert \bm{n}(s_h, a_h,\cdot)\rVert_{\ell_1}= \thressm$} \Comment{{\scriptsize{\patname is identified}}}
      \For{$c\in\mdpset$}
      \If{$\pattern_{(s_h,a_h,c)}\neq \pattern_{(s_h,a_h)}$ \textbf{or} $\permut_{(s_h,a_h,c)}\neq \permut_{(s_h,a_h)}$}
      \State $u(c) \gets u(c) - 1$ \Comment{{\scriptsize{group identification}}}
      %, $\forall c\in\mdpset$ s.t \\
      %\quad $\pattern_{(s,a,c)}\neq \pattern_{(s,a)}$ or the permutation $\permut_{(s,a,c)}\neq \permut_{(s,a)}$
      \EndIf
      \EndFor
      \EndIf
      \If{$\exists$ only 1 group $c^*\in \mdpset$ s.t. $u(c^*)>0$}
        \State {{\Call{augment}{$\patvis_{\pattern_{(s,a,c^*)}}$,$\bm{n}(s,a,\cdot)$,$R(s,a)$,$\sigma_{s,a}$} $\forall(s,a)$}}
          % \State Get the permutation $\permut$ that orders $\bm{n}(s,a,\cdot)$ descendingly
          % \State $\bm{n}(s,a,\cdot) \leftarrow \bm{n}(s,a,\cdot) + \permut^{-1}(\patvisp_{\pattern_{s,a,c}})$,\\ $R(s,a) \leftarrow R(s,a) + \patvisr_{\pattern_{s,a,c}}$, \\ $n(s,a) \leftarrow n(s,a) + \sum_{s^\prime} \bm{n}(s,a,s^\prime)$
      \EndIf
    \EndFor
    \State Run Procedure~\ref{alg:main} Line 17-18
    \EndFor 
\end{algorithmic}
\end{algorithm}

%!TEX root = ../0_aaai2021_multitaskrl_main.tex

\section{Additional Definitions}
\label{app:def}

\begin{definition} [\patname Gap]
\label{def:pat_gap}
  Define the {\patname distance} between two \patnames ~$\pattern_a$ and $\pattern_b$  as 
  $\rho(\pattern_a,\pattern_b) = \| \patp_a - \patp_b \|_2 + | \patr_a - \patr_b |.$ Suppose there is a minimum \patname distance $\tau$, such that for any two different \patnames $\pattern_a, \pattern_b \in \patset$, $\rho(\pattern_a,\pattern_b) \geq \tau$. We name $\tau$ as \textit{\patname gap}.
\end{definition}
\paragraph{Remark.} According to Lemma~\ref{lem:bound}, the \patname gap between any two \patnames will not exceed 2 (suppose reward is in between 0 and 1). 

%!TEX root = ../0_aaai2021_multitaskrl_main.tex
\begin{lemma}
\label{lem:bound}
Let $\mathbf{a}=(a_1,\dots, a_n)$ and $\mathbf{b}=(b_1,\dots, b_n)$ be two vectors in $\mathbb{R}^n$ such that $\sum_{i=1}^n a_i=\sum_{j=1}^n b_j=1$. Moreover, assume there hold
\begin{align*}
    &1 \geq a_1\geq a_2\geq\cdots \geq a_n \geq 0,\\
    &1 \geq b_1\geq b_2\geq\cdots \geq b_n \geq 0.
\end{align*}
Then 
\begin{align}
    ||\mathbf{a}-\mathbf{b}||^2_{2}:=\sum_{i=1}^n (a_i-b_i)^2\leq \frac{n-1}{n}.\label{ineq: a and b}
\end{align}
The equality holds when we choose, e.g., $\mathbf{a}=(1,0,\dots, 0)$ and $\mathbf{b}=(\frac{1}{n},\frac{1}{n},\dots, \frac{1}{n})$.
\end{lemma}

\begin{proof}
We prove the lemma by induction. 
\begin{itemize}
    \item When $n=1$, $\mathbf{a}=\mathbf{b}=1$. Thus the inequality is trivial.
    \item Assume \eqref{ineq: a and b} holds for $n=k$, $k\geq 1$. We show that \eqref{ineq: a and b} is also true for $n=k+1$. Given vectors $\mathbf{a}=(a_1,\dots, a_k, a_{k+1})$ and $\mathbf{b}=(b_1,\dots, b_k, b_{k+1})$ such that they satisfy the conditions in the lemma, construct two new vectors:
    \begin{align*}
        \mathbf{a}'&:=\left(a_1+\frac{a_{k+1}}{k},\dots, a_k+\frac{a_{k+1}}{k}\right),\\
        \mathbf{b}'&:=\left(b_1+\frac{b_{k+1}}{k},\dots, b_k+\frac{b_{k+1}}{k}\right).
    \end{align*}
    It is obvious that $\mathbf{a}'$ and $\mathbf{b}'$ satisfy the conditions in the induction hypothesis. Thus,
    \begin{align}\label{ineq: a' and b'}
        ||\mathbf{a}'-\mathbf{b}'||^2_2\leq \frac{k-1}{k}.
    \end{align}
    We calculate
    \begin{align}
        ||\mathbf{a}'-\mathbf{b}'||^2_2&=\sum_{i=1}^k (a_i-b_i+\frac{a_{k+1}-b_{k+1}}{k})^2\\
        &=\sum_{i=1}^k (a_i-b_i)^2+\frac{(a_{k+1}-b_{k+1})^2}{k}+\frac{2(a_{k+1}-b_{k+1})}{k}\cdot \sum_{i=1}^k(a_i-b_i)\\
        &=\sum_{i=1}^k (a_i-b_i)^2-\frac{(a_{k+1}-b_{k+1})^2}{k},\label{eq: a' and b'}
    \end{align}
    where the last equality comes from the fact
    \[
    \sum_{i=1}^k (a_i-b_i)=b_{k+1}-a_{k+1}.
    \]
    By assumptions, 
    \[
    \sum_{i=1}^{k+1} a_i=1\;\text{and}\; a_1\geq a_2\geq\cdots \geq a_{k+1} \geq 0,
    \]
     hence
     \begin{align}
     a_{k+1}\in \left[0, \frac{1}{k+1}\right].\label{eq: bound of a_k+1}
     \end{align}
     Similarly, 
     \begin{align}
     b_{k+1}\in \left[0, \frac{1}{k+1}\right].\label{eq: bound of b_k+1}
     \end{align}
     Combine \eqref{ineq: a' and b'}, \eqref{eq: a' and b'}, \eqref{eq: bound of a_k+1} and \eqref{eq: bound of b_k+1},
     \begin{align*}
         ||\mathbf{a}-\mathbf{b}||_{2}^2&=\sum_{i=1}^{k+1} (a_i-b_i)^2\\
         &\leq \frac{k-1}{k}+\frac{k+1}{k}(a_{k+1}-b_{k+1})^2\\
         &\leq \frac{k-1}{k}+\frac{k+1}{k}\cdot \frac{1}{(k+1)^2}\\
         &\leq \frac{k}{k+1}.
     \end{align*}
     Thus \eqref{ineq: a and b} also holds for $n=k+1$. The proof is finished.
\end{itemize}

\end{proof}

The permutation from an s-a transition dynamics to a \patname is recorded by a ranking permutation defined as below.

\begin{definition} [Ranking Permutation]
    For an s-a pair $(s,a)$ with transition probability vector $\mathbf{p} \in \mathbb{R}^S$ where $p_i=p(s_i|s,a)$, by sorting its elements from the largest to the smallest value, we get an ordered vector $\patp$.
    Define function $\permut: \{1,\cdots,S\} \to \{1, \cdots, S\}$ as a mapping from ranking to the indices in $\mathbf{p}$. For example, $\permut(i)$ is the index of the $i$-th largest element of $\mathbf{p}$, i.e., $\patp_i = \mathbf{p}_{\permut(i)}$. The inverse function $\permut^{-1}$ maps indices to ranking. So $\permut^{-1}(j)$ is the ranking of $\mathbf{p}_j$, i.e., $\patp_{\permut^{-1}(j)}=\mathbf{p}_j$.
    The way~{way?} ordering is unique if for any $\mathbf{p}_i = \mathbf{p}_j$ and $i<j$, we put $\mathbf{p}_i$ before $\mathbf{p}_j$ in $\patp$. As a result, $\permut(\cdot)$ is a bijection and can be regarded as a permutation. We call it \textbf{ranking permutation} of $(s,a)$.
\end{definition}

For simplicity, we slightly abuse notation and use $\permut(\mathbf{p})$ to denote the re-ordered vector. Thus $\patp=\permut(\mathbf{p})$ and $\mathbf{p}=\permut^{-1}(\patp)$.
% \input{algos/patternize}

% In an environment or many environments, all s-a dynamics could be converted to multiple \patnames.~\fh{not true!} 

\begin{definition}[Ranking Gap]
\label{def:prob_gap}
    Define $\nu$ as the minimal notable ranking gap, such that for any $\pattern \in \patset$, if $\patp_i > \patp_j$ are two adjacent elements in $\patp$ and $\patp_i - \patp_j \geq \mathcal{O}(\frac{\epsilon (1-\gamma)}{\sqrt{S}V_{\max}})$, then $\patp_i - \patp_j \geq \nu$ holds. In other words, two adjacent elements of $\patp$ satisfy either $\patp_i - \patp_j \geq \nu$ or $\patp_i - \patp_j \leq \mathcal{O}(\frac{\epsilon (1-\gamma)}{\sqrt{S}V_{\max}})$.
\end{definition}

If two adjacent elements are different by no more than $\mathcal{O}(\frac{\epsilon (1-\gamma)}{\sqrt{S}V_{\max}})$, then the corruption can be ignored because it will not influence the value of the policy too much, as proved in Lemma~\ref{lem:simu}. Otherwise we need adequate samples to make sure the ranking of elements will succeed. 
% Similarly, if $\tau$ is smaller than $\mathcal{O}(\frac{\epsilon (1-\gamma)}{SV_{\max}})$, then take $\hat{\tau}=\mathcal{O}(\frac{\epsilon (1-\gamma)}{SV_{\max}})$ as 

% \input{appendix/bound}
%!TEX root = ../0_aaai2021_multitaskrl_main.tex

\section{Proofs of Main Theorems}
\label{app:proof}
% We first make the following mild assumption, which is commonly used in RL~\cite{jaksch2010near}, and it ensures the reachability of all state from any state on average.

% \begin{assumption}
% 	There is a diameter $D$ such that any state $s^\prime$ is reachable from any states $s$ in at most $D$ steps on average.
% \end{assumption}

% We define $\nu$ as the ranking gap. For any $\pattern \in \patset$, if $\patp_i > \patp_j$ are two adjacent elements in $\patp$, then either $\patp_i - \patp_j \geq \nu$, or $\patp_i - \patp_j \leq \mathcal{O}(\frac{\epsilon (1-\gamma)}{SV_{\max}})$. The ranking gap implies that for any s-a pair, the probabilities of transitioning to any two states are either very close, or substantially different. \fh{I deleted this paragraph in main text. Where to put this paragraph? I don't think it is clear. }

\subsection{Proof of Theorem~\ref{thm:main}}
% \begin{theorem}[Sample Complexity of Algorithm~\ref{alg:main}]
% The sample complexity is $\mathcal{\tilde{O}}(\frac{\zeta_1 V_{\max}}{\epsilon (1-\gamma)})$ where
% $$
% \zeta_1 = \mathcal{\tilde{O}}\left(T^\prime \zeta_s + (T-T^\prime) (\frac{SA}{\tau^2})\right),
% $$
% and
% $$
% T^\prime = \mathcal{O}\left(\frac{V_{\max}^2 }{p_{\min} SA \epsilon^2 (1-\gamma)^2} \ln \frac{G}{\delta}\right).
% $$
% \end{theorem}

To prove Theorem~\ref{thm:main}, we first present Lemma~\ref{lem:concen}, Lemma~\ref{lem:permut}, Lemma~\ref{lem:horizon}, Lemma~\ref{lem:simu} and Lemma~\ref{lem:visits}.

Lemma~\ref{lem:concen} and Lemma~\ref{lem:permut} provide the sample size requirements for correctly identifying the \patname of an s-a pair. 

\begin{lemma}
\label{lem:concen}
For any state-action pair, suppose the ranking permutation of the estimated probability vector is the same with that of the underlying probability vector, then it would be identified to its corresponding \patname group correctly with $\mathcal{O}(\frac{1}{\tau^2}\ln \frac{1}{\delta})$ samples, with probability at least $1-\delta$, where $\tau$ is the \patname gap defined in Definition~\ref{def:pat_gap}. 
\end{lemma}

\begin{proof}

For an state-action pair $(s,a)$, define the observation vector of the $i^\tha$ sample as $Z_i = [\mathbb{I}_{s^\prime=s_1},\mathbb{I}_{s^\prime=s_2}, \cdots, \mathbb{I}_{s^\prime=s_S}, r]$. 

Define $X_n = \sum_{i=1}^n Z_i - n \theta(s,a)$, and set $X_0=0$.

We first prove the sequence $\{X_n\}$ is a vector-valued martingale.

\begin{align*}
\mathrm{E}\left(X_{n} | X_{0}, X_{1}, \cdots, X_{n-1}\right) 
&= E\left[\sum_{i=1}^n Z_i - n \theta(s,a)| X_{0}, X_{1}, \ldots, X_{n-1}\right] \\
&= E\left[\sum_{i=1}^{n-1} Z_i - (n-1) \theta(s,a) + Z_n - \theta(s,a)| X_{0}, X_{1}, \ldots, X_{n-1}\right]\\
&= X_{n-1} + E[Z_n-\theta(s,a)] \\
&= X_{n-1}
\end{align*}

Obviously, $E[\|X_n\|] < \infty$ for all $n$. Thus $\{X_n\}$ is a (strong) martingale.

By application of the extended Hoeffding's inequality~\cite{hayes2005large}, we get 
\begin{align*}
	Pr(\|\frac{1}{n}\sum_{i=1}^n Z_i - \theta(s,a)\| \geq \epsilon) \leq 2 \mathrm{e}^{2-\frac{n\epsilon^2}{2}}
\end{align*}

Set the failure probability as $\delta$, we obtain $n \geq \mathcal{O}(\frac{1}{\epsilon^2}\ln \frac{1}{\delta})$.

\end{proof}

% \begin{definition}[Probability Gap]
% \label{def:prob_gap}
%     Define $\nu$ as the minimal notable probability gap, such that for any $\pattern \in \patset$, if $\patp_i > \patp_j$ are two adjacent elements in $\patp$ and $\patp_i - \patp_j \geq \mathcal{O}(\frac{\epsilon (1-\gamma)}{SV_{\max}})$, then $\patp_i - \patp_j \geq \nu$ holds. In other words, two adjacent elements of $\patp$ satisfy either $\patp_i - \patp_j \geq \nu$ or $\patp_i - \patp_j \leq \mathcal{O}(\frac{\epsilon (1-\gamma)}{SV_{\max}})$.
% \end{definition}

% \textbf{Remark. } The existence of $\nu$ is to prevent the conversion between s-a dynamics and \patnames from being corrupted by a wrong ranking of elements. If two adjacent elements are different by no more than $\mathcal{O}(\frac{\epsilon (1-\gamma)}{SV_{\max}})$, then the corruption can be ignored because it will not influence the value of the policy too much, as proved in Lemma~\ref{lem:simu}. Otherwise we need adequate samples to make sure the ranking of elements will succeed. 

Lemma~\ref{lem:concen} assumes perfect permutation, and Lemma~\ref{lem:permut} addresses the problem of how to avoid notable corruptions of permutations. For ease of illustrating, we define the concept \textit{almost the same} and \textit{almost correct} in Definition~\ref{def:almost}.

\begin{definition}[Almost the same and almost correct]
\label{def:almost}
If for two probability vectors $p$ and $p^\prime$, their ranking permutations $\permut$ and $\permut^\prime$ are the same except for elements whose difference is smaller than $\mathcal{O}(\frac{\epsilon (1-\gamma)}{\sqrt{S}V_{\max}})$, then we call $\permut$ and $\permut^\prime$ \textit{almost the same}. If $p^\prime$ is the approximation of the ground truth $p$, then we call $\permut^\prime$ \textit{almost correct}.
\end{definition}

\begin{lemma}
\label{lem:permut}
With $\mathcal{\tilde{O}}(\frac{1}{\nu^2}\ln \frac{S}{\delta})$ samples of a s-a pair, its transition permutation will be almost correct, with probability $1-\delta$. 
\end{lemma}

\begin{proof}
	Suppose there is a transition probability vector $p$ with length $l$, as well as transition-difference gap $\nu$. We estimate $p$ by randomly sampling indices $1, \cdots, i, \cdots, l$ according to the probability distribution $p$. Let $\hat{p} = [\frac{n(1)}{n}, \cdots, \frac{n(l)}{n}]$.

	For any two adjacent elements $p_i$ and $p_j$ (adjacent means there is no $p_k$ whose value is between $p_i$ and $p_j$), we want our estimations $\frac{n(i)}{n}$ and $\frac{n(j)}{n}$ to satisfy $\frac{n(i)}{n} > \frac{n(j)}{n}$. 
	It is sufficient if we guarantee $\frac{n(i)}{n}$ and $\frac{n(j)}{n}$ are respectively $\nu/2$-close to $p_i$ and $p_j$. By Hoeffding's inequality,
 	$$
 		P( | \frac{n(i)}{n} - p_i | > \frac{\nu}{2}) \leq 2 \exp(-2n(\nu/2)^2)
 	$$
 	Therefore, $n \geq \mathcal{O}(\frac{1}{\nu^2} \ln \frac{1}{\delta})$ is sufficient. By union bound, we have $n \geq \mathcal{O}(\frac{1}{\nu^2} \ln \frac{S}{\delta})$.
\end{proof}

Lemma~\ref{lem:concen} and Lemma~\ref{lem:permut} imply that the small threshold should satisfy $\thressm=\mathcal{\tilde{O}}(\frac{1}{\min\{\tau^2,\nu^2\}})$.

Then, Lemma~\ref{lem:horizon} claims if horizon is set to be large enough, all s-a pairs will have sufficient samples to be correctly grouped.

\begin{lemma}
\label{lem:horizon}
If $H=\mathcal{\tilde{O}}(\frac{DSA}{\mingap^2})$, all state-action pairs in the task will have at least $\mathcal{\tilde{O}}(\frac{1}{\mingap^2}\ln \frac{TSA}{\delta})$ samples with probability $1-\delta$.
\end{lemma}

The proof of Lemma~\ref{lem:horizon} is similar to Lemma 2.1 in paper ~\cite{brunskill2013sample}. 

Lemma~\ref{lem:simu} is a variant of the ``simulation lemma''~\cite{kearns2002near,brafman2003rmax} with \patname estimation.

\begin{lemma}
\label{lem:simu}
For any two MDPs $M$ and $\tilde{M}$ with the same $\states, \actions, \mu, \gamma$, if for any s-a pair $(s,a)$, the ranking permutations of $p(s,a)$ and $\tilde{p}(s,a)$ are almost the same, and $\pattern = (desc(p(s,a)), r(s,a))$ as well as $\tilde{\pattern}=(desc(\tilde{p}(s,a)),\tilde{r}(s,a))$ satisfy $\|\patp - \tilde{\pattern}^{(p)}\| \leq \mathcal{O}(\frac{\epsilon (1-\gamma)}{V_{\max}})$ and $| \patr - \tilde{g}^{(r)} | \leq \mathcal{O}(\frac{\epsilon (1-\gamma)}{V_{\max}})$, then for any policy $\pi$, $| V^\pi_M - V^\pi_{\tilde{M}} | \leq \epsilon$.
\end{lemma}

\begin{proof}
	For an s-a pair $(s,a)$, suppose its ranking permutation in $M$ is $\permut$, and the ranking permutation in $\tilde{M}$ is $\tilde{\permut}$. 

	We first assume $\permut$ and $\tilde{\permut}$ are exactly the same. So we have $\patp=\permut(p(s,a))$ and $\tilde{\pattern}^{(p)}=\permut(\tilde{p}(s,a))$.

	Thus $\|\patp - \tilde{\pattern}^{(p)}\| \leq \mathcal{O}(\frac{\epsilon (1-\gamma)}{V_{\max}})$ implies $\| p(s,a) - \tilde{p}(s,a) \| \leq \mathcal{O}(\frac{\epsilon (1-\gamma)}{V_{\max}})$, because of the property of permutation.

	Similarly, $| \patr - \tilde{g}^{(r)} | \leq \mathcal{O}(\frac{\epsilon (1-\gamma)}{V_{\max}})$ implies $| r(s,a) - \tilde{r}(s,a)| \leq \mathcal{O}(\frac{\epsilon (1-\gamma)}{V_{\max}})$. 

	Then, following the standard proof~\cite{strehl2008analysis,strehl2006incremental}, it is easy to show $| V^\pi_M - V^\pi_{\tilde{M}} | \leq \epsilon$.

	Next, we allow $\permut$ and $\tilde{\permut}$ be almost the same (see Definition~\ref{def:almost}), and show that it only causes up to a constant factor increase in the value difference $| V^\pi_M - V^\pi_{\tilde{M}} |$.

	Without loss of generality, assume $\tilde{\permut}$ only reverses $\permut$ in indices $i$ and $j$, i.e., $\permut(i) = \tilde{\permut}(j)$ and $\permut(j) = \tilde{\permut}(i)$. According to the definition of almost the same, $p_{\permut(i)} - p_{\permut(j)} = p_{\tilde{\permut}(j)} - p_{\tilde{\permut}(i)} \leq \mathcal{O}(\frac{\epsilon (1-\gamma)}{\sqrt{S} V_{\max}})$ 
	Then we have
	\begin{align*}
		\| p(s,a) - \tilde{p}(s,a) \| &= \| \permut^{-1} (\patp) - \tilde{\permut}^{-1} (\tilde{\pattern}^{(p)}) \| \\
		&= | g^{(p)}_1 - \tilde{g}^{(p)}_1 | + \cdots + | g^{(p)}_i - \tilde{g}^{(p)}_j | + | g^{(p)}_j - \tilde{g}^{(p)}_i | + \cdots + | g^{(p)}_S - \tilde{g}^{(p)}_S | \\
		&\leq \|\patp - \tilde{\pattern}^{(p)}\| + | g^{(p)}_i - \tilde{g}^{(p)}_j | + | g^{(p)}_j - \tilde{g}^{(p)}_i | \\
		&\leq \mathcal{O}(\frac{\epsilon (1-\gamma)}{V_{\max}}) + \mathcal{O}(\frac{2 \epsilon (1-\gamma)}{\sqrt{S} V_{\max}}) \\
		&\leq (1 + \frac{2}{\sqrt{S}}) \mathcal{O}( \frac{\epsilon (1-\gamma)}{V_{\max}})
	\end{align*}

	Therefore, if $\tilde{\permut}$ differs with $\permut$ in all indices, as long as they are almost the same, $| V^\pi_M - V^\pi_{\tilde{M}}| \leq 2 \epsilon$. By adjusting the constant factor, $| V^\pi_M - V^\pi_{\tilde{M}}| \leq \epsilon$ also holds.
\end{proof}

According to Lemma~\ref{lem:simu}, if each \patname gets $\mathcal{O}(\frac{\epsilon (1-\gamma)}{V_{\max}})$-accurate estimation, then all the s-a transition dynamics associated with the same \patname will be accurate enough to generate an $\epsilon$-optimal policy. Therefore, the regular known threshold $\threslg$ is still the same as in RMax, i.e., $\threslg = \mathcal{\tilde{O}}(\frac{S V^2_{\max}}{\epsilon^2 (1-\gamma)^2})$. 
Note that the small threshold should not exceed the regular threshold, so $\thressm=\mathcal{\tilde{O}}(\frac{1}{\mingap^2})$, where $\mingap = \max \{ \min (\tau,\nu), \mathcal{O}(\frac{\epsilon (1-\gamma)}{\sqrt{S}V_{\max}})\}$. If $\tau$ or $\nu$ is smaller than $\mathcal{O}(\frac{\epsilon (1-\gamma)}{\sqrt{S}V_{\max}})$, then the small threshold becomes the regular threshold and \ourmod degenerates to RMax.

Next, we show in Lemma~\ref{lem:visits} the total number of visits to unknown state-action pairs during $T$ tasks.

\begin{lemma}
\label{lem:visits}
The total number of visits to unknown s-a pairs during the execution of Algorithm~\ref{alg:main} for $T$ tasks is
\begin{equation}
\label{eq:visits}
	\mathcal{\tilde{O}} \Big(\frac{TSA}{\mingap^2 } + \frac{S G V^2_{\max}}{\epsilon^2 (1-\gamma)^2} \Big)
\end{equation}
\end{lemma}

\begin{proof}
	For every task, Algorithm~\ref{alg:main} first uses known threshold $\thressm=\mathcal{\tilde{O}}(\frac{1}{\mingap^2})$ for all s-a pairs. And the first $\thressm$ visits to an s-a pair are all visits to unknowns. So all the $SA$ s-a pairs over $T$ tasks take $\mathcal{O}(\frac{TSA}{\mingap^2 })$ steps of visiting unknowns in total.

	Once an s-a pair is roughly known (having visits more than $\thressm$), the \patname is identified, and the known threshold is changed to $\threslg$ for the s-a pair. If the corresponding \patname is fully known (having visits more than $\threslg$), then the s-a pair immediately becomes fully known by incorporating all visit counts of the \patname. If the corresponding \patname is not fully known yet, visits to the s-a pair are still counted as visits to unknown, until the \patname is known. Therefore, for every possible \patname, there are $\threslg$ unknown visits. And $G$ \patnames result in $G\threslg$ unknown visits, which is the second term in Equation~\ref{eq:visits}
\end{proof}

Now we can proceed to prove the main theorem.

\begin{proof}
(of Theorem~\ref{thm:main})
	We apply the PAC-MDP theorem proposed by~\cite{strehl2006incremental} to get the sample complexity of \ourmod. Proposition 1 in~\cite{strehl2006incremental} claims that any greedy learning algorithm with known set $K$ and known state-action MDP $M_K$ satisfies 3 conditions (optimism, accuracy and learning complexity) will follow a $4\epsilon$-optimal policy on all but $\mathcal{O}\left(\frac{\zeta(\epsilon, \delta)}{\epsilon(1-\gamma)} \ln \frac{1}{\delta} \ln \frac{1}{\epsilon(1-\gamma)}\right)$ timesteps with probability $1-2\delta$, where $\zeta(\epsilon, \delta)$ is the total number of updates of action-value estimates plus the number of visits to unknowns. This proposition, while it focuses on single-task learners, can be easily adapted to work for multi-task learners, as shown in ~\cite{brunskill2013sample}.

	Now we verify that the required 3 conditions all hold for our algorithm.

	\noindent(1) $Q_t(s,a) \geq Q^*(s,a) - \epsilon$ for any timestep $t$ (optimism).

	This condition naturally holds as the single-task learner RMax chooses actions by optimistic value functions. \ourmod does not change the way of choosing actions. It is similar for using $E^3$ or MBIE as the single-task learner.

	\noindent(2) $V_{t}(s)-V_{M_{K_{t}}}^{\pi_{t}}(s) \leq \epsilon$ for any timestep $t$ (accuracy).

	An s-a pair is in $M_K$ if it is fully known, i.e., $n(s,a) \geq \threslg$. A part of $n(s,a)$ may come from the visits to other s-a pairs with the same \patname. According to Lemma~\ref{lem:simu}, condition (2) holds if the estimation of the \patname is within $\mathcal{O}(\frac{\epsilon (1-\gamma)}{V_{\max}})$ accuracy. By Hoeffding's inequality, to achieve this accuracy, $\threslg=\tilde{\mathcal{O}}(\frac{S V^2_{\max}}{\epsilon^2 (1-\gamma)^2})$ samples are required for a \patname. 

	\noindent(3) The total number of updates of action-value estimates plus the number of visits to unknowns is bounded by $\zeta(\epsilon, \delta)$ (learning complexity).

	Lemma~\ref{lem:visits} already gives the number of visits to unknown s-a pairs, and the updates of action-value estimates will happen no more than $TSA$ times for $T$ tasks. Hence, $\zeta(\epsilon, \delta)=\mathcal{\tilde{O}} \Big(\frac{TSA}{\mingap^2 } + \frac{S G V^2_{\max}}{\epsilon^2 (1-\gamma)^2} \Big)$.

	Therefore, the sample complexity of \ourmod is 
	$$
		\mathcal{\tilde{O}} \Big( \big(\frac{TSA}{\mingap^2 } + \frac{S G V^2_{\max}}{\epsilon^2 (1-\gamma)^2}\big) \left(\frac{V_{\max}}{\epsilon(1-\gamma)} \ln \frac{1}{\delta} \ln \frac{1}{\epsilon(1-\gamma)}\right) \Big)
	$$
\end{proof}

\subsection{Proof of Theorem~\ref{thm:main_2}}

\begin{proof}
(of Theorem~\ref{thm:main_2})

The proof of Theorem~\ref{thm:main_2} is similar to the proof of Theorem~\ref{thm:main}, because \ourmodtwo is adapted from \ourmod. The only difference lies in the number of visits to unknown s-a pairs.

In the first phase, \ourmodtwo is the same with \ourmod, so the number of visits to identify \patnames is $\mathcal{\tilde{O}}(\frac{T_1 SA}{\mingap^2})$. 

In the second phase, \ourmodtwo avoids visiting all s-a pairs for at least $\thressm$ times under the help of finite models. As \cite{brunskill2013sample} shows, we need at most $C^2$ informative s-a pairs to fully identify a model, where an s-a pair is ``informative'' if at least two MDP models have sufficient disagreement in its dynamics. Similarly with Lemma~\ref{lem:horizon}, $\mathcal{\tilde{O}}(\frac{DC^2}{\mingap^2})$ samples are enough to let all these $C^2$ informative s-a pairs roughly known. Then the correct model for the current task would be identified. Thus, for every task in the second phase, $\mathcal{\tilde{O}}(\frac{DC^2}{\mingap^2})$ visits to unknowns are needed.

Finally, for each \patname, its visits are shared among s-a pairs and tasks, no matter which phase they are in. Hence there are still $\mathcal{\tilde{O}}(\frac{SGV^3_{\max}}{\mingap^2 \epsilon (1-\gamma)})$ visits to unknowns.

Adding the above three parts of visits to unknowns, and following the proof of Theorem~\ref{thm:main}, we obtain the sample complexity of Theorem~\ref{thm:main_2}.

\end{proof}

%!TEX root = ../0_aaai2021_multitaskrl_main.tex

\section{Additional Experiment Settings and Results}
\label{app:exp_addition}

\subsection{Setups}
\label{app:exp_setup}

\noindent\textbf{Computing Infrastructure}
All experiments are conducted on a PC equipped with a 3.6 GHz INTEL CPU of 6 cores. 

\noindent\textbf{Hyper-parameter Settings}
In Maze, an agent navigates with actions ``up'', ``down'', ``left'' and ``right''. The reward of the goal state is set to be 1.0, and the step cost is set as 0.2. 

The base learners in FMRL, our \ourmod and our \ourmodtwo are chosen to be RMax (known threshold being 500) without loss of generality.
The threshold $\thressm$ is set to be 50, the number of episodes 3000, and the number of in-episode steps 30.
$\hat{\tau}$ is set to be 0.15 for online MTRL environments, and 0.24 for Finite-Model environments. Model gap $\Gamma$ for FMRL and \ourmodtwo is set to be 0.6.
In Finite-Model MTRL experiments, $T_1$ is set to be 15 for both \ourmodtwo and FMRL.

The results are averaged over 20 runs. The randomization in the multiple runs comes from different task sequences generated across runs, although the comparison in each run is done on the same task sequence. We also provide the generated task sequences in our code to ensure reproducibility.

% \newpage
% \subsection{Addition Experimental Results}
% \textbf{Comparison of Mean Rewards Across Episodes}
% \label{app:exp_addition}

% \input{appendix/fig_mean}

\subsection{Comparison of Per-task Reward with Confidence Intervals}
\label{app:exp_interval}

%!TEX root = ../0_aaai2021_multitaskrl_main.tex
% \vspace{-1.5em}

\begin{figure}[!htbp]
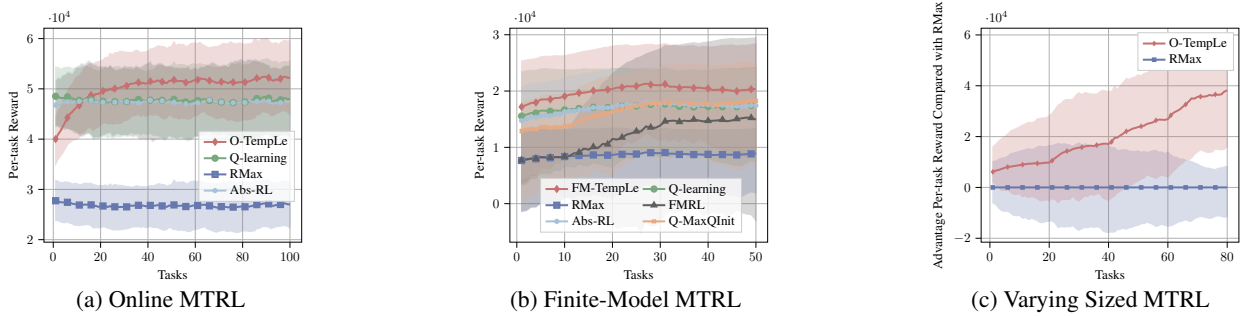

\centering
	\begin{subfigure}[t]{0.3\columnwidth}
		\centering
		\input{\fighome/maze/lifelong-maze_h-4_w-4_mt_runs20_abs_shade.tex}
	\vspace{-0.4em}
	\caption{Online MTRL}
	\label{sfig:maze_interval}
	\end{subfigure}
	\hfill
	\begin{subfigure}[t]{0.3\columnwidth}
		\centering
		\input{\fighome/group/mt_runs30_abs_maxqint_shade.tex}
		\vspace{-0.4em}
		\caption{Finite-Model MTRL}
		\label{sfig:group_interval}
	\end{subfigure}
	\hfill
	\begin{subfigure}[t]{0.3\columnwidth}
		\centering
		\input{\fighome/var/lifelong-varsize_rmax_runs10.tex}
		\vspace{-0.4em}
		\caption{Varying Sized MTRL}
		\label{sfig:var_interval}
	\end{subfigure}
\vspace{-0.5em}
\caption{Performance of \ourmod and \ourmodtwo compared against baselines in (a) Online MTRL, (b) Finite-Model MTRL and (c) varying sized MTRL. }
\label{fig:interval}
\end{figure}
% \vspace{-0.5em}

\subsection{Additional Results of Baseline Methods}
\label{app:exp_baseline}

%!TEX root = ../0_aaai2021_multitaskrl_main.tex
% \vspace{-1.5em}

\begin{figure}[!htbp]
\centering
	\begin{subfigure}[t]{0.3\columnwidth}
		\centering
		\input{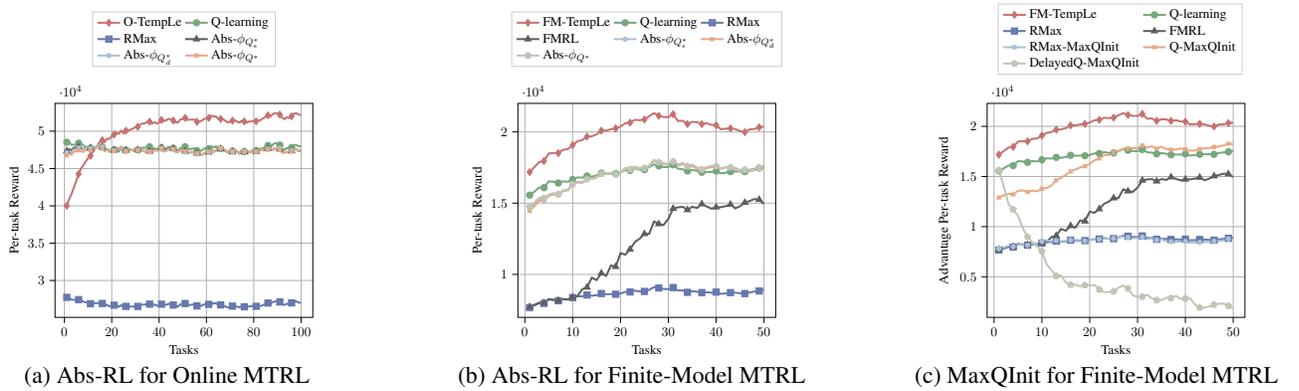}
	\vspace{-0.4em}
	\caption{Abs-RL for Online MTRL}
	\label{sfig:abs_online}
	\end{subfigure}
	\hfill
	\begin{subfigure}[t]{0.3\columnwidth}
		\centering
		% This file was created by tikzplotlib v0.9.0.
\begin{tikzpicture}[scale=0.5]

\definecolor{color0}{rgb}{0.462745098039216,0.654901960784314,0.490196078431373}
\definecolor{color1}{rgb}{0.4,0.470588235294118,0.67843137254902}
\definecolor{color2}{rgb}{0.776470588235294,0.443137254901961,0.443137254901961}
\definecolor{color3}{rgb}{0.662745098039216,0.756862745098039,0.835294117647059}
\definecolor{color4}{rgb}{0.901960784313726,0.662745098039216,0.517647058823529}
\definecolor{color5}{rgb}{0.752941176470588,0.772549019607843,0.713725490196078}

\begin{axis}[
legend cell align={left},
legend columns=3,
legend style={fill opacity=0.8, draw opacity=1, text opacity=1, at={(0.5,1.3)}, anchor=center, draw=white!80!black},
tick align=outside,
tick pos=left,
x grid style={white!69.0196078431373!black},
xlabel={Tasks},
xmajorgrids,
xmin=-1.45, xmax=52.45,
xtick style={color=black},
y grid style={white!69.0196078431373!black},
ylabel={ Per-task Reward},
ymajorgrids,
ymin=6984.06581161696, ymax=22003.4846227104,
ytick style={color=black}
]
\addplot [thick, line width=1.5, color2, mark=diamond*, mark size=2, mark repeat=3, mark options={solid}]
table {%
1 17180.6333333333
2 17531.3133333333
3 17878.9153333333
4 17956.7304666667
5 18486.2007533333
6 18524.3173446667
7 18509.9889435333
8 18751.8533825133
9 18796.7947109287
10 19078.3485731691
11 19345.5003825189
12 19429.890344267
13 19670.7213098403
14 19695.2725121896
15 19826.4952609706
16 20090.3024015402
17 20099.2188280529
18 20184.8436119143
19 20250.5292507228
20 20375.5896589839
21 20607.0940264188
22 20649.4012904436
23 20850.0578280659
24 20821.8920452593
25 20871.6828407334
26 21107.2745566601
27 21320.7837676607
28 21132.008724228
29 21044.0011851385
30 21040.5410666247
31 21246.7069599622
32 20768.7829306326
33 20781.931304236
34 20553.7448404791
35 20705.5003564312
36 20689.1403207881
37 20549.3462887093
38 20607.1049931717
39 20495.0644938545
40 20468.6747111357
41 20184.4505733555
42 20228.6055160199
43 20253.3082977513
44 20197.1041346428
45 20062.6103878452
46 19985.1560157273
47 20176.3737474879
48 20204.0263727392
49 20326.4404021319
50 20375.1496952521
};
\addlegendentry{\ourmodtwo}
\addplot [thick, line width=1.5, color0, mark=*, mark size=2, mark repeat=3, mark options={solid}]
table {%
1 15561.3333333333
2 15806.7533333333
3 16099.1646666667
4 16087.5748666667
5 16526.2640466667
6 16506.4443086667
7 16414.1732111333
8 16535.5692233533
9 16498.5656343513
10 16668.4690709162
11 16810.8788304912
12 16789.3709474421
13 16906.2305193646
14 16860.9408007615
15 16944.1867206853
16 17108.4547152834
17 17074.1559104218
18 17084.9536527129
19 17071.4782874416
20 17154.8971253641
21 17308.2740794944
22 17296.5700048783
23 17440.0063377238
24 17339.7490372847
25 17347.2441335563
26 17534.463053534
27 17723.3400815139
28 17615.1094066959
29 17532.4417993596
30 17536.2542860903
31 17693.5421908146
32 17423.2113050665
33 17381.0435078932
34 17253.4924904372
35 17335.0899080601
36 17270.1775839208
37 17164.133158862
38 17196.6898429758
39 17184.0175253449
40 17249.6391061438
41 17088.941862196
42 17135.7976759764
43 17209.6245750455
44 17215.9787842076
45 17222.5342391202
46 17212.1808152081
47 17342.4527336873
48 17404.5741269853
49 17457.1467142867
50 17591.5820428581
};
\addlegendentry{Q-learning}
\addplot [thick, line width=1.5, color1, mark=square*, mark size=2, mark repeat=3, mark options={solid}]
table {%
1 7666.76666666667
2 7815.18
3 7974.232
4 7962.47546666667
5 8237.71125333333
6 8197.06346133333
7 8149.12044853333
8 8257.18840368
9 8232.83622997867
10 8365.87594031413
11 8461.29834628272
12 8458.72517832111
13 8553.48599382234
14 8501.7073944401
15 8538.39665499609
16 8654.66032282982
17 8608.01762388017
18 8612.20919482549
19 8607.56827534294
20 8652.33811447531
21 8772.66096969445
22 8768.564872725
23 8860.06505211917
24 8803.82854690725
25 8803.60569221653
26 8917.57178966154
27 9102.28127736205
28 9045.18648295918
29 8977.97783466326
30 8959.78338453027
31 9072.48171274391
32 8875.10354146952
33 8875.25318732257
34 8740.85453525698
35 8824.00908173128
36 8777.39817355815
37 8720.38168953567
38 8744.04685391543
39 8712.29550185722
40 8757.07261833817
41 8637.69535650435
42 8658.72915418725
43 8716.76290543519
44 8713.71328155834
45 8682.66528673584
46 8651.99875806225
47 8730.74554892269
48 8791.80099403042
49 8839.10756129405
50 8926.15680516464
};
\addlegendentry{RMax}
\addplot [thick, line width=1.5, white!36.8627450980392!black, mark=triangle*, mark size=2, mark repeat=3, mark options={solid}]
table {%
1 7720.03333333333
2 7867.98
3 8018.45866666667
4 7996.44613333333
5 8268.04485333333
6 8220.08370133333
7 8167.16199786667
8 8270.99246474667
9 8252.68988493867
10 8386.75756311147
11 8481.25514013365
12 8962.16295945362
13 9091.32999684159
14 9771.2836638241
15 9596.48529744169
16 10081.5301010309
17 9896.05375759443
18 10625.0017151683
19 10538.3582103182
20 11493.269055953
21 11385.1421503577
22 11762.1812686553
23 12271.7598084564
24 12507.5738276108
25 12846.3164448497
26 12846.6748003647
27 13712.5039869949
28 13539.1869216288
29 13465.2248961326
30 13892.519073186
31 14616.473832534
32 14706.3964492806
33 14740.7234710192
34 14528.097790584
35 14744.4080115256
36 14659.7705437064
37 14939.4534893357
38 14772.7081404021
39 14611.9606596953
40 14716.4545937257
41 14719.1758010198
42 14789.8548875845
43 14909.4293988261
44 14631.5764589435
45 14787.9854797158
46 15045.7102650775
47 15148.4125719031
48 15289.2579813795
49 15246.0055165749
50 14942.5849649174
};
\addlegendentry{FMRL}
\addplot [thick, line width=1.5, color3, mark=asterisk, mark size=2, mark repeat=3, mark options={solid}]
table {%
1 14621.4666666667
2 14947.46
3 15312.494
4 15235.2812666667
5 15523.3464733333
6 15610.7951593333
7 15578.4389767333
8 15855.76507906
9 15920.1352378207
10 16354.1617140386
11 16474.1655426347
12 16461.8656550379
13 16724.4290895341
14 16697.2261805807
15 16793.3602291893
16 17067.9375396037
17 17077.3637856433
18 17072.5107404123
19 17025.8296663711
20 17120.436699734
21 17370.0663630939
22 17353.2430601179
23 17479.5020874394
24 17387.5252120288
25 17407.4860241593
26 17604.7907550767
27 17972.7283462357
28 17848.2121782788
29 17779.8576271176
30 17721.0918644058
31 17878.4560112986
32 17654.593743502
33 17742.1143691518
34 17573.2029322366
35 17657.2993056797
36 17454.2193751117
37 17309.2641042672
38 17524.2476938405
39 17575.3195911231
40 17603.5742986774
41 17425.5035354764
42 17513.4965152621
43 17539.4835304025
44 17470.8218440289
45 17328.349659626
46 17403.9013603301
47 17319.6845576304
48 17390.8994352007
49 17519.7961583473
50 17472.8932091793
};
\addlegendentry{Abs-$\phi_{Q_\epsilon^*}$}
\addplot [thick, line width=1.5, color4, mark=x, mark size=2, mark repeat=3, mark options={solid}]
table {%
1 14463.1666666667
2 14793.62
3 15183.1846666667
4 15184.3695333333
5 15509.5492466667
6 15631.5009886667
7 15572.2942231333
8 15864.9714674867
9 15940.5676540713
10 16382.2942219975
11 16485.2781331311
12 16452.3536531513
13 16720.3449545029
14 16725.0471257192
15 16840.2557464807
16 17075.7901718326
17 17087.0944879827
18 17056.2483725177
19 17007.0168685993
20 17169.0251817394
21 17459.1759968988
22 17421.8117305422
23 17569.4972241547
24 17460.8041684059
25 17457.1970848986
26 17630.4240430754
27 18019.4849721012
28 17835.2364748911
29 17754.8861607353
30 17691.8542113284
31 17864.7087901956
32 17632.527911176
33 17770.9484533918
34 17507.6569413859
35 17597.1745805807
36 17421.5371225226
37 17312.690076937
38 17457.93773591
39 17522.2572956523
40 17579.1548994204
41 17414.106076145
42 17500.5954685305
43 17485.9725883441
44 17380.3453295097
45 17222.9007965588
46 17310.6673835696
47 17232.2906452126
48 17309.108247358
49 17432.0240892889
50 17444.0583470267
};
\addlegendentry{Abs-$\phi_{Q_d^*}$}
\addplot [thick, line width=1.5, color5, mark=pentagon*, mark size=2, mark repeat=3, mark options={solid}]
table {%
1 14760.2666666667
2 15035.35
3 15386.235
4 15343.8315
5 15602.7316833333
6 15706.1051816667
7 15635.2379968333
8 15896.3975304833
9 15897.947777435
10 16317.5963330248
11 16452.723366389
12 16433.3776964168
13 16678.2532601084
14 16713.1979340976
15 16827.6914740212
16 17063.492326619
17 17079.0464272905
18 17087.7117845614
19 17011.1606061053
20 17137.4245454948
21 17409.9154242786
22 17365.6272151841
23 17515.1144936657
24 17411.2630442991
25 17462.6834065359
26 17626.2150658823
27 18021.5635592941
28 17876.470536698
29 17810.4134830282
30 17748.978801392
31 17926.9575879195
32 17702.2451624609
33 17829.0639795481
34 17599.0275815933
35 17713.574823434
36 17513.6873410906
37 17384.4952736482
38 17529.8157462834
39 17618.244171655
40 17601.1864211562
41 17419.8377790406
42 17531.3973344698
43 17511.3142676895
44 17395.8795075872
45 17263.1582234952
46 17330.0690678123
47 17249.9788276978
48 17337.5676115947
49 17465.4441837685
50 17415.9464320583
};
\addlegendentry{Abs-$\phi_{Q^*}$}
\end{axis}

\end{tikzpicture}
		\vspace{-0.4em}
		\caption{{Abs-RL for Finite-Model MTRL}}
		\label{sfig:abs_fm}
	\end{subfigure}
	\hfill
	\begin{subfigure}[t]{0.3\columnwidth}
		\centering
		% This file was created by tikzplotlib v0.9.0.
\begin{tikzpicture}[scale=0.5]

\definecolor{color0}{rgb}{0.462745098039216,0.654901960784314,0.490196078431373}
\definecolor{color1}{rgb}{0.4,0.470588235294118,0.67843137254902}
\definecolor{color2}{rgb}{0.776470588235294,0.443137254901961,0.443137254901961}
\definecolor{color3}{rgb}{0.662745098039216,0.756862745098039,0.835294117647059}
\definecolor{color4}{rgb}{0.901960784313726,0.662745098039216,0.517647058823529}
\definecolor{color5}{rgb}{0.752941176470588,0.772549019607843,0.713725490196078}

\begin{axis}[
legend cell align={left},
legend columns=2,
legend style={fill opacity=0.8, draw opacity=1, text opacity=1, at={(0.99,1.3)}, anchor=east, draw=white!80!black},
tick align=outside,
tick pos=left,
x grid style={white!69.0196078431373!black},
xlabel={Tasks},
xmajorgrids,
xmin=-1.45, xmax=52.45,
xtick style={color=black},
y grid style={white!69.0196078431373!black},
ylabel={Advantage Per-task Reward},
ymajorgrids,
ymin=957.257769704119, ymax=22290.4754818491,
ytick style={color=black}
]
\addplot [thick, line width=1.5, color2, mark=diamond*, mark size=2, mark repeat=3, mark options={solid}]
table {%
1 17180.6333333333
2 17531.3133333333
3 17878.9153333333
4 17956.7304666667
5 18486.2007533333
6 18524.3173446667
7 18509.9889435333
8 18751.8533825133
9 18796.7947109287
10 19078.3485731691
11 19345.5003825189
12 19429.890344267
13 19670.7213098403
14 19695.2725121896
15 19826.4952609706
16 20090.3024015402
17 20099.2188280529
18 20184.8436119143
19 20250.5292507228
20 20375.5896589839
21 20607.0940264188
22 20649.4012904436
23 20850.0578280659
24 20821.8920452593
25 20871.6828407334
26 21107.2745566601
27 21320.7837676607
28 21132.008724228
29 21044.0011851385
30 21040.5410666247
31 21246.7069599622
32 20768.7829306326
33 20781.931304236
34 20553.7448404791
35 20705.5003564312
36 20689.1403207881
37 20549.3462887093
38 20607.1049931717
39 20495.0644938545
40 20468.6747111357
41 20184.4505733555
42 20228.6055160199
43 20253.3082977513
44 20197.1041346428
45 20062.6103878452
46 19985.1560157273
47 20176.3737474879
48 20204.0263727392
49 20326.4404021319
50 20375.1496952521
};
\addlegendentry{\ourmodtwo}
\addplot [thick, line width=1.5, color0, mark=*, mark size=2, mark repeat=3, mark options={solid}]
table {%
1 15561.3333333333
2 15806.7533333333
3 16099.1646666667
4 16087.5748666667
5 16526.2640466667
6 16506.4443086667
7 16414.1732111333
8 16535.5692233533
9 16498.5656343513
10 16668.4690709162
11 16810.8788304912
12 16789.3709474421
13 16906.2305193646
14 16860.9408007615
15 16944.1867206853
16 17108.4547152834
17 17074.1559104218
18 17084.9536527129
19 17071.4782874416
20 17154.8971253641
21 17308.2740794944
22 17296.5700048783
23 17440.0063377238
24 17339.7490372847
25 17347.2441335563
26 17534.463053534
27 17723.3400815139
28 17615.1094066959
29 17532.4417993596
30 17536.2542860903
31 17693.5421908146
32 17423.2113050665
33 17381.0435078932
34 17253.4924904372
35 17335.0899080601
36 17270.1775839208
37 17164.133158862
38 17196.6898429758
39 17184.0175253449
40 17249.6391061438
41 17088.941862196
42 17135.7976759764
43 17209.6245750455
44 17215.9787842076
45 17222.5342391202
46 17212.1808152081
47 17342.4527336873
48 17404.5741269853
49 17457.1467142867
50 17591.5820428581
};
\addlegendentry{Q-learning}
\addplot [thick, line width=1.5, color1, mark=square*, mark size=2, mark repeat=3, mark options={solid}]
table {%
1 7666.76666666667
2 7815.18
3 7974.232
4 7962.47546666667
5 8237.71125333333
6 8197.06346133333
7 8149.12044853333
8 8257.18840368
9 8232.83622997867
10 8365.87594031413
11 8461.29834628272
12 8458.72517832111
13 8553.48599382234
14 8501.7073944401
15 8538.39665499609
16 8654.66032282982
17 8608.01762388017
18 8612.20919482549
19 8607.56827534294
20 8652.33811447531
21 8772.66096969445
22 8768.564872725
23 8860.06505211917
24 8803.82854690725
25 8803.60569221653
26 8917.57178966154
27 9102.28127736205
28 9045.18648295918
29 8977.97783466326
30 8959.78338453027
31 9072.48171274391
32 8875.10354146952
33 8875.25318732257
34 8740.85453525698
35 8824.00908173128
36 8777.39817355815
37 8720.38168953567
38 8744.04685391543
39 8712.29550185722
40 8757.07261833817
41 8637.69535650435
42 8658.72915418725
43 8716.76290543519
44 8713.71328155834
45 8682.66528673584
46 8651.99875806225
47 8730.74554892269
48 8791.80099403042
49 8839.10756129405
50 8926.15680516464
};
\addlegendentry{RMax}
\addplot [thick, line width=1.5, white!36.8627450980392!black, mark=triangle*, mark size=2, mark repeat=3, mark options={solid}]
table {%
1 7720.03333333333
2 7867.98
3 8018.45866666667
4 7996.44613333333
5 8268.04485333333
6 8220.08370133333
7 8167.16199786667
8 8270.99246474667
9 8252.68988493867
10 8386.75756311147
11 8481.25514013365
12 8962.16295945362
13 9091.32999684159
14 9771.2836638241
15 9596.48529744169
16 10081.5301010309
17 9896.05375759443
18 10625.0017151683
19 10538.3582103182
20 11493.269055953
21 11385.1421503577
22 11762.1812686553
23 12271.7598084564
24 12507.5738276108
25 12846.3164448497
26 12846.6748003647
27 13712.5039869949
28 13539.1869216288
29 13465.2248961326
30 13892.519073186
31 14616.473832534
32 14706.3964492806
33 14740.7234710192
34 14528.097790584
35 14744.4080115256
36 14659.7705437064
37 14939.4534893357
38 14772.7081404021
39 14611.9606596953
40 14716.4545937257
41 14719.1758010198
42 14789.8548875845
43 14909.4293988261
44 14631.5764589435
45 14787.9854797158
46 15045.7102650775
47 15148.4125719031
48 15289.2579813795
49 15246.0055165749
50 14942.5849649174
};
\addlegendentry{FMRL}
\addplot [thick, line width=1.5, color3, mark=asterisk, mark size=2, mark repeat=3, mark options={solid}]
table {%
1 7797.5
2 7916.16666666667
3 8093.47333333333
4 8009.01933333333
5 8194.62073333333
6 8243.30532666667
7 8145.304794
8 8336.04098126667
9 8325.54354980667
10 8523.65252815933
11 8604.65394201007
12 8535.20188114239
13 8616.42502636149
14 8544.58252372534
15 8570.99093801947
16 8660.09184421752
17 8680.2459931291
18 8658.08806048286
19 8579.19592110124
20 8615.33966232445
21 8783.845696092
22 8749.2611264828
23 8842.27834716786
24 8796.44717911774
25 8828.25246120596
26 8902.7705484187
27 9016.51682691016
28 8934.85847755248
29 8847.8292964639
30 8863.17970015084
31 8955.93173013576
32 8754.93855712218
33 8824.5980347433
34 8659.63823126897
35 8759.46440814207
36 8556.0813006612
37 8503.13983726174
38 8496.25252020224
39 8531.19726818201
40 8530.40087469714
41 8389.85078722743
42 8440.49237517135
43 8440.67980432089
44 8479.85515722213
45 8464.81297483325
46 8556.21834401659
47 8579.9231762816
48 8649.02085865344
49 8768.85210612143
50 8760.09689550929
};
\addlegendentry{RMax-MaxQInit}
\addplot [thick, line width=1.5, color4, mark=x, mark size=2, mark repeat=3, mark options={solid}]
table {%
1 12892.4666666667
2 13126.8766666667
3 13279.8123333333
4 13213.1244333333
5 13575.9386566667
6 13614.7381243333
7 13462.4643119
8 13585.75788071
9 13523.7354259723
10 13779.3985500418
11 13844.0053617043
12 14104.0748255338
13 14613.5606763138
14 14867.7879420157
15 15152.5491478142
16 15514.8975663661
17 15705.7678097295
18 15904.4910287565
19 16034.3852592142
20 16257.5133999594
21 16621.9220599635
22 16795.4365206338
23 17130.6995352371
24 17150.2862483801
25 17376.7309568754
26 17543.7178611878
27 17781.9394084024
28 17791.2821342288
29 17818.0172541393
30 17818.952195392
31 18102.2236425195
32 17895.5479449342
33 18021.2798171074
34 17775.2685020634
35 17978.5149851904
36 17743.3568200047
37 17670.0778046709
38 17682.7233575371
39 17783.4410217834
40 17720.6635862717
41 17570.1105609779
42 17583.1861715468
43 17738.9675543921
44 17796.7874656195
45 17844.3220523909
46 17949.4631804852
47 17981.47019577
48 18111.7298428597
49 18261.3835252404
50 18232.7651727163
};
\addlegendentry{Q-MaxQInit}
\addplot [thick, line width=1.5, color5, mark=pentagon*, mark size=2, mark repeat=3, mark options={solid}]
table {%
1 15601
2 14082.3333333333
3 12279.44
4 11707.7226666667
5 11091.2604
6 9867.66102666666
7 8974.764924
8 8449.1484316
9 8036.04358844
10 7543.689229596
11 6252.8503066364
12 5663.46860930609
13 5080.72508170882
14 5135.9992402046
15 4432.18264951748
16 4205.90771789906
17 4208.57694610915
18 4058.08591816491
19 4159.44732634842
20 4144.14592704691
21 4187.51133434222
22 3730.88686757466
23 3482.1681808172
24 3466.03802940214
25 3532.89755979526
26 3915.25113714907
27 4136.70935676749
28 3837.65842109075
29 3173.055912315
30 2984.4203210835
31 3015.14828897515
32 3187.43679341097
33 2864.14644740321
34 2673.55180266289
35 2691.77328906327
36 2905.04596015694
37 2873.28803080791
38 3111.80922772712
39 2797.78497162107
40 2828.0998077923
41 2847.58982701307
42 2252.04751097843
43 1933.46609321392
44 1926.94948389253
45 2029.67120216994
46 2217.35408195295
47 2310.09200709099
48 2326.48613971522
49 2100.38085907703
50 1928.39277316933
};
\addlegendentry{DelayedQ-MaxQInit}
\end{axis}

\end{tikzpicture}
		\vspace{-0.4em}
		\caption{{MaxQInit for Finite-Model MTRL}}
		\label{sfig:maxqinit_fm}
	\end{subfigure}
\vspace{-0.5em}
\caption{Additional experimental results for other versions of Abs-RL and MaxQInit. }
\label{fig:baseline}
\end{figure}
% \vspace{-0.5em}

\subsection{Results of Varying Action Size}
\label{app:action}
Our proposed method can also work when the action space of the tasks are different, which could happen in transfer reinforcement learning (TRL) settings. For example, in a navigating task, the available actions can be simply ``up'', ``down'', ``left'' and ``right'' (as shown in Figure~\ref{sfig:4_action}), but a more difficult task may also allow the actions ``up-left'', ``up-right'', ``down-left'', ``down-right'' (as shown in Figure~\ref{sfig:8_action}).
Intuitively, there is some shared knowledge between these two tasks, and the agent will learn the 8-action task better if it can transfer some knowledge from the 4-action task. However, few existing methods can transfer appropriate knowledge between these two tasks without pre-defined inter-task mappings.
In contrast, our proposed \modname is able to transfer knowledge between these two tasks without any prior knowledge. 

In Figure~\ref{sfig:action_result}, we show the performance of Q-learning, RMax and \modname on the 8-action gridworld, where \modname has already learned from a 4-action task and gathered the \patname information. The results suggest that \modname automatically figures out the relation among the state-action pairs in two tasks with different action spaces.

\begin{figure}[!htbp]
\centering
	\begin{subfigure}[t]{0.32\columnwidth}
		\centering
		\includegraphics[width=0.6\textwidth, bb=0 0 400 400]{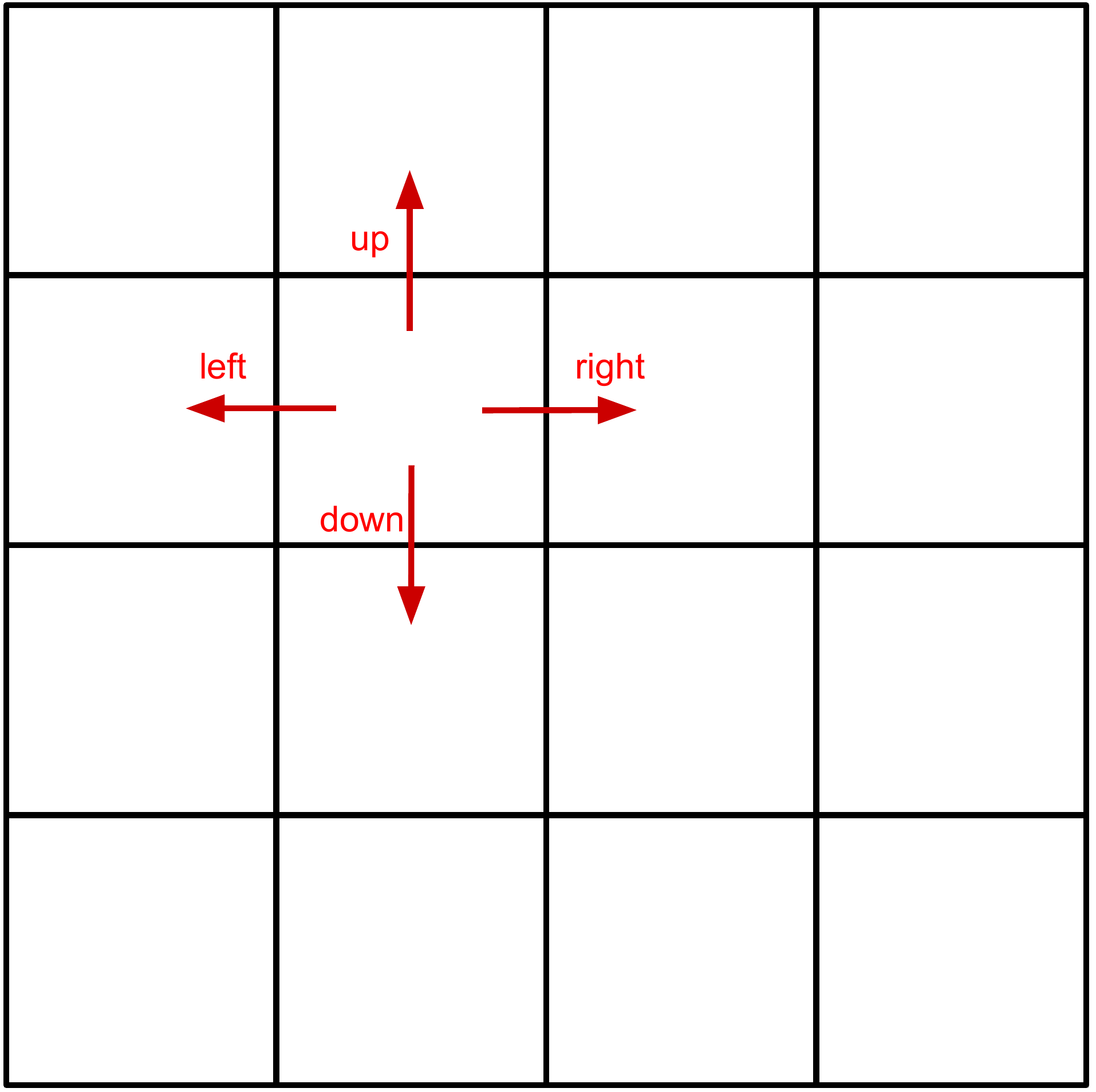}
		\caption{Source task}
		\label{sfig:4_action}
	\end{subfigure}
	% \hfill
	\begin{subfigure}[t]{0.32\columnwidth}
		\centering
		\includegraphics[width=0.6\textwidth, bb=0 0 400 400]{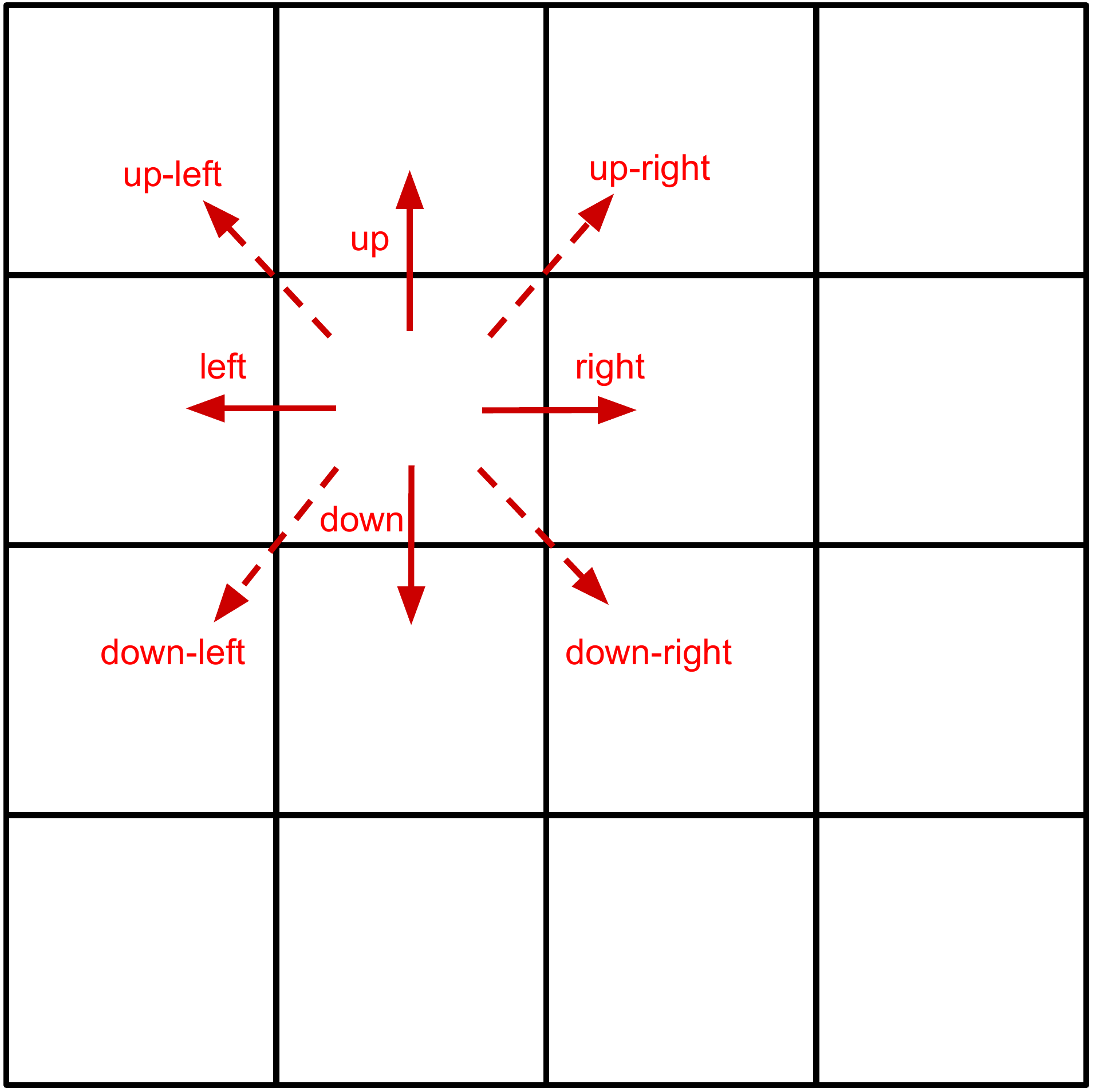}
		\caption{Target task}
		\label{sfig:8_action}
	\end{subfigure}
	% \hfill
	\begin{subfigure}[t]{0.32\columnwidth}
		\centering
		\includegraphics[width=0.9\textwidth, bb=0 0 400 400]{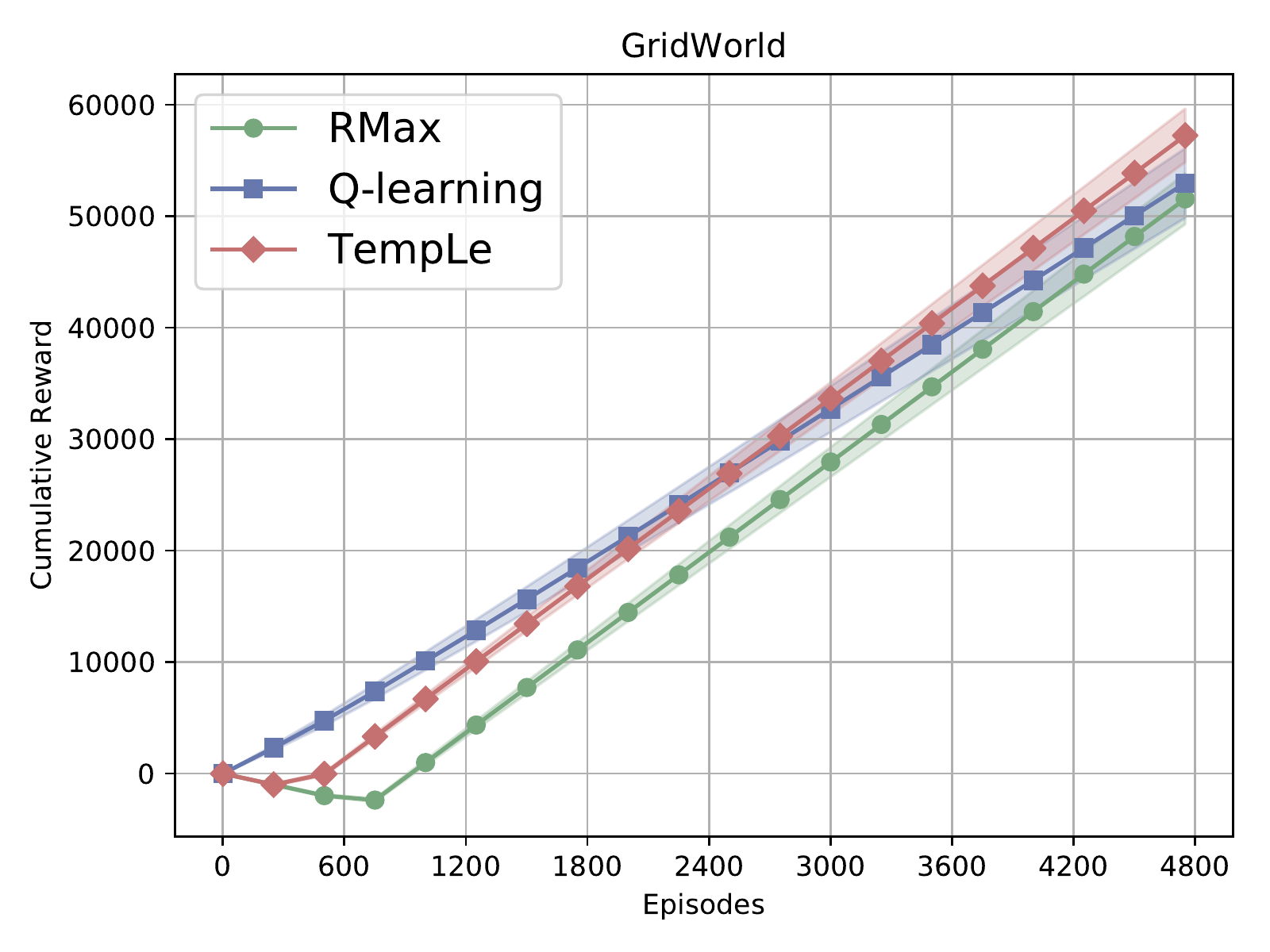}
		\caption{CartPole}
		\label{sfig:action_result}
	\end{subfigure}
  \caption{{Transfer learning with varying action size.}}
  \label{fig:action}
\end{figure}

\subsection{Universal Applicability of \modname}
\label{app:universal}

We also observe that our proposed template learning is universally applicable to many classical stochastic environments. For example, all gridworld-based environments like FourRoom, Taxi, FrozenLake~\cite{1606.01540}, etc. 
In addition, \citet{strehl2004exploration} propose 3 challenging MDPs as Figure~\ref{fig:mdps} shows. It can be seen from these MDP definition that the number of templates is smaller than the number of state-action pairs in all of them. For instance, the \patname $((1,0,\cdots,0),0)$ appears for multiple times in all of the three tasks. Thus, in each of the environments, \modname can transfer knowledge from known state-action pairs to unknown state-action pairs with the same \patname. More interestingly, since these tasks have some common \patnames, if we sequentially learn these three tasks, \modname has the potential to transfer knowledge among them, in spite that they are totally different environments in common sense.

\begin{figure}[!htbp]
    \centering
    \includegraphics[width=0.5\textwidth]{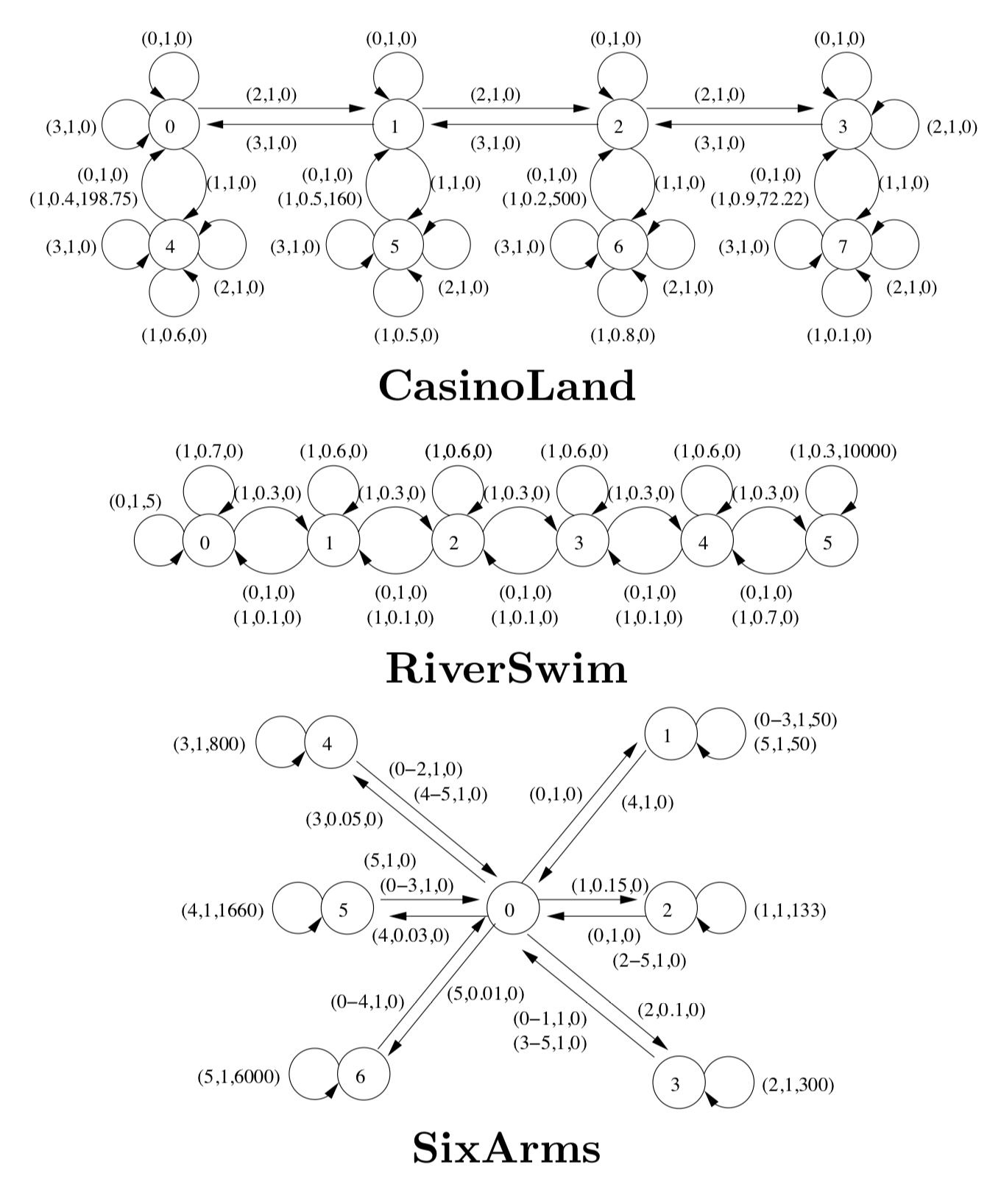}
    \caption{\textbf{Three challenging MDPs~\cite{strehl2004exploration}}: CasinoLand (top), RiverSwim (middle), and SixArms (bottom). Each node in the graph is a state and each edge is labeled by one or more transitions. A transition is of the form $(a,p,r)$ where $a$ is an action, $p$ the probability that action will result in the corresponding transition, and $r$ is the reward for taking the transition.}
    \label{fig:mdps}
\end{figure}

Figure~\ref{sfig:riverswim} Figure~\ref{sfig:fourroom}, and Figure~\ref{sfig:grid_large} respectively show the performance of \modname compared with baselines on RiverSwim\citep{strehl2004exploration}, FourRoom and a large GridWorld, which are well-known hard-to-explore environments. \modname outperforms the single-task learners, because it can aggregate similar information in the environment, which saves samples and facilitates exploration.
%!TEX root = 0_aaai2021_multitaskrl_main.tex
\begin{figure}[!htbp]
\centering
	\begin{subfigure}[t]{0.3\columnwidth}
		\centering
		% This file was created by matplotlib2tikz v0.7.4.
\begin{tikzpicture}[scale=0.45]

\definecolor{color0}{rgb}{0.462745098039216,0.654901960784314,0.490196078431373}
\definecolor{color1}{rgb}{0.4,0.470588235294118,0.67843137254902}
\definecolor{color2}{rgb}{0.776470588235294,0.443137254901961,0.443137254901961}

\begin{axis}[
legend cell align={left},
legend style={at={(0.03,0.97)}, anchor=north west, draw=white!80.0!black},
tick align=outside,
tick pos=left,
x grid style={white!69.01960784313725!black},
xlabel={Episodes},
xmajorgrids,
xmin=-474, xmax=9976,
xtick style={color=black},
y grid style={white!69.01960784313725!black},
ylabel={Cumulative Reward},
ymajorgrids,
ymin=-46587.1958867227, ymax=405145.886269202,
ytick style={color=black}
]
\path [draw=color0, fill=color0, opacity=0.25]
(axis cs:1,-31.8149264108718)
--(axis cs:1,5.21492641087188)
--(axis cs:501,15537.1629507665)
--(axis cs:1001,33367.4315157042)
--(axis cs:1501,51222.0394902545)
--(axis cs:2001,69165.3584080663)
--(axis cs:2501,86955.6680108636)
--(axis cs:3001,104967.583394638)
--(axis cs:3501,122851.772646339)
--(axis cs:4001,140749.751756775)
--(axis cs:4501,158803.082067004)
--(axis cs:5001,176743.951417678)
--(axis cs:5501,194592.821332057)
--(axis cs:6001,212728.331755587)
--(axis cs:6501,230639.985791248)
--(axis cs:7001,248539.141165673)
--(axis cs:7501,266357.450376345)
--(axis cs:8001,284226.259619313)
--(axis cs:8501,302124.954747923)
--(axis cs:9001,319931.991403503)
--(axis cs:9501,337724.064962105)
--(axis cs:9501,261561.135037895)
--(axis cs:9501,261561.135037895)
--(axis cs:9001,247508.608596497)
--(axis cs:8501,233532.845252078)
--(axis cs:8001,219382.740380687)
--(axis cs:7501,205349.549623655)
--(axis cs:7001,191286.058834327)
--(axis cs:6501,177109.614208752)
--(axis cs:6001,162921.868244413)
--(axis cs:5501,148599.578667943)
--(axis cs:5001,134661.048582322)
--(axis cs:4501,120468.317932996)
--(axis cs:4001,106311.648243225)
--(axis cs:3501,92208.8273536615)
--(axis cs:3001,78166.0166053622)
--(axis cs:2501,63954.5319891364)
--(axis cs:2001,49908.8415919337)
--(axis cs:1501,35752.3605097455)
--(axis cs:1001,21675.7684842958)
--(axis cs:501,7764.63704923354)
--(axis cs:1,-31.8149264108718)
--cycle;

\path [draw=color1, fill=color1, opacity=0.25]
(axis cs:1,-38.9939457008809)
--(axis cs:1,-32.8060542991189)
--(axis cs:501,-8654.89559009478)
--(axis cs:1001,-17915.1350695443)
--(axis cs:1501,-19336.1260294557)
--(axis cs:2001,-1898.79706850539)
--(axis cs:2501,18458.0638051211)
--(axis cs:3001,38868.2898279432)
--(axis cs:3501,59191.0799497201)
--(axis cs:4001,79720.3841410678)
--(axis cs:4501,100484.919535425)
--(axis cs:5001,121472.476059731)
--(axis cs:5501,142675.139594516)
--(axis cs:6001,163760.487342283)
--(axis cs:6501,185039.618211063)
--(axis cs:7001,206278.927251616)
--(axis cs:7501,227453.747404009)
--(axis cs:8001,248665.270051206)
--(axis cs:8501,269852.307022963)
--(axis cs:9001,291234.429995367)
--(axis cs:9501,312480.632044101)
--(axis cs:9501,252741.367955899)
--(axis cs:9501,252741.367955899)
--(axis cs:9001,235026.970004633)
--(axis cs:8501,217278.692977037)
--(axis cs:8001,199538.529948794)
--(axis cs:7501,181743.852595991)
--(axis cs:7001,164032.872748384)
--(axis cs:6501,146286.781788937)
--(axis cs:6001,128427.312657717)
--(axis cs:5501,110808.060405484)
--(axis cs:5001,93168.7239402686)
--(axis cs:4501,75425.6804645752)
--(axis cs:4001,57808.6158589322)
--(axis cs:3501,40179.5200502799)
--(axis cs:3001,22689.9101720568)
--(axis cs:2501,5224.73619487888)
--(axis cs:2001,-12334.6029314946)
--(axis cs:1501,-26053.8739705443)
--(axis cs:1001,-19104.6649304557)
--(axis cs:501,-9535.30440990521)
--(axis cs:1,-38.9939457008809)
--cycle;

\path [draw=color2, fill=color2, opacity=0.25]
(axis cs:1,-39.0243278598338)
--(axis cs:1,-34.375672140166)
--(axis cs:501,1706.5671452926)
--(axis cs:1001,22824.1139519044)
--(axis cs:1501,44012.4374878759)
--(axis cs:2001,65208.6883804155)
--(axis cs:2501,86337.8790106129)
--(axis cs:3001,107769.798518332)
--(axis cs:3501,129083.698375664)
--(axis cs:4001,150401.473600917)
--(axis cs:4501,171730.497052762)
--(axis cs:5001,193037.472386976)
--(axis cs:5501,214418.369839928)
--(axis cs:6001,235725.378056783)
--(axis cs:6501,257081.587573476)
--(axis cs:7001,278365.333432325)
--(axis cs:7501,299601.387053136)
--(axis cs:8001,320910.08188869)
--(axis cs:8501,342164.54006817)
--(axis cs:9001,363351.501344749)
--(axis cs:9501,384612.564353024)
--(axis cs:9501,295173.035646976)
--(axis cs:9501,295173.035646976)
--(axis cs:9001,278753.898655251)
--(axis cs:8501,262416.65993183)
--(axis cs:8001,246146.31811131)
--(axis cs:7501,229971.412946864)
--(axis cs:7001,213362.866567675)
--(axis cs:6501,196995.812426524)
--(axis cs:6001,180554.021943217)
--(axis cs:5501,164263.430160072)
--(axis cs:5001,147814.727613024)
--(axis cs:4501,131455.502947238)
--(axis cs:4001,114881.326399083)
--(axis cs:3501,98377.3016243357)
--(axis cs:3001,81913.8014816678)
--(axis cs:2501,65225.9209893871)
--(axis cs:2001,48850.1116195845)
--(axis cs:1501,32349.1625121241)
--(axis cs:1001,15910.8860480956)
--(axis cs:501,-636.367145292602)
--(axis cs:1,-39.0243278598338)
--cycle;

\addplot [semithick, color2, mark=diamond*, mark size=3, mark options={solid}]
table {%
1 -36.6999999999999
501 535.1
1001 19367.5
1501 38180.8
2001 57029.4
2501 75781.9
3001 94841.8
3501 113730.5
4001 132641.4
4501 151593
5001 170426.1
5501 189340.9
6001 208139.7
6501 227038.7
7001 245864.1
7501 264786.4
8001 283528.2
8501 302290.6
9001 321052.7
9501 339892.8
};
\addlegendentry{O-TempLe}
\addplot [semithick, color0, mark=*, mark size=3, mark options={solid}]
table {%
1 -13.3
501 11650.9
1001 27521.6
1501 43487.2
2001 59537.1
2501 75455.1
3001 91566.8
3501 107530.3
4001 123530.7
4501 139635.7
5001 155702.5
5501 171596.2
6001 187825.1
6501 203874.8
7001 219912.6
7501 235853.5
8001 251804.5
8501 267828.9
9001 283720.3
9501 299642.6
};
\addlegendentry{Q-learning}
\addplot [semithick, color1, mark=square*, mark size=3, mark options={solid}]
table {%
1 -35.8999999999999
501 -9095.1
1001 -18509.9
1501 -22695
2001 -7116.7
2501 11841.4
3001 30779.1
3501 49685.3
4001 68764.5
4501 87955.3
5001 107320.6
5501 126741.6
6001 146093.9
6501 165663.2
7001 185155.9
7501 204598.8
8001 224101.9
8501 243565.5
9001 263130.7
9501 282611
};
\addlegendentry{RMax}
\end{axis}

\end{tikzpicture}
	\vspace{-0.4em}
	\caption{{$9 \times 9$ FourRoom}}
	\label{sfig:fourroom}
	\end{subfigure}
	\hfill
	\begin{subfigure}[t]{0.3\columnwidth}
		\centering
		% This file was created by matplotlib2tikz v0.7.4.
\begin{tikzpicture}[scale=0.45]

\definecolor{color0}{rgb}{0.462745098039216,0.654901960784314,0.490196078431373}
\definecolor{color1}{rgb}{0.4,0.470588235294118,0.67843137254902}
\definecolor{color2}{rgb}{0.776470588235294,0.443137254901961,0.443137254901961}

\begin{axis}[
legend cell align={left},
legend style={draw=white!80.0!black},
tick align=outside,
tick pos=left,
x grid style={white!69.01960784313725!black},
xlabel={Episodes},
xmajorgrids,
xmin=-236.5, xmax=4988.5,
xtick style={color=black},
y grid style={white!69.01960784313725!black},
ylabel={Cumulative Reward},
ymajorgrids,
ymin=-516333.960890564, ymax=10845050.3099686,
ytick style={color=black},
ytick={-2000000,0,2000000,4000000,6000000,8000000,10000000,12000000},
yticklabels={−0.2,0.0,0.2,0.4,0.6,0.8,1.0,1.2}
]
\path [draw=color0, fill=color0, opacity=0.25]
(axis cs:1,219.348029747877)
--(axis cs:1,239.651970252123)
--(axis cs:251,33571.3324554163)
--(axis cs:501,66104.5985236624)
--(axis cs:751,98736.8339258202)
--(axis cs:1001,131415.650300552)
--(axis cs:1251,164050.937896097)
--(axis cs:1501,196484.076916756)
--(axis cs:1751,229242.589755567)
--(axis cs:2001,261867.395932946)
--(axis cs:2251,294391.872206033)
--(axis cs:2501,326871.856429449)
--(axis cs:2751,359616.642551618)
--(axis cs:3001,392242.81412828)
--(axis cs:3251,425014.47134757)
--(axis cs:3501,457703.392527903)
--(axis cs:3751,490174.989106625)
--(axis cs:4001,522715.369543481)
--(axis cs:4251,555414.392164317)
--(axis cs:4501,588059.660866902)
--(axis cs:4751,620585.638754632)
--(axis cs:4751,619614.361245368)
--(axis cs:4751,619614.361245368)
--(axis cs:4501,587117.339133098)
--(axis cs:4251,554508.607835683)
--(axis cs:4001,521807.630456519)
--(axis cs:3751,489257.010893375)
--(axis cs:3501,456773.607472097)
--(axis cs:3251,424100.52865243)
--(axis cs:3001,391343.18587172)
--(axis cs:2751,358702.357448382)
--(axis cs:2501,325942.143570551)
--(axis cs:2251,293469.127793967)
--(axis cs:2001,260930.604067054)
--(axis cs:1751,228309.410244433)
--(axis cs:1501,195547.923083244)
--(axis cs:1251,163114.062103903)
--(axis cs:1001,130492.349699448)
--(axis cs:751,97799.1660741798)
--(axis cs:501,65183.4014763376)
--(axis cs:251,32661.6675445837)
--(axis cs:1,219.348029747877)
--cycle;

\path [draw=color1, fill=color1, opacity=0.25]
(axis cs:1,92.5968757625672)
--(axis cs:1,105.403124237433)
--(axis cs:251,78726.4579555221)
--(axis cs:501,202136.088357457)
--(axis cs:751,330113.446200504)
--(axis cs:1001,485192.730534055)
--(axis cs:1251,650939.49244035)
--(axis cs:1501,842075.649769909)
--(axis cs:1751,1029144.10629563)
--(axis cs:2001,1294389.51339123)
--(axis cs:2251,1484112.7687525)
--(axis cs:2501,1650375.23380225)
--(axis cs:2751,1872341.13775946)
--(axis cs:3001,2105020.89560822)
--(axis cs:3251,2316371.47392373)
--(axis cs:3501,2512323.79296576)
--(axis cs:3751,2731038.06232075)
--(axis cs:4001,2943413.103648)
--(axis cs:4251,3146005.88185111)
--(axis cs:4501,3324622.28427469)
--(axis cs:4751,3537502.17553332)
--(axis cs:4751,3256583.82446668)
--(axis cs:4751,3256583.82446668)
--(axis cs:4501,3059280.71572531)
--(axis cs:4251,2895255.11814889)
--(axis cs:4001,2698935.896352)
--(axis cs:3751,2478580.93767925)
--(axis cs:3501,2260386.20703424)
--(axis cs:3251,2086979.52607627)
--(axis cs:3001,1874167.10439178)
--(axis cs:2751,1651662.86224054)
--(axis cs:2501,1463543.76619775)
--(axis cs:2251,1290235.2312475)
--(axis cs:2001,1107695.48660877)
--(axis cs:1751,857899.893704371)
--(axis cs:1501,668405.350230091)
--(axis cs:1251,508902.50755965)
--(axis cs:1001,385695.269465945)
--(axis cs:751,270555.553799496)
--(axis cs:501,174271.911642543)
--(axis cs:251,70064.5420444779)
--(axis cs:1,92.5968757625672)
--cycle;

\path [draw=color2, fill=color2, opacity=0.25]
(axis cs:1,92.9750621894396)
--(axis cs:1,103.02493781056)
--(axis cs:251,247519.706386774)
--(axis cs:501,833475.910267647)
--(axis cs:751,1327405.28951826)
--(axis cs:1001,1878587.63322878)
--(axis cs:1251,2437895.43808682)
--(axis cs:1501,2971017.55454781)
--(axis cs:1751,3535600.46306109)
--(axis cs:2001,3999160.87772556)
--(axis cs:2251,4555601.70816467)
--(axis cs:2501,5144742.07816933)
--(axis cs:2751,5723263.205554)
--(axis cs:3001,6275273.08518769)
--(axis cs:3251,6823777.49957867)
--(axis cs:3501,7349563.81756903)
--(axis cs:3751,7940427.39649118)
--(axis cs:4001,8556644.25236906)
--(axis cs:4251,9171582.1828163)
--(axis cs:4501,9761800.77649141)
--(axis cs:4751,10328623.7522023)
--(axis cs:4751,4001264.24779771)
--(axis cs:4751,4001264.24779771)
--(axis cs:4501,3793348.22350859)
--(axis cs:4251,3547736.8171837)
--(axis cs:4001,3323907.74763094)
--(axis cs:3751,3077508.60350882)
--(axis cs:3501,2862990.18243097)
--(axis cs:3251,2661398.50042133)
--(axis cs:3001,2447371.91481231)
--(axis cs:2751,2222149.794446)
--(axis cs:2501,2005591.92183067)
--(axis cs:2251,1794825.29183533)
--(axis cs:2001,1568477.12227444)
--(axis cs:1751,1367996.53693891)
--(axis cs:1501,1139046.44545219)
--(axis cs:1251,901966.561913183)
--(axis cs:1001,715644.366771216)
--(axis cs:751,462665.710481742)
--(axis cs:501,275261.089732353)
--(axis cs:251,86433.2936132258)
--(axis cs:1,92.9750621894396)
--cycle;

\addplot [semithick, color2, mark=diamond*, mark size=3, mark options={solid}]
table {%
1 98
251 166976.5
501 554368.5
751 895035.5
1001 1297116
1251 1669931
1501 2055032
1751 2451798.5
2001 2783819
2251 3175213.5
2501 3575167
2751 3972706.5
3001 4361322.5
3251 4742588
3501 5106277
3751 5508968
4001 5940276
4251 6359659.5
4501 6777574.5
4751 7164944
};
\addlegendentry{O-TempLe}
\addplot [semithick, color0, mark=*, mark size=3, mark options={solid}]
table {%
1 229.5
251 33116.5
501 65644
751 98268
1001 130954
1251 163582.5
1501 196016
1751 228776
2001 261399
2251 293930.5
2501 326407
2751 359159.5
3001 391793
3251 424557.5
3501 457238.5
3751 489716
4001 522261.5
4251 554961.5
4501 587588.5
4751 620100
};
\addlegendentry{Q-learning}
\addplot [semithick, color1, mark=square*, mark size=3, mark options={solid}]
table {%
1 99
251 74395.5
501 188204
751 300334.5
1001 435444
1251 579921
1501 755240.5
1751 943522
2001 1201042.5
2251 1387174
2501 1556959.5
2751 1762002
3001 1989594
3251 2201675.5
3501 2386355
3751 2604809.5
4001 2821174.5
4251 3020630.5
4501 3191951.5
4751 3397043
};
\addlegendentry{RMax}
\end{axis}

\end{tikzpicture}
		\vspace{-0.4em}
		\caption{{RiverSwim}}
		\label{sfig:riverswim}
	\end{subfigure}
	\hfill
	\begin{subfigure}[t]{0.3\columnwidth}
		\centering
		% This file was created by matplotlib2tikz v0.7.4.
\begin{tikzpicture}[scale=0.45]

\definecolor{color0}{rgb}{0.462745098039216,0.654901960784314,0.490196078431373}
\definecolor{color1}{rgb}{0.4,0.470588235294118,0.67843137254902}
\definecolor{color2}{rgb}{0.776470588235294,0.443137254901961,0.443137254901961}

\begin{axis}[
legend cell align={left},
legend style={at={(0.03,0.97)}, anchor=north west, draw=white!80.0!black},
tick align=outside,
tick pos=left,
x grid style={white!69.01960784313725!black},
xlabel={Episodes},
xmajorgrids,
xmin=-1424, xmax=29926,
xtick style={color=black},
y grid style={white!69.01960784313725!black},
ylabel={Cumulative Reward},
ymajorgrids,
ymin=-9006.86771278993, ymax=189143.778531927,
ytick style={color=black}
]
\path [draw=color0, fill=color0, opacity=0.25]
(axis cs:1,0.258380151290434)
--(axis cs:1,1.74161984870957)
--(axis cs:1501,6894.11942265)
--(axis cs:3001,15087.9917768833)
--(axis cs:4501,23499.3050166566)
--(axis cs:6001,32020.3260900473)
--(axis cs:7501,40602.4300044688)
--(axis cs:9001,49111.9544963917)
--(axis cs:10501,57706.540407758)
--(axis cs:12001,66384.2862669988)
--(axis cs:13501,75045.3222534294)
--(axis cs:15001,83634.9842939562)
--(axis cs:16501,92327.5488120668)
--(axis cs:18001,100943.003254672)
--(axis cs:19501,109618.151374651)
--(axis cs:21001,118250.469185849)
--(axis cs:22501,126860.494559434)
--(axis cs:24001,135451.120442754)
--(axis cs:25501,144198.499608353)
--(axis cs:27001,152823.804922854)
--(axis cs:28501,161453.063926228)
--(axis cs:28501,132585.936073772)
--(axis cs:28501,132585.936073772)
--(axis cs:27001,125414.395077146)
--(axis cs:25501,118282.300391647)
--(axis cs:24001,111014.079557246)
--(axis cs:22501,103832.105440566)
--(axis cs:21001,96656.9308141509)
--(axis cs:19501,89396.6486253493)
--(axis cs:18001,82227.5967453279)
--(axis cs:16501,75053.8511879332)
--(axis cs:15001,67822.6157060438)
--(axis cs:13501,60668.0777465706)
--(axis cs:12001,53482.7137330012)
--(axis cs:10501,46326.259592242)
--(axis cs:9001,39226.2455036083)
--(axis cs:7501,32188.5699955312)
--(axis cs:6001,25101.4739099527)
--(axis cs:4501,18037.4949833434)
--(axis cs:3001,11149.4082231167)
--(axis cs:1501,4659.08057735)
--(axis cs:1,0.258380151290434)
--cycle;

\path [draw=color1, fill=color1, opacity=0.25]
(axis cs:1,0.42830094339717)
--(axis cs:1,1.37169905660283)
--(axis cs:1501,666.938657234419)
--(axis cs:3001,1574.44108757193)
--(axis cs:4501,2256.33670427128)
--(axis cs:6001,3088.95657860084)
--(axis cs:7501,3856.02306449744)
--(axis cs:9001,4746.86855705446)
--(axis cs:10501,9491.24026368602)
--(axis cs:12001,18327.2945699097)
--(axis cs:13501,28183.7650317418)
--(axis cs:15001,38137.7505771038)
--(axis cs:16501,48055.3178723589)
--(axis cs:18001,58046.6258392044)
--(axis cs:19501,68002.2669977184)
--(axis cs:21001,77968.3020643998)
--(axis cs:22501,87949.950013723)
--(axis cs:24001,97936.3235787175)
--(axis cs:25501,107893.069423374)
--(axis cs:27001,117894.617977435)
--(axis cs:28501,127950.321542526)
--(axis cs:28501,105163.278457474)
--(axis cs:28501,105163.278457474)
--(axis cs:27001,96505.9820225645)
--(axis cs:25501,87888.9305766257)
--(axis cs:24001,79285.4764212825)
--(axis cs:22501,70643.249986277)
--(axis cs:21001,62035.0979356001)
--(axis cs:19501,53367.9330022815)
--(axis cs:18001,44704.3741607956)
--(axis cs:16501,35988.4821276411)
--(axis cs:15001,27394.4494228962)
--(axis cs:13501,19842.2349682582)
--(axis cs:12001,12241.3054300903)
--(axis cs:10501,6182.75973631398)
--(axis cs:9001,3679.33144294554)
--(axis cs:7501,3008.37693550256)
--(axis cs:6001,2204.64342139916)
--(axis cs:4501,1321.46329572872)
--(axis cs:3001,749.758912428072)
--(axis cs:1501,324.661342765581)
--(axis cs:1,0.42830094339717)
--cycle;

\path [draw=color2, fill=color2, opacity=0.25]
(axis cs:1,-0.0201562118716424)
--(axis cs:1,0.620156211871642)
--(axis cs:1501,1838.51269421161)
--(axis cs:3001,9784.33474701141)
--(axis cs:4501,19806.9569129834)
--(axis cs:6001,29789.0398905782)
--(axis cs:7501,39811.7930113417)
--(axis cs:9001,49754.4780354799)
--(axis cs:10501,59755.816898328)
--(axis cs:12001,69778.8513868225)
--(axis cs:13501,79877.1656735059)
--(axis cs:15001,89936.3802875801)
--(axis cs:16501,100008.313149313)
--(axis cs:18001,110076.415538881)
--(axis cs:19501,120099.092881161)
--(axis cs:21001,130098.206112798)
--(axis cs:22501,140106.20352604)
--(axis cs:24001,150166.219423131)
--(axis cs:25501,160148.941796229)
--(axis cs:27001,170151.002801272)
--(axis cs:28501,180136.930975349)
--(axis cs:28501,152889.669024651)
--(axis cs:28501,152889.669024651)
--(axis cs:27001,144376.997198728)
--(axis cs:25501,135785.458203771)
--(axis cs:24001,127210.580576869)
--(axis cs:22501,118616.79647396)
--(axis cs:21001,110090.993887202)
--(axis cs:19501,101544.107118839)
--(axis cs:18001,92953.384461119)
--(axis cs:16501,84334.2868506873)
--(axis cs:15001,75750.4197124199)
--(axis cs:13501,67158.2343264941)
--(axis cs:12001,58553.5486131775)
--(axis cs:10501,50015.783101672)
--(axis cs:9001,41447.9219645201)
--(axis cs:7501,32912.4069886583)
--(axis cs:6001,24364.1601094218)
--(axis cs:4501,15827.8430870166)
--(axis cs:3001,7222.06525298859)
--(axis cs:1501,1337.28730578839)
--(axis cs:1,-0.0201562118716424)
--cycle;

\addplot [semithick, color2, mark=diamond*, mark size=3, mark options={solid}]
table {%
1 0.3
1501 1587.9
3001 8503.2
4501 17817.4
6001 27076.6
7501 36362.1
9001 45601.2
10501 54885.8
12001 64166.2
13501 73517.7
15001 82843.4
16501 92171.3
18001 101514.9
19501 110821.6
21001 120094.6
22501 129361.5
24001 138688.4
25501 147967.2
27001 157264
28501 166513.3
};
\addlegendentry{O-TempLe}
\addplot [semithick, color0, mark=*, mark size=3, mark options={solid}]
table {%
1 1
1501 5776.6
3001 13118.7
4501 20768.4
6001 28560.9
7501 36395.5
9001 44169.1
10501 52016.4
12001 59933.5
13501 67856.7
15001 75728.8
16501 83690.7
18001 91585.3
19501 99507.4
21001 107453.7
22501 115346.3
24001 123232.6
25501 131240.4
27001 139119.1
28501 147019.5
};
\addlegendentry{Q-learning}
\addplot [semithick, color1, mark=square*, mark size=3, mark options={solid}]
table {%
1 0.9
1501 495.8
3001 1162.1
4501 1788.9
6001 2646.8
7501 3432.2
9001 4213.1
10501 7837
12001 15284.3
13501 24013
15001 32766.1
16501 42021.9
18001 51375.5
19501 60685.1
21001 70001.7
22501 79296.6
24001 88610.9
25501 97891
27001 107200.3
28501 116556.8
};
\addlegendentry{RMax}
\end{axis}

\end{tikzpicture}
		\vspace{-0.4em}
		\caption{{$10 \times 10$ GridWord}}
		\label{sfig:grid_large}
	\end{subfigure}
\vspace{-0.5em}
\caption{Performance of \ourmod on challenging single-task problems, compared with RMax and Q-learning.}
\label{fig:extend}
\end{figure}
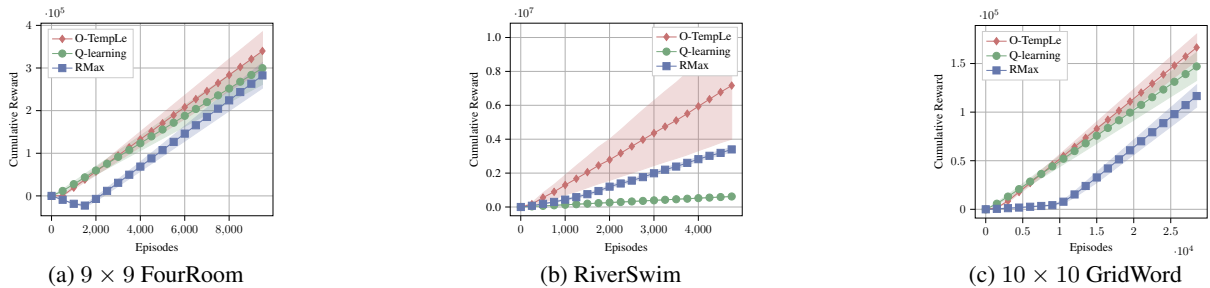

\newpage
\subsection{Discussion: Potential Extension to Deep RL}
\label{app:deep}

Model-based deep RL is an important research area~\cite{kaiser2019model,nagabandi2018neural}, where the learner learns a dynamics model of the environment. 
More specifically, the learner attempts to learn a function $f$ (usually parameterized by a neural network $\theta$) such that $f_{\theta}(s_t)$ approximates $s_{t+1}$, where $s_t$ is the current state and $s_{t+1}$ is the next state. The reward function can be modeled in a similar way, while we only discuss the transition model here for simplicity.

% $\theta = \argmin_{\theta} \| f_{\theta}(s_t) - s_{t-1} \| $

Our proposed \modname can be extended to large-scale MDPs and deep RL to learn the dynamics model more accurately. Below we explain the concrete method and some empirical results. 

\modname is essentially estimating the “relative” transition among states due to the permutation operation. For example, \modname considers the transition from $s_1$ to $s_2$ with probability 0.5 to be similar to the transition from $s_7$ to $s_8$ with probability 0.5, since the relative state difference of them is the same. This is equivalent to predicting a “state shift” in a continuous state space, which is $s_{t+1} - s_t$. In our paper, we focus on discrete state space and model the transition probabilities with discrete distributions.
Similarly, in the continuous case, we can use continuous distributions (e.g. Gaussian) to approximately model the state shift, without doing state counting and ranking. Note that \citet{nagabandi2018neural} also use the relative state shift in their deep RL model, whose experiments have justified the advantages of using relative state shift rather than absolute state difference. But their model is deterministic while we consider stochastic cases. 

In addition to the relative state shift modeling, another key idea of \modname is to cluster the old state-action dynamics and augment the new state-action dynamics. In the continuous case, if we assume the transition probabilities are from a mixture of Gaussian distributions, then a similar cluster and augment method can also be used. 
The extended algorithm works as follows:

(1) use a neural network (NN) to predict the relative state shift: $\hat{\Delta} \approx s_{t+1} - s_t$; \\
(2) approximately model $\Delta$’s using a mixture of Gaussian (other distribution models are also applicable). From the trajectories/history, we compute $\Delta= s_{t+1} - s_t$, cluster them (GEN-TT/TT-UPDATE step of \modname) and use the averaged $\bar{\Delta}$ from each Gaussian subpopulation/cluster to improve the prediction of the NN by minimizing MSE($\hat{\Delta}$, $\bar{\Delta}$); \\
(3) use $\bar{\Delta}$ to augment the accuracy of $\hat{\Delta}$ by identifying it into an existing cluster (AUGMENT step in \modname). As a result, we can learn an accurate prediction model of the environment.

We implemented the above idea on the continuous environments CartPole, LunarLander and Mujoco Hopper. We use a 2-layer MLP with 64 nodes per layer. The model learning methods are summarized as below.
\begin{itemize}
	\item \textbf{Absolute.} Directly predict the absolute next state.
	\item \textbf{Relative.} Predict the relative state shift~\cite{nagabandi2018neural}.
	\item \textbf{Relative+augment (ours).} Predict the relative state shift, and augment the model by clustering.
\end{itemize}

The results are shown in Figure~\ref{fig:deep}, where we can see that our method learns the most accurate model compared with two baselines, because learning relative state shift reduces the variance and the augmentation allows knowledge transferring among state-actions with similar dynamics. 

\begin{figure}[!htbp]
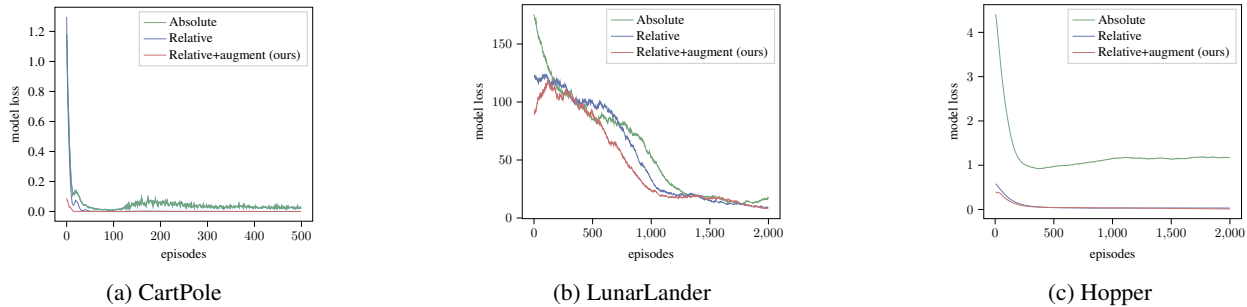

\centering
	\begin{subfigure}[t]{0.3\columnwidth}
		\centering
		\input{\fighome/new/cart.tex}
		\caption{CartPole}
		\label{sfig:cart}
	\end{subfigure}
	\hfill
	\begin{subfigure}[t]{0.3\columnwidth}
		\centering
		\input{\fighome/new/lunar.tex}	
		\caption{LunarLander} 
		\label{sfig:lunar}
	\end{subfigure}
	\hfill
	\begin{subfigure}[t]{0.3\columnwidth}
		\centering
		\input{\fighome/new/hopper.tex}	
		\caption{Hopper} 
		\label{sfig:hopper}
	\end{subfigure}
\caption{{Extending \modname to deep RL.} 
%Results are averaged over 10 runs.
}
\vspace{-0.5em}
\label{fig:deep}
\end{figure}

% Model losses (the lower the better):
% --------------------------------
% - CartPole: [Baseline 1] 0.05893; [Baseline 2] 0.00160; [Ours] 0.00023
% - Hopper: [Baseline 1] 0.90644; [Baseline 2] 0.03982; [Ours] 0.02782
% ---------------------------------

\end{document}